\documentclass[twoside]{article}

\def\includehome{./}
\def\fighome{./figures}

\usepackage[accepted]{\includehome/aistats2020}
\usepackage[bookmarks]{hyperref}     
\usepackage{url}            
\usepackage{nicefrac}       

\usepackage{xspace}
\usepackage[outdir=./]{epstopdf}
\usepackage{graphicx,float,pgfplots,wrapfig,sidecap,lipsum}
\usepackage{tabularx}
\usepackage{microtype}
\usepackage{subcaption}
\usepackage{booktabs}
\usepackage{paralist}
\usepackage{algorithm}
\usepackage[noend]{algorithmic}

\usepackage{amsthm}
\usepackage{amsmath}
\usepackage{amssymb}
\usepackage{enumitem}
\usepackage{xcolor}
\usepackage[font=normal,labelfont=bf]{caption}
\usepackage{mathtools} 

\usepackage{tablefootnote}
\usepackage{diagbox}
\usepackage{slashbox}
\usepackage{multirow}
\newcommand{\tabincell}[2]{\begin{tabular}{@{}#1@{}}#2\end{tabular}}

 \usepackage{psfrag}
 \usepackage{tikz}
 \usetikzlibrary{fit}
 \usetikzlibrary{calc,shapes}
 \usetikzlibrary{decorations.pathmorphing} 
 \usetikzlibrary{fit}					
 \usetikzlibrary{backgrounds}	
 \usepackage[parfill]{parskip}
 \usepackage{pgffor}

\usepackage{bm}
\usepackage{mathrsfs} 
\usepackage{stmaryrd}
\usepackage[utf8]{inputenc}
\usepackage[english]{babel}
\usepackage{mathrsfs}

\usepackage{pgfplots}
\pgfplotsset{compat=newest}
\usepgfplotslibrary{groupplots}
\usepgfplotslibrary{dateplot}

\theoremstyle{definition}

\def\ourNNLong{{CP Layer}\xspace}
\def\ourNN{{CPL}\xspace}
\def\ourNNlong{{CP layers}\xspace}

\def\ourcpnn{{CP}}

\def\ourNNLong{{CP Layer}\xspace}
\def\ourNN{{CPL}\xspace}
\def\ourNNlong{{CP layers}\xspace}
\def\ourTFC{{Tensorial-FC}\xspace}

\newcommand*{\mytensor}[1]{\mathcal{#1}}
\newcommand*{\mymatrix}[1]{\bm{#1}}
\newcommand*{\myvector}[1]{\bm{#1}}

\newcommand{\acc}{\mu}

\newcommand\abs[1]{\lvert#1\rvert}
\newcommand\absbig[1]{\big|#1\big|}
\newcommand\relu[1]{\textsf{ReLU}\left(#1\right)}
\newcommand\reluname{\textsf{ReLU}\xspace}

\newcommand\norm[1]{\left\lVert#1\right\rVert}
\newcommand\nucnorm[1]{\left\lVert#1\right\rVert_*}
\newcommand\fronorm[1]{\left\lVert#1\right\rVert_{\textsf{F}}}

\def\orignn{\mathbb{M}}
\def\tnn{\mathbb{M}}
\def\nnweight{\mytensor{M}}
\def\originput{x}
\def\tinput{X}
\def\origOut{y}
\def\tOut{Y}

\def\origcnn{\mathbb{M}}
\def\cnnweight{\mytensor{M}}

\newcommand{\ourfft}[2]{\mathcal{F}^{#1}_{#2}}

\def\Rbb{\mathbb{R}}
\def\Cbb{\mathbb{C}}

\def\tha{{\mbox{\tiny th}}}
\def\beq{\begin{equation}}

\def\eeq{\end{equation}\noindent}

\newcommand{\bp}{\begin{psfrags}}
\newcommand{\ep}{\end{psfrags}}
\newcommand{\bc}{\begin{center}}
\newcommand{\ec}{\end{center}}

\newcommand{\Conv}{\mathop{\scalebox{1.5}{\raisebox{-0.2ex}{$\ast$}}}}%

\DeclarePairedDelimiterX{\inp}[2]{\langle}{\rangle}{#1, #2}

\ifdefined\block
\else
\DeclareMathOperator{\block}{blk}
\fi

\DeclareMathOperator{\Tr}{Tr}

\ifdefined\theorem
\else
\newtheorem{theorem}{Theorem}[section]
\fi

\ifdefined\lemma
\else
\newtheorem{lemma}[theorem]{Lemma}
\fi

\ifdefined\definition
\else
\newtheorem{definition}[theorem]{Definition}
\fi

\ifdefined\assumption
\else
\newtheorem{assumption}[theorem]{Assumption}
\fi

\ifdefined\proposition
\else
\newtheorem{proposition}[theorem]{Proposition}
\fi

\ifdefined\corollary
\else
\newtheorem{corollary}[theorem]{Corollary}
\fi

\ifdefined\property
\else
\newtheorem{property}[theorem]{Property}
\fi

\ifdefined\fact
\else
\newtheorem{fact}[theorem]{Fact}
\fi

\ifdefined\remark
\else
\newtheorem*{remark}{Remark}
\fi

\floatname{algorithm}{Algorithm}

\newcommand\rf{\rho} 
\newcommand\tf{t} 
\newcommand\nb{\xi} 
\newcommand\lc{\zeta} 
\newcommand\cfr{\theta} 

\newcommand\tflong{tensorization factor\xspace} 
\newcommand\tfshort{TF\xspace} 
\newcommand\nblong{tensor noise bound\xspace} 
\newcommand\nbshort{TNB\xspace} 
\newcommand\lcshort{LC\xspace} 
\newcommand\fflong{Fourier factors\xspace} 

\def\FBRC{\textsf{FBRC}\xspace} 
\def\FBR{\textsf{FBR}\xspace} 
\def\CNNPROJECT{\textsf{CNN-Project}\xspace}
\def\TNNPROJECT{\textsf{TNN-Project}\xspace}

\newcommand{\mytitle}{Understanding Generalization in Deep Learning via Tensor Methods}

\newcommand\uoft{Graduate School of Information Science and Technology, The University of Tokyo}
\newcommand\riken{Center for Advanced Intelligence Project, RIKEN}
\newcommand\umdece{Department of Electrical and Computer Engineering, University of Maryland, College Park}
\newcommand\umdcs{Department of Computer Science, University of Maryland, College Park}

\usepackage[round]{natbib}

\begin{document}

%

%

\twocolumn[

\aistatstitle{\mytitle}

\aistatsauthor{
    Jingling Li${ }^{1,3}$
    \And  
    Yanchao Sun${ }^{1}$
    \And  
    Jiahao Su${ }^{4}$
    \And 
    Taiji Suzuki${ }^{2,3}$
    \And
    Furong Huang${ }^{1}$ 
}

\vspace{1em}

\aistatsaddress{
    ${ }^{1}$\umdcs \\
    ${ }^{2}$\uoft \\
    ${ }^{3}$\riken \\
    ${ }^{4}$\umdece \\
} 
\vspace{-1.2em}
]
\begin{abstract}
Deep neural networks generalize well on unseen data though the number of parameters often far exceeds the number of training examples. 
Recently proposed complexity measures have provided insights to understanding the generalizability in neural networks from perspectives of PAC-Bayes, robustness, overparametrization, compression and so on. 
In this work, we advance the understanding of the relations between the network's architecture and its generalizability from the compression perspective.
Using tensor analysis, we propose a series of intuitive, data-dependent and easily-measurable properties that tightly characterize the compressibility and generalizability of neural networks; 
thus, in practice, our generalization bound outperforms the previous compression-based ones, especially for neural networks using tensors as their weight kernels (e.g. CNNs). 
Moreover, these intuitive measurements provide further insights into designing neural network architectures with properties favorable for better/guaranteed generalizability. 
Our experimental results demonstrate that through the proposed measurable properties, our generalization error bound matches the trend of the test error well.
Our theoretical analysis further provides justifications for the empirical success and limitations of some widely-used tensor-based compression approaches. 
We also discover the improvements to the compressibility and robustness of current neural networks when incorporating tensor operations via our proposed layer-wise structure.
\end{abstract}

\section{Introduction}
\label{sec:intro}
Deep neural networks recently have made major breakthroughs in solving many difficult learning problems, especially in image classification~\citep{simonyan2014very, szegedy2015going, he2016deep, zagoruyko2016wide} and object recognition~\citep{krizhevsky2012imagenet,sermanet2013overfeat,simonyan2014very,zeiler2014visualizing}. 
The success of deep neural networks depends on the high expressive power and the ability to generalize. 
The high expressive power has been demonstrated empirically~\citep{he2016deep, zagoruyko2016wide} and theoretically~\citep{hornik1989multilayer, mhaskar2016deep}. 
Yet, fundamental questions on why deep neural networks generalize and what enables their generalizability remain unsettled.

A recent work by~\cite{arora2018stronger} characterizes the generalizability of a neural network from a compression perspective ---
the capacity of the network is characterized through its compressed version. 
The compression algorithm in~\cite{arora2018stronger} is based on random projection: each weight matrix of the compressed network are represented by a linear combination of basis matrices with entries i.i.d. sampled from $\pm 1$. 
The effective number of parameters in the weight matrix is the number of coefficients in this linear combination obtained via projection --- the inner product between the original weight matrix and these basis matrices. 
Though the idea of using compression in deriving the generalization bounds is novel, 
the compression scheme in~\cite{arora2018stronger} could be made more practical since
(1) the cost of forwarding pass in the compressed network still remains the same as the cost in the original one, even though the effective number of parameters to represent the original weight matrices decreases; 
(2) storing these random projection matrices could require more spaces than storing the original set of parameters. 
We propose a new theoretical analysis based on a more practical, well-developed, and principled compression scheme using tensor methods. 
Besides, we use tensor analysis to derive a much tighter bound for the layer-wise error propagation by exploiting additional structures in the weight tensors of neural networks, 
which as a result significantly tightens the generalization error bound in~\cite{arora2018stronger}.

Our approach aims to characterize the network's compressibility by measuring the low-rankness of the weight kernels. 
Existing compression methods in~\citep{jaderberg2014speeding, denton2014exploiting, lebedev2014speeding, kim2015compression, garipov2016ultimate, wang2018wide, su2019tensorial} implement low-rank approximations by performing matrix/tensor decomposition on weight matrices/kernels of well-trained models.  
However, the layers of SOTA networks, such as VGG~\citep{simonyan2014very} and WRN~\citep{zagoruyko2016wide}, are not necessarily low-rank: 
we apply CP-tensor decompositions~\citep{kolda2009tensor, anandkumar2014tensor,huang2015online,li2018guaranteed} to the weight tensors of well-trained VGG-16 and WRN-28-10, and \emph{the amplitudes of the components from the CP decomposition} (a.k.a \emph{\textbf{CP spectrum}}) are demonstrated by the brown curves in Figure~\ref{fig:eig_compare}, which indicate that the layers of these pre-trained networks are not low-rank. 
Therefore a straightforward compression of the network cannot be easily achieved and computationally expensive fine tuning is often needed.

\begin{figure}[!htbp]
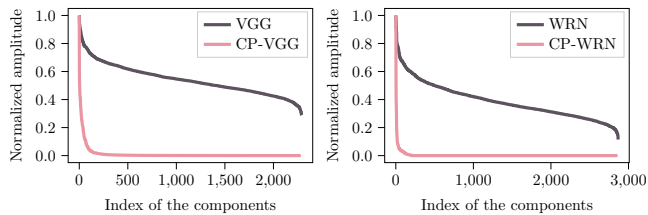

\centering
\begin{subfigure}[b]{0.49 \columnwidth}
	\centering
	\input{\fighome/VGG_conv13_eigenvalue_compare.tex}	
	\vspace{-0.5em}
	\caption{VGG16 (layer 13)}
	\label{fig:vgg_eig}
\end{subfigure}
\hfill
\begin{subfigure}[b]{0.49 \columnwidth}
	\centering
	\input{\fighome/WRN_conv27_eigenvalue_compare.tex}
	\vspace{-0.5em}
	\caption{WRN-28-10 (layer 28)}
	\label{fig:wrn_eig}
\end{subfigure}
\vspace{-1.5em}
\caption{CP spectrum comparison (CP-VGG and CP-WRN are neural networks with \ourNNlong).}
\vspace{-1em}
\label{fig:eig_compare}
\end{figure}

To overcome this limitation, we propose a layer-wise structure design, \emph{\ourNNLong} (\ourNN), by incorporating the variants of CP decompositions in~\citep{jaderberg2014speeding, kossaifi2017tensor2, howard2017mobilenets}. \ourNN re-parametrizes the weight tensors such that a \emph{Polyadic form} (CP form)~\citep{kolda2009tensor} can be easily learned in an end-to-end fashion. 

We demonstrate that empirically, \ourNN allows the network to learn a low-rank structure more easily, and thus helps with compression. 
For example, from the pink curves in Figure~\ref{fig:eig_compare}, we see that neural networks with \ourNN have a spiky CP spectrum, which is an indication of low-rankness. 
We rigorously prove that this low-rankness in return leads to a tighter generalization bound. 
Moreover, we are the first to provide theoretical guarantees for the usage of CP decomposition in deep neural networks in terms of compressibility and generalizability. 

\vspace{-0.5em}
\begin{definition}[\textbf{Proposed Architecture Layer}] \label{def:cpl}
A \ourNNLong (\ourNN) with width $R$ consists of $R$ set of parameters $\Big\{ \lambda^{(r)}, \big\{\myvector{v}^{(r)}_j\big\}_{j=1}^{N}  \Big\}_{r=1}^R$ where 
  $\myvector{v}^{(r)}_j$ is a vector in $\Rbb^{d_j}$ with unit norm. 
  The weight kernel of this \ourNN is a $N$-order tensor defined as $\mytensor{K} := \sum_{r=1}^{R} \lambda^{(r)} \myvector{v}^{(r)}_1 \otimes \cdots \otimes \myvector{v}^{(r)}_N$, where $\otimes$ denotes the vector outer-product (tensor product) defined in Appendix \ref{def:tensor_product})~\footnote{The $(i_1,i_2,\ldots,i_N)^{\tha}$ element of the weight kernel is $\sum_{r = 1}^{R} \lambda^{(r)} \myvector{v}^{(r)}_1(i_1) \times \cdots \times \myvector{v}^{(r)}_N(i_N)$.}. Note that $\mytensor{K} \in \Rbb^{d_1 \times \cdots \times d_N}$.
\end{definition} 
\vspace{-0.5em}
\textit{Remark.} \ourNN allows for flexible choices of the structures since the number of components $R$ is a tunable hyper-parameter that controls the number of parameters in \ourNN. 
The CP spectrum of this layer is denoted by $\{\lambda^{(r)}\}_{r=1}^R$ in a descending order.  
The size of the weight kernel is $d_0 \times d_1 \times \cdots \times d_N$, while the number of parameters in \ourNN is $(d_0+d_1+\cdots+d_N+1) \times R$. 

In contrast with existing works which apply CP decomposition to each layer of a reference network, no CP decomposition is needed since the components are explicitly stored as model parameters so that they can be learned from scratch via back-propagation. 
Moreover, compression in \ourNNlong is natural -- simply picking the top $\hat{R}$ components to retain and pruning out the rest of them.
Thus, the compression procedure using \ourNN does not require any costly fine-tuning while existing works on tensor-based compression may use hundreds of epochs for fine-tuning.

We further propose a series of simple, intuitive, data-dependent and easily-measurable \emph{properties} to measure the low-rankness in current neural networks. These properties not only guide the selection of the number of components to generate a good compression, but also tighten the bound of the layer-wise error propagation via tensor analysis. 
The proposed properties 
\begin{itemize}[leftmargin=*]
\vspace{-0.5em}
\item characterize the compressibility of the neural network,
 i.e., how much the original network can be compressed without compromising the performance on a training dataset more than certain range.
\vspace{-0.5em}
\item characterize the generalizability of the compressed network, i.e. tell if a neural network is trained using normal data or corrupted data.
\end{itemize}
\vspace{-0.5em}
In our theoretical analysis, we derive generalization error bounds for neural networks with \ourNNlong, which take both the input distribution and the compressibility of the network into account.  
We present a rigorous proof showing the connection of our proposed properties to the generalization error of a network. 
We will see in experiment section that our proposed bound is very effective at predicting the generalization error.
 
Notice that, in this paper, the Polyadic form is chosen simply as a demonstration on how tensor methods could be used to improve the analysis of generalization bounds of deep neural networks. Therefore, follow-ups works could potentially analyze the effects of other tensor decomposition methods using our theoretical framework.
 
\textbf{Summary of Contributions}
\begin{enumerate}[leftmargin=*,itemsep=0em,topsep=0em]
	\item \textbf{Better generalization bound of practical use.} 
	We verify that our generalization bounds can be used to guide the training of neural networks, since the calculated bound matches the trend of the test error on unseen data during the training process as shown in Figure~\ref{fig:bound}. 
Moreover, we demonstrate that our generalization bound is in practice tighter than the bound proposed by~\citep{arora2018stronger} as shown in Figure~\ref{fig:bound_compare} and Table~\ref{tbl:effective}. 
Notice that the generalization bound in~\citep{arora2018stronger} is already orders of magnitude better than previous norm-based or compression based bounds.
	\item \textbf{Intuitive measurements of compressibility and generalizability.} 
	We propose a set of properties to characterize the low-rankness in the weight tensors of neural networks in Section~\ref{sec:cnn_properties}. 
	Our theoretical analysis connects the measured low-rankness with the generalizability of the model, and such connections are verified in Figure~\ref{fig:properties}. 
	\item \textbf{First theoretical guarantee} on the generalizability and robustness for neural network architectures that allow fast and real time predictions on devices with limited memory (e.g. the architecture designs proposed in~\citep{jaderberg2014speeding, kossaifi2017tensor2, howard2017mobilenets}, which uses variants of the Polyadic form).

	\vspace{-0.25em}
	\item \textbf{Practical improvements.} 
	We demonstrate that pruning out the smaller components of CP decomposition in \ourNNlong roughly preserves the test performance without computationally expensive fine tuning (see Section~\ref{sec:exp_cpl_compression} and Table~\ref{tbl:compression}) as our proposed layer-wise structure is easily compressible. 
	Moreover, we discover that incorporating tensor operations via \ourNN reduces the generalization error of some well-known neural network architectures, and further improves the robustness of SOTA methods for learning under noisy labels (see Table~\ref{tbl:memorization}, Table~\ref{tbl:mnist}, Figure~\ref{fig:pf_sym_cr}, and Figure~\ref{fig:learning curve}). 
\end{enumerate}

\section{Related Works}
\vspace{-0.25em}
\label{sec:other_related}
\textbf{Existing Metrics to Characterizing Generalization.} Classical and recent works have analyzed the generalizability of neural networks 
from different perspective such as VC-dimension~\citep{bartlett1999almost, harvey2017nearly}, sharpness of the solution~\citep{keskar2016large}, robustness of the algorithm~\citep{xu2012robustness}, stability and robustness of the model~\citep{Hardt2016, Kuzborskij2018, Gonen2017, sokolic2016generalization} and over-parameterization~\citep{neyshabur2018towards, du2018power}, 
or using various approaches such as PAC-Bayes theory~\citep{mcallester1999some, mcallester1999pac, langford2002not, neyshabur2015norm, neyshabur2017exploring, dziugaite2017computing, golowich2018}, norm-based analysis~\citep{bartlett2002rademacher, neyshabur2015path, kawaguchi2017generalization, golowich2017size}, compression based approach~\citep{arora2018stronger}, and combinations of the above approaches~\citep{neyshabur2017exploring, neyshabur2017pac, bartlett2017spectrally, zhou2018non} (see~\citep{jakubovitz2018generalization} for a complete survey). 
While these works provide deep theoretical insights to the understanding of the generalizability in neural networks, they did not provide practical techniques to improve generalization.  

For the progress on non-vacuous generalization bounds, \cite{dziugaite2017computing} use non-convex optimization and PAC-Bayesian analysis to obtain a non-vacuous sample bound on MNIST, and ~\cite{zhou2018non} use a PAC-Bayesian compression approach to obtain non-vacuous generalization bounds on both MNIST and ImageNet via smart choices of the prior. 
While being creative, both bounds are less intuitive and provide little insight into what properties are favorable for networks to have better generalizability. 
In addition, the tensor-based compression methods are complementary to the compression approach used in~\citep{zhou2018non}, which combines pruning, quantization and huffman coding~\citep{han2015deep};
the tensor-based  compression methods can be combined with the approaches used in~\citep{han2015deep} to potentially tighten the generalization bound obtained in~\citep{zhou2018non}.  

\textbf{Improving generalization in practice.} Authors of~\citep{neyshabur2015path} proposed an optimization method PATH-SGD which improves the generalization performance empirically. While~\citep{neyshabur2015path} focuses on the optimization approach, we provide a different practical approach that helps the understanding of the relations between the network architecture and its generalization ability.

\textbf{Comparison with Arora et al.~\citep{arora2018stronger}.} Besides practical improvements of generalization error, our work improves the results obtained by~\citep{arora2018stronger}: 1) we provide a tightened layer-wise analysis using tensor method to directly bound the operator norm of the weight kernel (e.g. Lemma~\ref{lem:rcpl-fc-operator-norm} and Lemma~\ref{lem:rcpl-conv-operator-norm}). The interlayer properties introduced by~\citep{arora2018stronger} are orthogonal to our proposed layer-wise properties and they can be well-combined; 2) in practice, our bound outperforms that of~\citep{arora2018stronger} in terms of the achieved degree of compression (detailed discussions in Section~\ref{sec:exp_gen_imp} and Section~\ref{sec:comp_arora});  3) for fully connected (FC) neural networks, our proposed reshaping factor (definition~\ref{def:rf}) further tightens the generalization bound as long as the inputs to the FC layers have some low-rank structures; 4) we extend our theoretical analysis to neural networks with skip connections, while the theoretical analysis in~\cite{arora2018stronger} only applies to FC and CNN.

\textbf{Comparison with existing CP decomposition for network compression.} While CP decomposition has been commonly used in neural network compression~\citep{denton2014exploiting, lebedev2014speeding, kossaifi2017tensor2}, our proposed compression method is very different from theirs. First, the the tensor contraction layer~\cite{kossaifi2017tensor2} is a special case of our \ourNN for FC layers when we set the number of components to be 1. Second, the number of components in our proposed \ourNN can be arbitrarily large (as it is a tunable hyper-parameter), while the number of components of layers in~\citep{denton2014exploiting, lebedev2014speeding, kossaifi2017tensor2} are determined by the compression ratio. Third, no tensor decomposition is needed for evaluating the generalizability and compressing neural networks with \ourNNlong as the components from the CP decomposition are already stored as model parameters. 
Moreover, as the smaller components in \ourNN are pruned during the compression, the performance of the compressed neural net is often preserved and thus no expensive fine tuning is required (see Table~\ref{tbl:compression}).  The depthwise-separable convolution used in MobileNet~\citep{howard2017mobilenets} is a specific implementation of \ourNN; thus, our theoretical analysis can provide generalization guarantees for the MobileNet architecture.

\vspace{-0.25em}
\section{Notations and Preliminaries}
\vspace{-0.25em}
\label{sec:setup}
In this paper, we use $S$ to denote the set of training samples drawn from a distribution $D$ with $|S| = m$. Let $n$ denote the number of layers in a given neural network, and superscripts of form ${ }^{(k)}$ denote properties related to the $k^{\tha}$ layer. We put ``CP'' in front of a network's name to denote such a network with \ourNNlong (e.g. CP-VGG denotes a VGG with \ourNNlong).
For any positive integer $n$, let $[n] := \{1,2, ..., n\}$. Let $\abs{a}$ denote the absolute value of a scalar $a$. 
Given a vector $\myvector{a} \in \Rbb^d$, a matrix $\mymatrix{A} \in \Rbb^{d\times k}$, and a tensor $\mytensor{A} \in \Rbb^{d_1\times d_2 \times d_3}$,  their norms are defined as follows:
(1) \textbf{Vector norm:} $\norm{\myvector{a}}$ denotes the $\ell_2$ norm.
(2) \textbf{Matrix norms:} Let $\nucnorm{\mymatrix{A}}$
denote its nuclear norm,  $\fronorm{\mymatrix{A}}$
denote its Frobenius norm, and $\norm{\mymatrix{A}}$
denote its operator norm (spectral norm), where $\sigma_i(\mymatrix{A})$ denotes the $i^{\tha}$ largest singular value of $\mymatrix{A}$.
(3) \textbf{Tensor norms:} Let $\norm{ \mytensor{A}} = \max_{x \in \Rbb^{d_1}, y \in \Rbb^{d_2}, z \in \Rbb^{d_3}} \frac{\abs{A(x,y,z)} }{\norm{x}\norm{y}\norm{z}}$ denote its operator norm, and $\fronorm{ \mytensor{A}}$ its Frobenius norm. 
Moreover, we use $\otimes$ to denote the \textbf{outer product operator}, and $\Conv$ to denote \textbf{the convolution operator}.
We use \textbf{$\ourfft{}{m}$} to denote \textbf{$m$-dimensional discrete Fourier transform}, and use tilde symbols  
to denote tensors after DFT (e.g. $\Tilde{\mytensor{T}} = \ourfft{}{m}(\mytensor{T})$).
A \textbf{Polyadic decomposition (CP decomposition)}~\citep{kruskal1989rank, kolda2009tensor} of a $N$-order tensor $\mytensor{K} \in \Rbb^{d_1 \times d_2 \times \cdots \times d_N}$ is a linear combination of rank-one tensors that is equal to $\mytensor{K}$:
$\mytensor{K} = \sum_{r=1}^R \lambda^{(r)} \myvector{v}^{(r)}_1 \otimes \cdots \otimes \myvector{v}^{(r)}_N$
 where $\forall r \in [R], \forall j \in [N], \norm{\myvector{v}^{(r)}_{j}} = 1$.
\textbf{Margin loss~\cite{arora2018stronger}:} we use $L_{\gamma}(\origcnn)$ and $\hat{L}_{\gamma}(\origcnn)$ to denote the expected and empirical margin loss of a neural network $\origcnn$ with respect to a margin $\gamma \geq 0$. The expected margin loss of a neural network $\origcnn$ is defined as
 $L_{\gamma}(\origcnn):=\mathbb{P}_{(\myvector{x},y)\in D}\big[\origcnn(\myvector{x})[y]\leq \gamma+\max_{i\neq y}\origcnn(\myvector{x})[i]\big].$

\section{CNNs with \ourNN: Compressibility and Generalization}
\label{sec:cnn}
In this section, we derive the generalization bound for a convolutional neural network (denoted as $\origcnn$) using tensor methods and standard Fourier analysis. 
The complete proof is in Appendix Section~\ref{app:cnn_proofs}.
For simplicity, we assume that there is no pooling layer (e.g.~max pooling) in $\origcnn$ since adding pooling layer will only lead to a smaller generalization bound (the perturbation error in our analysis decreases with the presence of pooling layers). 
The derived generalization bound can be directly extended to various neural network architectures (e.g. neural networks with pooling layers, and neural networks with batch normalization). 
The generalization bounds for fully connected neural networks and neural networks with skip connections are presented in Appendix Section~\ref{app:proofs} and ~\ref{app:sc_proofs} respectively.

\subsection{Compression of a CNN with \ourNN}
 We first illustrate how to compress any given CNN $\origcnn$ by presenting a compression algorithm (Algorithm~\ref{alg:cnn_compress}). We will see that this compression algorithm guarantees a good estimation of the generalization bound for the compressed network $\hat{\origcnn}$.

\textbf{Original CNN $\origcnn$} is of \textbf{$n$} layers with \reluname activation, its $k^{\tha}$ layer weight tensor $\mytensor{\cnnweight}^{(k)}$
    is a $4^{\tha}$ order tensor of size = \emph{\# of input channel} $s^{(k)}$ $\times$ \emph{\# of output channel} $o^{(k)}$  $\times$ kernel height  $k_x^{(k)}$ $\times$ kernel width $k_y^{(k)}$. 
    Let the $3^{\mbox{\tiny rd}}$ order tensor $\mytensor{\tinput}^{(k)} \in \mathbb{R}^{H^{(k)} \times W^{(k)} \times s^{(k)}}$ denote the input to the $k^{\tha}$ layer, and $\mytensor{\tOut}^{(k)} \in \mathbb{R}^{H^{(k)} \times W^{(k)} \times o^{(k)}}$ denote the output of the $k^{\tha}$ layer before activation.
    Therefore $\mytensor{\tinput}^{(k)} = \relu{\mytensor{\tOut}^{(k-1)}}$. We use $i$ to denote the index of input channels, and $j$ to denote the index of output channels. 
	We further use $f$ and $g$ to denote the indices of width and height in the frequency domain. 

\vspace{-0.5em}
\begin{proposition}[\textbf{Polyadic Form of original CNN $\origcnn$}]
\label{assum:CP-cnn}
For each layer $k$, the weight tensor $\mytensor{\cnnweight}^{(k)}$ has a Polyadic form with number of components $R^{(k)} \leq \min\{s^{(k)} o^{(k)}, s^{(k)} k_x^{(k)} k_y^{(k)}, o^{(k)} k_x^{(k)} k_y^{(k)}\}$ ~\citep{kolda2009tensor}:  
$\mytensor{\cnnweight}^{(k)} = \sum_{r = 1}^{R^{(k)}} \lambda_{r}^{(k)} \myvector{a}_{r}^{(k)} \otimes \myvector{b}_{r}^{(k)} \otimes \mymatrix{C}_{r}^{(k)}$, 
where the CP-spectrum is in a descending order, i.e., $\lambda_{1}^{(k)} \geq \lambda_{2}^{(k)} \geq \cdots \geq \lambda_{R^{(k)}}^{(k)}$. All $\myvector{a}_r^{(k)}, \myvector{b}_r^{(k)}$ are unit vectors in $\Rbb^{s^{(k)}}$ and $\Rbb^{o^{(k)}}$ respectively, and $\mymatrix{C}_r^{(k)}$ is a matrix in $\Rbb^{k_x^{(k)}\times k_y^{(k)}}$ with $\lVert \mymatrix{C}_r^{(k)}\rVert_{\textsf{F}}=1$.
The $R^{(k)}$ required for the Polyadic Form is called tensor rank. 
\end{proposition}
\vspace{-0.5em}

\textbf{Transform original CNN to a CNN with \ourNNlong.} By Proposition~\ref{assum:CP-cnn}, each weight tensor $\mytensor{\cnnweight}^{(k)}$ in $\origcnn$ can be represented in a Polyadic form (CP form) and thus is transformed to a \ourNN. 
The total number of parameters in \ourNN is $R^{(k)} \times (s^{(k)} + o^{(k)} + k_x^{(k)}  k_y^{(k)}+1)$. 
Thus, a smaller $R^{(k)}$ leads to fewer number of effective parameters and indicates more compression.

\textbf{Compress Original CNN $\origcnn$ to $\hat{\origcnn}$.} We illustrate the compression procedure in Algorithm~\ref{alg:cnn_compress}. Feeding a CNN $\origcnn$ to the compression algorithm, we obtain a compressed CNN $\hat{\origcnn}$, where for each layer $k$, the weight tensor in $\hat{\origcnn}$ is
$
\mytensor{\hat{M}}^{(k)} = \sum_{r = 1}^{\hat{R}^{(k)}} \lambda_{r}^{(k)} \myvector{a}_{r}^{(k)} \otimes \myvector{b}_{r}^{(k)} \otimes \myvector{c}_{r}^{(k)} \text{ for some $\hat{R}^{(k)} \leq R^{(k)}$ }.
$ 
Similarly, we use $\hat{\mytensor{\tinput}}^{(k)}$ to denote the input tensor of the $k^{\tha}$ layer in $\hat{\origcnn}$ and $\hat{\mytensor{\tOut}}^{(k)}$ to denote the output tensor of the $k^{\tha}$ layer in $\hat{\origcnn}$ before activation. 
Therefore $\hat{\mytensor{\tinput}}^{(k)} = \relu{\hat{\mytensor{\tOut}}^{(k-1)}}$. 
Notice that $\hat{\mytensor{\tinput}}^{(k)}, \hat{\mytensor{\tOut}}^{(k)}$ are of the same shapes as $\mytensor{\tinput}^{(k)}, \mytensor{\tOut}^{(k)}$  respectively and  $\mytensor{\tinput}^{(1)} = \hat{\mytensor{\tinput}}^{(1)}$ since the input data to both networks $\origcnn$ and $\hat{\origcnn}$ is the same. 

The compression Algorithm~\ref{alg:cnn_compress} is designed to compress any CNN, and therefore requires applying explicit CP decompositions to the weight tensors of traditional CNNs (the step 3 in Algorithm~\ref{alg:cnn_compress}).
However, for a CNN with \ourNNlong, these CP components are already stored as weight parameters in our \ourNN structure, and thus are known to the compression algorithm in advance. Therefore, no tensor decomposition is needed when compressing CNNs with \ourNN as we can prune out the components with smaller amplitudes directly. 
\setlength{\textfloatsep}{0pt}

\begin{algorithm}[h!]
\caption{\textbf{Compression of Convolutional Neural Networks} \\ 
\footnotesize{${}^\square$\FBRC (in Appendix~\ref{app:algos}) calculates a set of number of components $\{\hat{R}^{(k)}\}_{k=1}^n$ for the compressed network such that $\fronorm{\origcnn(\mytensor{\tinput}) - \hat{\origcnn}(\mytensor{\tinput})} \leq \epsilon \fronorm{\origcnn(\mytensor{\tinput})}$ holds for any given $\epsilon$ and for any input $\mytensor{\tinput}$ in the training dataset $S$.\smallskip\\
${}^{\triangle}$\CNNPROJECT (in Appendix~\ref{app:algos}) takes a given set of number of components $\{\hat{R}^{(k)}\}_{k=1}^n$ and returns a compressed network $\hat{\origcnn}$ by pruning out the smaller components in the CP spectrum of the weight tensors of $\origcnn$. \smallskip\\
More intuitions of the sub-procedures \FBRC and \CNNPROJECT are described in Section~\ref{sec:cnn_properties} and Appendix~\ref{app:algos}.}
}
\label{alg:cnn_compress}
\begin{algorithmic}[1]
\REQUIRE A CNN $\origcnn$ of $n$ layers 
and a margin $\gamma$
\ENSURE A compressed $\hat{\origcnn}$ whose expected error $L_0(\hat{\origcnn}) \le \hat{L}_{\gamma}(\origcnn) + \tilde{O}\Big(\sqrt{\frac{ \sum_{k=1}^{n} \hat{R}^{(k)}(s^{(k)} + o^{(k)} + k_x^{(k)} \times k_y^{(k)} + 1)}{m}}\Big)$ 
\STATE Calculate all layer cushions $\{ \lc^{(k)} \}_{k=1}^{n}$ based on definition~\ref{def:cnn_lc}
\STATE Pick $R^{(k)} = \min\{s^{(k)} o^{(k)}, s^{(k)} k_x^{(k)} k_y^{(k)}, o^{(k)} k_x^{(k)} k_y^{(k)}\}$ for each layer $k$
\STATE If $\origcnn$ does not have \ourNN, apply a CP-decomposition to the weight tensor of each layer $k$
\STATE Set the perturbation parameter $\epsilon := \frac{\gamma}{2\max_{\mytensor{\tinput}} \fronorm{\origcnn(\mytensor{\tinput})}} $
\STATE Compute number of components needed for each layer of the compressed network $\{ \hat{R}^{(k)} \}_{k=1}^{n} \gets \hyperref[alg:find_best_rank_cnn]{\text{\FBRC}}^\square\Big(\{\cnnweight^{(k)}\}_{k=1}^n, \{R^{(k)}\}_{k=1}^n, \{ \lc^{(k)} \}_{k=1}^{n }, \epsilon\Big)$ 
\STATE $\hat{\origcnn} \gets$ \hyperref[alg:cnn_project]{\CNNPROJECT}${}^{\triangle}\Big(\origcnn, \{\hat{R}^{(k)}\}_{i=1}^n\Big)$
\STATE Return the compressed convolutional neural network $\hat{\origcnn}$ 
\end{algorithmic}
\end{algorithm}

\subsection{Characterizing Compressibility of CNN with \ourNN: Network Properties}
\label{sec:cnn_properties}
In this section, we propose the following layer-wise properties that can be evaluated based on the training data  $S$: \textit{\tflong (\tfshort)}, \textit{\nblong (\nbshort)}, and \textit{layer cushion (\lcshort)}~\citep{arora2018stronger}. These proposed properties are very effective at characterizing the compressibility of a neural network. As Algorithm~\ref{alg:cnn_compress}'s sub-procedure \FBRC selects a set of number of components $\{\hat{R}^{(k)}\}_{k=1}^{n}$ to obtain a compressed network $\hat{\origcnn}$ whose output is similar to that of the original network (i.e., $\fronorm{\origcnn(\mytensor{\tinput}) -  \hat{\origcnn}(\mytensor{\tinput})} \leq \epsilon \fronorm{\origcnn(\mytensor{\tinput})}$ for any input $\mytensor{\tinput} \in S$), our proposed properties will assist the selections of $\{\hat{R}^{(k)}\}_{k=1}^{n}$  to guarantee that Algorithm~\ref{alg:cnn_compress} returns a ``good'' compressed network.

\begin{definition} \label{def:cnn_tf}
[\tflong $\tf^{(k)}_j$] 
The {\em {\tflong}s} $\left\{ \tf^{(k)}_j \right\}_{j=1}^{R^{(k)}}$ of the $k^{\tha}$ layer is defined as
\begin{equation}
\tf^{(k)}_j := \max_{f, g} \sum_{r = 1}^{j} \left| \lambda^{(k)}_{r} \right| \left| \Tilde{C}^{(f, g)}_r \right|
\label{eq:cnn_tf}
\end{equation}
where $\lambda_r^{(k)}$ is the $r^{\tha}$ largest value in the CP spectrum of the weight tensor $\mytensor{M}^{(k)}$ and $\Tilde{C}^{(f, g)}_r$ denotes the amplitude at the frequency $(f, g)$.
\end{definition}
\textit{Remark.} 
The {\tflong} characterizes both the generalizability and the expressive power of a given network. 
For a fixed $j$, a smaller {\tflong} indicates the original network is more compressible and thus has a smaller generalization bound. 
However, a smaller {\tflong} may also indicate that the given network do not possess enough expressive power. 
Thus, during the compression of a neural network with good generalizability, we need to find a ``good'' $j$ that generates a {\tflong} demonstrating the balance between a small generalization gap and high expressive power.
\begin{definition} \label{def:cnn_tnb}  
[\nblong $\nb^{(k)}_j$] 
The {\em \nblong} $\left\{ \nb^{(k)}_j \right\}_{j=1}^{R^{(k)}}$ of the $k^{\tha}$ layer 
measures the amplitudes of the remaining components after pruning the ones
with amplitudes smaller than the $\lambda^{(k)}_j$:
\begin{equation}
\nb^{(k)}_j := \max_{f, g} \sum_{r=j+1}^{R^{(k)}} \left| \lambda_{r}^{(k)} \right| \left| \Tilde{C}^{(f, g)}_r \right|
\label{eq:cnn_tnb}
\end{equation}
\end{definition}
\textit{Remark.} For a fixed $j$, a smaller tensor noise bound indicates the original neural network's weight tensor is more low-rank and thus more compressible. 

\begin{definition} \label{def:cnn_lc} 
[layer cushion $\lc^{(k)}$]
As introduced in~\cite{arora2018stronger}, the layer cushion of the $k^{\tha}$ layer is defined to be the largest value $\lc^{(k)}$ such that for any $\mytensor{X}^{(k)} \in S$, 
\begin{equation}
\lc^{(k)} \left( \fronorm{\mytensor{M}^{(k)}} \Big/ \sqrt{H^{(k)}W^{(k)}} \right)  \fronorm{\mytensor{X}^{(k)}}  \leq \fronorm{ \mytensor{M}^{(k+1)}}
\label{eq:cnn_lc}
\end{equation}
Following~\citet{arora2018stronger}, layer cushion considers 
how much the output tensor $\fronorm{\mytensor{M}^{(k+1)} }$
grows w.r.t. the weight tensor $\fronorm{\mytensor{M}^{(k)}}$ and the input $\fronorm{ \mytensor{X}^{(k)}}$. 
\end{definition}
\textit{Remark}. 
As introduced in~\cite{arora2018stronger}, the layer cushion considers how much smaller the output $\fronorm{ \mytensor{\tinput}^{(k+1)} }$ of the $k^{\tha}$ layer (after activation) compares with the product between the weight tensor $\fronorm{\cnnweight^{(k)} }$ and the input $\fronorm{ \mytensor{\tinput}^{(k)}} $. 
Note that our layer cushion can be larger than $1$ if models use batchnorm, and larger layer cushions will render smaller generalization bounds as also shown in~\citep{arora2018stronger}.

Our proposed properties, orthogonal to the interlayer properties introduced in~\citep{arora2018stronger},  provide better measurements of the compressibility in each individual convolutional layer via the use of tensor analysis and Fourier analysis, and thus lead to a tighter bound of the layer-wise error propagation.

\subsection{Generalization Guarantee of CNNs}
\label{sec:cnn_gen_thm}

Based on Algorithm~\ref{alg:cnn_compress} and our proposed properties in section~\ref{sec:cnn_properties}, 
we obtain a generalization bound for the compressed convolutional neural network $\hat{\origcnn}$ and, in section~\ref{sec:experiments}, we will evaluate this bound explicitly.

\begin{theorem} [\textbf{Main Theorem}]
\label{thm:cnn_thm}
For any convolutional neural network $\origcnn$ with $n$ layers, Algorithm~\ref{alg:cnn_compress} generates a compressed CNN $\hat{\origcnn}$ such that with high probability, the expected error $L_0(\hat{\origcnn})$ is bounded by the empirical margin loss $\hat{L}_{\gamma}(\origcnn)$ (for any margin $\gamma \geq 0$) and a complexity term defined as follows
\begin{equation}
\begin{aligned}
& L_0(\hat{\origcnn}) \leq \hat{L}_{\gamma}(\origcnn) + \\
~ & \tilde{O}\left(\sqrt{\frac{ \sum_{k=1}^{n} \hat{R}^{(k)}(s^{(k)} + o^{(k)} + k_x^{(k)}k_y^{(k)} + 1)}{m}}\right) 
\label{eq:bound}
\end{aligned}
\end{equation}
given that for all layer $k$, the number of components $\hat{R}^{(k)}$ in the compressed network satisfies that
\begin{equation}
\label{eq:main_eq}
\hat{R}^{(k)} = \min \Big\{ j \in [R^{(k)}] | \nb^{(k)}_j \Pi_{i=k+1}^{n} \tf^{(i)}_j \leq C \Big\} 
\end{equation}
 $$ \text{ with }C = \frac{\gamma}{2 n \max_{\mytensor{\tinput}  \in S } \fronorm{\origcnn(\mytensor{\tinput})}} \Pi_{i=k}^n \lc^{(i)} \fronorm{\cnnweight^{(i)}}$$
where $\tf^{(k)}_j$,  $\nb^{(k)}_j$ and $\lc^{(k)}$ are data dependent measurable properties --- \tflong, \nblong, and layer cushion of the $k^{\tha}$ layer in definitions~\ref{def:cnn_tf},~\ref{def:cnn_tnb} and~\ref{def:cnn_lc} respectively.
\end{theorem}
\textit{Remark.} How well the compressed neural network approximates the original network is related to the choice of $\hat{R}^{(k)}$.  Inside equation~\eqref{eq:main_eq}, $C$ is some value independent of the choice of $j$ in the inequality. Therefore, the number of components for the $k^{\tha}$ layer in the compressed network, $\hat{R}^{(k)}$, is the smallest $j \in [R^{(k)}] $ such that the inequality $\nb^{(k)}_j \Pi_{i=k+1}^{n} \tf^{(i)}_j \leq C$ holds. Hence, smaller {\tflong}s and {\nblong}s will make the LHS smaller, and larger layer cushions will make the RHS, $C$, larger. As a result, if the above inequality for each layer can be satisfied by a smaller $j$, the obtained generalization bound will be tighter as we can obtain a smaller $\hat{R}^{(k)}$.

\textbf{Analysis of generalization bounds in Theorem~\ref{thm:cnn_thm}}: 
This proposed generalization error bound is proportional to the number of components in the {\ourNNlong} of the compressed neural network. 
Therefore, when the original neural network is highly compressible or very low-rank, the number of components needed will be lower, which thus renders a smaller generalization error bound.

The proof of Theorem~\ref{thm:cnn_thm} is in Appendix Section~\ref{app:cnn_proofs}), and the proof sketch is as follows.

\textbf{Proof sketch of Theorem~\ref{thm:cnn_thm}:} 
We first establish that the difference of the outputs between the compressed CNN $\hat{\origcnn}$ and the original CNN $\origcnn$ is bounded by $\frac{\gamma}{2\max_{\mytensor{\tinput}} \fronorm{\origcnn(\mytensor{\tinput})}}$ using Lemma~\ref{lem:cnn_lemma}.
Then we show the covering number of the compressed network $\hat{\origcnn}$ is $\tilde{O}(d)$ via Lemma~\ref{lem:cnn_covering_number}, where $d$ denotes the total number of parameters in the compressed network. Bounding the covering number of CNNs with \ourNN to be of order $\tilde{O}(d)$ is non-trivial as we need careful handlings of the error propagations to avoid a dependence on the product of number of components.
After bounding the covering number, the rest of the proof follows from conventional learning theory and Theorem 2.1 in~\citep{arora2018stronger}.

\section{Experiments}
\label{sec:experiments}
\begin{table}
\begin{center}
    \caption{ Comparison of the training and test accuracies between neural networks (NNs) with \ourNN (\ourcpnn-VGG-16, \ourcpnn-WRN-28-10) and traditional NNs (VGG-16. WRN-28-10) on CIFAR10 dataset.
    }
    \label{table:exp_power}
\vspace{-0.25em}
\resizebox{\columnwidth}{!}{%
\begin{tabular}{cc|cc|cc}
        \multicolumn{2}{c|}{\multirow{2}{*}{\diagbox{Dataset}{Acc.}{Architect.}}} & \multicolumn{2}{c|}{VGG-16} & \multicolumn{2}{c}{WRN-28-10} \\\cline{3-6} 
        & & \tabincell{c}{with CPL} & \tabincell{c}{without CPL} & \tabincell{c}{with CPL} &  \tabincell{c}{without CPL}  \\[5pt] \hline
        \multirow{2}{*}{CIFAR10} & Training & {100\%} & {100\%} & {100\%} & {100\%}  \\\cline{2-6} 
        & Test & 93.68\%     & 92.64\%$^\dag$  & 95.09\%    & 95.83\%$^*$ \\\hline 
        \multirow{2}{*}{CIFAR100} & Training & {100\%} & {100\%} & {100\%} & {100\%}  \\\cline{2-6} 
        & Test & 71.8\% & 70.84\%$^\ddag$  & 76.36\% & 79.5\%$^*$ \tablefootnote{$^\dag$https://github.com/kuangliu/pytorch-cifar\\$^\ddag$https://github.com/geifmany/cifar-vgg\\ $^*$~\cite{zagoruyko2016wide}}
\end{tabular}
}
\end{center}
\vspace{-1em}
\end{table}

\begin{table}
	\centering
	\caption{Test accuracy on CIFAR10 with various label corruptions rates (CR).}
    \vspace{-0.25em}
    \label{tbl:memorization}
	\resizebox{\columnwidth}{!}{%
	\begin{tabular}{ ll | llll}
			\multicolumn{2}{c|}{Network / CR}  & 0.2  & 0.4 & 0.6 & 0.8\\
			\hline
			\multirow{2}{*}{CIFAR10}
			& VGG-16
			& $68.76$ 
			& $44.26$ 
			& $24.89$ 
			& $13.21$ 
			\\
			\cline{2-6}
			&\ourcpnn-VGG-16 
			& \textbf{71.09}
			& \textbf{51.76}
			& \textbf{35.60}
			& \textbf{20.06}
			\\
			\hline
			\multirow{2}{*}{CIFAR100}
			& VGG-16
			& $50.94$ 
			& $30.46$ 
			& $13.6$ 
			& $1.11$ 
			\\
			\cline{2-6}
			&\ourcpnn-VGG-16 
			& \textbf{54.51}
			& \textbf{34.13}
			& \textbf{15.23}
			& \textbf{3.10}
		\end{tabular}
	}
\end{table}

\begin{table*}[bthp]
\small
	\centering
	\caption{Average test accuracy on MNIST over the last ten epochs. Baseline simply denotes training a neural network on the corrupted training set without further processing.
	PairFlip denotes that the label mistakes can only happen within very similar classes and Symmetric denotes that the label mistakes may happen across different classes uniformly~\citep{han2018co}. 
	 }
	\vspace{-0.75em}
	\begin{tabular}{ l | lllllll}
			Task: Rate  & \tabincell{l}{Baseline\\~\citep{han2018co}}   & \tabincell{l}{F-correction\\~\citep{han2018co}} & \tabincell{l}{MentorNet\\~\citep{jiang2017mentornet}} & \tabincell{l}{CT\\~\citep{han2018co}} & CT + CPL\\
			\hline
			PairFlip: 45\% 
			& $56.52 \pm 0.55$ 
			& $0.24 \pm 0.03$ 
			& $80.88\pm 4.45$ 
			& $87.63\pm 0.21$ 
			& $\mathbf{92.43 \pm 0.01} $
			\\
			\hline
			
			Symmetric: 50\% 
			& $66.05 \pm 0.61$
			& $79.61 \pm 1.96$
			& $90.05 \pm 0.30$
			& $91.32 \pm 0.06$
			& $\mathbf{94.70 \pm 0.05}$ 
			
			\\
			\hline
			Symmetric: 20\% 
			&$ 94.05 \pm 0.16$ 
			&$\mathbf{98.80 \pm 0.12}$
			&$ 96.70 \pm 0.22$
			&$ 97.25 \pm 0.03$
			&$ 97.91 \pm 0.01$ 
			\\
	\end{tabular}
	\label{tbl:mnist}
	\vspace{-1.75em}
\end{table*}

\textbf{Architecture and optimization setting.} The architectures we use in the experiments consist of VGG-16~\citep{simonyan2014very},  \ourcpnn-VGG-16, WRN-28-10~\citep{zagoruyko2016wide} and \ourcpnn-WRN-28-10 (all with batch normalization). Details of the optimization settings are in~\ref{app:opt_setting}.

\subsection{Evaluation of Proposed Properties and Generalization Bounds}
\label{sec:exp_compare}
\textbf{Tighter Generalization Bound.} As shown in Fig~\ref{fig:bound_compare}, our bound is much tighter than the the state-of-the-art bound achieved in~\cite{arora2018stronger}. The effective number of parameters in~\cite{arora2018stronger} is orders of magnitude tighter than other capacity measures, such as $\ell_{1,\infty}$~\citep{bartlett2002rademacher}, Frobenius~\citep{neyshabur2015norm}, spec $\ell_{1,2}$~\citep{bartlett2017spectrally} and spec-fro~\citep{neyshabur2017pac} as shown in their Figure 4 Left. 
The use of a more effective and practical compression approach allows us to achieve better compression (detailed discussions are in Appendix Section~\ref{sec:comp_arora}).

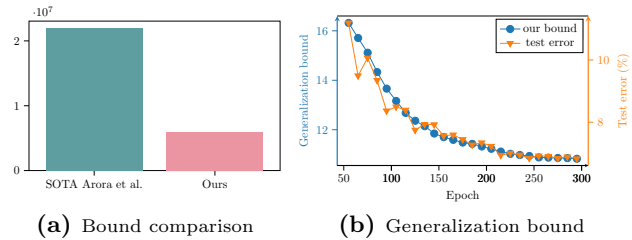
\begin{figure}[!htbp]
\centering
	\begin{subfigure}[b]{0.44\columnwidth}
\begin{tikzpicture}[scale=0.45]

\definecolor{color0}{rgb}{0.36328125,0.32421875,0.3671875}
\definecolor{color1}{rgb}{0.372549019607843,0.619607843137255,0.627450980392157}
\definecolor{color2}{rgb}{0.921875,0.5859375,0.640625}

\begin{axis}[
height=6cm,
tick align=outside,
tick pos=left,
width=9cm,
x grid style={white!69.01960784313725!black},
xmin=1.46, xmax=3.54,
xtick style={color=black},
xtick={1,2,3},
xticklabels={Original, SOTA Arora et al.
, Ours},
y grid style={white!69.01960784313725!black},
ymin=0, ymax=23049148.91265,
ytick style={color=black}
]
\draw[fill=color1,draw opacity=0] (axis cs:1.6,0) rectangle (axis cs:2.4,21951570.393);
\draw[fill=color2,draw opacity=0] (axis cs:2.6,0) rectangle (axis cs:3.4,5870050);
\end{axis}

\end{tikzpicture}	 
		\vspace{-1em} 	
	  	\caption{\scriptsize{Bound comparison}}
	  	\label{fig:bound_compare}
	\end{subfigure}
		\hfill
\begin{subfigure}[b]{0.54\columnwidth}
\begin{tikzpicture}[scale=0.45]

\definecolor{color0}{rgb}{0.12156862745098,0.466666666666667,0.705882352941177}
\definecolor{color1}{rgb}{1,0.498039215686275,0.0549019607843137}

\begin{axis}[
axis x line=bottom,
axis y line=left,
every outer y axis line/.append style={color0},
every y tick label/.append style={color0},
height=6cm,
legend cell align={left},
legend style={draw=white!80.0!black},
tick align=outside,
tick pos=left,
width=9cm,
x grid style={white!69.01960784313725!black},
xlabel={Epoch},
xmin=43, xmax=307,
xtick style={color=black},
xtick={0,100,200,300,400},
xticklabels={0,100,200,300,400},
y grid style={white!69.01960784313725!black},
ylabel={\textcolor{color0}{Generalization bound}},
ymin=10.5511979595808, ymax=16.5945650192399,
ytick style={color=color0},
y axis line style={line width = 1pt}
]
\addplot [semithick, color0, mark=*, mark size=3, mark options={solid}]
table {%
55 16.3198665165281
65 15.709834313028
75 15.1147231160606
85 14.3312949569999
95 13.6620870807006
105 13.1635882028691
115 12.679106535392
125 12.363489950726
135 12.1433536062648
145 11.8458991677855
155 11.7012741208891
165 11.590252513624
175 11.4810139814931
185 11.4274971201708
195 11.3217105349083
205 11.2291564696134
215 11.1168581985279
225 11.0256229972104
235 10.9789575313964
245 10.9382643441207
255 10.879300125757
265 10.8716633181911
275 10.8588421588122
285 10.8533043258825
295 10.8258964622926
};
\label{plot_one}
\addlegendentry{our bound}
\end{axis}

\begin{axis}[
axis y line=right,
every outer y axis line/.append style={color1},
every y tick label/.append style={color1},
height=6cm,
tick align=outside,
width=9cm,
x grid style={white!69.01960784313725!black},
xmin=43, xmax=307,
xtick pos=left,
xtick style={color=black},
y grid style={white!69.01960784313725!black},
ylabel={\textcolor{color1}{Test error (\%)}},
ymin=6.61175, ymax=11.41325,
ytick pos=right,
ytick style={color=color1},
y axis line style={line width = 1pt}
]
\addlegendimage{/pgfplots/refstyle=plot_one}\addlegendentry{our bound}
\addplot [semithick, color1, mark=triangle*, mark size=3, mark options={solid,rotate=180}]
table {%
55 11.195
65 9.494
75 10.064
85 9.347
95 8.368
105 8.494
115 8.393
125 7.749
135 7.936
145 7.923
155 7.567
165 7.595
175 7.447
185 7.274
195 7.342
205 7.231
215 6.941
225 7.007
235 6.951
245 6.849
255 6.921
265 6.906
275 6.84
285 6.889
295 6.83
};
\addlegendentry{test error}
\end{axis}

\end{tikzpicture}	  	
		\vspace{-1.5em}
	  	\caption{\scriptsize{Generalization bound}}
	  	\label{fig:bound}
	\end{subfigure}
\caption{
(a) Effective number of parameters (proportional to the generalization bound) compared with the one derived by the current state-of-the-art~\citep{arora2018stronger} 
for VGG-16. 
(b) Generalization bound vs test error for \ourcpnn-VGG-16. Two y-axes are applied for better visualization of the comparisons between the bound and the actual generalization/test error. 
}
\label{fig:bound_summary}
\end{figure}

\textbf{Generalization Bounds Correlated with Test Error.} We demonstrate how our generalization bound in Theorem~\ref{thm:cnn_thm} is practically useful in characterizing the generalizability during training.
In Figure~\ref{fig:bound}, 
\textbf{(1)} our calculated generalization bound matches well with the trend of the generalization error: after 140 epochs, the training error is almost zero but the test error continues to decrease in later epochs and our computed generalization bound captures these improvements especially well since epoch 150; 
\textbf{(2)} 
our calculated bound in Figure~\ref{fig:bound} for the well-trained \ourcpnn-VGG-16 at epoch 300 is around $10$ while the total number of parameters in this \ourcpnn-VGG-16 is around 14.7M.

\textbf{Compressibility of \ourNN: Property Evaluation.}
\label{exp:prop}
We evaluate and compare our proposed properties measuring compressibility, \emph{\tflong (\tfshort), \nblong (\nbshort) and layer cushion (\lcshort)}, on two different sets of models --- well-trained models with small generalization errors (thus expected to obtain small $\{\hat{R}^{(k)}\}_{k=1}^n$) vs. corrupted models with large generalization errors (thus expected to obtain large  $\{\hat{R}^{(k)}\}_{k=1}^n$). 
In Figure~\ref{fig:properties}\textbf{(a)}, the number of components $\{\hat{R}^{(k)}\}_{k=1}^n$ returned by the compression algorithm is much smaller for well-trained models than that for corrupted models, which indicates that well-trained models have higher compressibility compared to corrupted ones as expected in our theory. Moreover, in Figure~\ref{fig:properties}\textbf{(b-d)}, we can indeed tell if the model is trained using ``good'' data or corrupted data by evaluating our proposed properties.

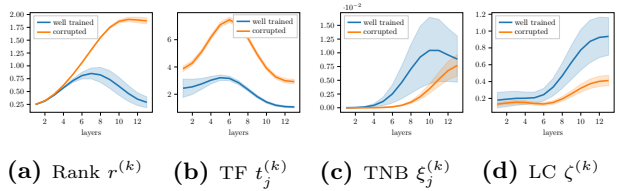
\begin{figure}[h]
\centering
	\begin{subfigure}[t]{0.225\columnwidth}
		\centering
\begin{tikzpicture}[scale=0.3]

\definecolor{color0}{rgb}{0.12156862745098,0.466666666666667,0.705882352941177}
\definecolor{color1}{rgb}{1,0.498039215686275,0.0549019607843137}

\begin{axis}[
height=6cm,
legend cell align={left},
legend style={at={(0.03,0.97)}, anchor=north west, draw=white!80.0!black},
tick align=outside,
tick pos=left,
width=7cm,
x grid style={white!69.01960784313725!black},
xlabel={layers},
xmin=0.4, xmax=13.6,
xtick style={color=black},
y grid style={white!69.01960784313725!black},
ymin=95.9359080579708, ymax=2035.96016146836,
ytick style={color=black},
ytick={0,250,500,750,1000,1250,1500,1750,2000,2250},
yticklabels={0.00,0.25,0.50,0.75,1.00,1.25,1.50,1.75,2.00,2.25}
]
\path [draw=color0, fill=color0, opacity=0.2]
(axis cs:1,249.470496460325)
--(axis cs:1,247.869121893568)
--(axis cs:2,312.367284545135)
--(axis cs:3,424.881274534894)
--(axis cs:4,557.305196064375)
--(axis cs:5,675.425604087054)
--(axis cs:6,745.807645836515)
--(axis cs:7,746.749152240661)
--(axis cs:8,678.049017870757)
--(axis cs:9,560.967624987871)
--(axis cs:10,427.978350345089)
--(axis cs:11,309.765283850459)
--(axis cs:12,226.083423934901)
--(axis cs:13,184.118828667534)
--(axis cs:13,395.637516544983)
--(axis cs:13,395.637516544983)
--(axis cs:12,476.725097964097)
--(axis cs:11,617.631165586481)
--(axis cs:10,777.164163179594)
--(axis cs:9,906.700819491146)
--(axis cs:8,970.558755034376)
--(axis cs:7,957.163052053696)
--(axis cs:6,874.909323834581)
--(axis cs:5,743.466908273738)
--(axis cs:4,588.335088783523)
--(axis cs:3,437.195105450879)
--(axis cs:2,316.684698483768)
--(axis cs:1,249.470496460325)
--cycle;

\path [draw=color1, fill=color1, opacity=0.2]
(axis cs:1,250.322124852704)
--(axis cs:1,250.318819923541)
--(axis cs:2,319.794176686693)
--(axis cs:3,448.198737643791)
--(axis cs:4,621.380514402018)
--(axis cs:5,828.041008650377)
--(axis cs:6,1060.20985265151)
--(axis cs:7,1306.49078846146)
--(axis cs:8,1542.04898213765)
--(axis cs:9,1728.80330465388)
--(axis cs:10,1835.37706502227)
--(axis cs:11,1861.45406238791)
--(axis cs:12,1840.01679559375)
--(axis cs:13,1814.90694011668)
--(axis cs:13,1935.21527372302)
--(axis cs:13,1935.21527372302)
--(axis cs:12,1947.77116100006)
--(axis cs:11,1947.7772408588)
--(axis cs:10,1896.98126869121)
--(axis cs:9,1767.6808817796)
--(axis cs:8,1563.54287026963)
--(axis cs:7,1316.79734730177)
--(axis cs:6,1064.45331148018)
--(axis cs:5,829.524040866063)
--(axis cs:4,621.814945297828)
--(axis cs:3,448.304102447599)
--(axis cs:2,319.81438310002)
--(axis cs:1,250.322124852704)
--cycle;

\addplot [line width=2.0pt, color0]
table {%
1 248.669809176947
2 314.525991514452
3 431.038189992886
4 572.820142423949
5 709.446256180396
6 810.358484835548
7 851.956102147179
8 824.303886452567
9 733.834222239509
10 602.571256762342
11 463.69822471847
12 351.404260949499
13 289.878172606259
};
\addlegendentry{well trained}
\addplot [line width=2.0pt, color1]
table {%
1 250.320472388123
2 319.804279893357
3 448.251420045695
4 621.597729849923
5 828.78252475822
6 1062.33158206585
7 1311.64406788162
8 1552.79592620364
9 1748.24209321674
10 1866.17916685674
11 1904.61565162335
12 1893.8939782969
13 1875.06110691985
};
\addlegendentry{corrupted}
\end{axis}

\end{tikzpicture}	  	
		\vspace{-1em}
	  	\caption{\scriptsize{Rank $r^{(k)}$}}\label{fig:compare_rank}
	\end{subfigure}
	\hfill
	\begin{subfigure}[t]{0.225\columnwidth}
		\centering
\begin{tikzpicture}[scale=0.3]

\definecolor{color0}{rgb}{0.12156862745098,0.466666666666667,0.705882352941177}
\definecolor{color1}{rgb}{1,0.498039215686275,0.0549019607843137}

\begin{axis}[
height=6cm,
legend cell align={left},
legend style={draw=white!80.0!black},
tick align=outside,
tick pos=left,
width=7cm,
x grid style={white!69.01960784313725!black},
xlabel={layers},
xmin=0.4, xmax=13.6,
xtick style={color=black},
y grid style={white!69.01960784313725!black},
ymin=0.676689391181921, ymax=7.98190552755115,
ytick style={color=black}
]
\path [draw=color0, fill=color0, opacity=0.2]
(axis cs:1,3.11147761506754)
--(axis cs:1,1.76958935982834)
--(axis cs:2,1.98659098076038)
--(axis cs:3,2.35782043796456)
--(axis cs:4,2.75422472572999)
--(axis cs:5,3.01035138131269)
--(axis cs:6,2.9981468650283)
--(axis cs:7,2.7001148471897)
--(axis cs:8,2.22158972814575)
--(axis cs:9,1.72626081702029)
--(axis cs:10,1.34475073573918)
--(axis cs:11,1.12304547168705)
--(axis cs:12,1.03152404895329)
--(axis cs:13,1.00874467010779)
--(axis cs:13,1.12412095391132)
--(axis cs:13,1.12412095391132)
--(axis cs:12,1.15425093012)
--(axis cs:11,1.26203475292156)
--(axis cs:10,1.51023890509413)
--(axis cs:9,1.92768339441211)
--(axis cs:8,2.46450097054077)
--(axis cs:7,2.98500364448577)
--(axis cs:6,3.32772713794017)
--(axis cs:5,3.4113136919585)
--(axis cs:4,3.30142091800641)
--(axis cs:3,3.15968233386837)
--(axis cs:2,3.10210337896098)
--(axis cs:1,3.11147761506754)
--cycle;

\path [draw=color1, fill=color1, opacity=0.2]
(axis cs:1,4.05151376909588)
--(axis cs:1,3.67865293372726)
--(axis cs:2,4.11017633255532)
--(axis cs:3,4.9438695482516)
--(axis cs:4,6.0013619608106)
--(axis cs:5,6.9108002890078)
--(axis cs:6,7.26847633256024)
--(axis cs:7,6.8929050872591)
--(axis cs:8,5.93665147278224)
--(axis cs:9,4.76980534159105)
--(axis cs:10,3.76200596259813)
--(axis cs:11,3.11916905657531)
--(axis cs:12,2.83560414353281)
--(axis cs:13,2.76343212348929)
--(axis cs:13,3.09022607450169)
--(axis cs:13,3.09022607450169)
--(axis cs:12,3.14647842570808)
--(axis cs:11,3.4135653506836)
--(axis cs:10,4.05768706125219)
--(axis cs:9,5.09030174880748)
--(axis cs:8,6.29265380654819)
--(axis cs:7,7.27346257896795)
--(axis cs:6,7.64985024862528)
--(axis cs:5,7.27457009373261)
--(axis cs:4,6.34682453355547)
--(axis cs:3,5.28652030501051)
--(axis cs:2,4.46711643593077)
--(axis cs:1,4.05151376909588)
--cycle;

\addplot [line width=2.0pt, color0]
table {%
1 2.44053348744794
2 2.54434717986068
3 2.75875138591647
4 3.0278228218682
5 3.2108325366356
6 3.16293700148423
7 2.84255924583773
8 2.34304534934326
9 1.8269721057162
10 1.42749482041665
11 1.19254011230431
12 1.09288748953664
13 1.06643281200956
};
\addlegendentry{well trained}
\addplot [line width=2.0pt, color1]
table {%
1 3.86508335141157
2 4.28864638424305
3 5.11519492663105
4 6.17409324718303
5 7.09268519137021
6 7.45916329059276
7 7.08318383311352
8 6.11465263966521
9 4.93005354519926
10 3.90984651192516
11 3.26636720362945
12 2.99104128462044
13 2.92682909899549
};
\addlegendentry{corrupted}
\end{axis}

\end{tikzpicture}	 
		\vspace{-1em} 	
	  	\caption{\scriptsize{\tfshort $\tf_j^{(k)}$}}\label{fig:compare_tf}
	\end{subfigure}
	\hfill
	\begin{subfigure}[t]{0.225\columnwidth}
		\centering
\begin{tikzpicture}[scale=0.3]

\definecolor{color0}{rgb}{0.12156862745098,0.466666666666667,0.705882352941177}
\definecolor{color1}{rgb}{1,0.498039215686275,0.0549019607843137}

\begin{axis}[
height=6cm,
legend cell align={left},
legend style={at={(0.03,0.97)}, anchor=north west, draw=white!80.0!black},
tick align=outside,
tick pos=left,
width=7cm,
x grid style={white!69.01960784313725!black},
xlabel={layers},
xmin=0.4, xmax=13.6,
xtick style={color=black},
y grid style={white!69.01960784313725!black},
ymin=-0.000824013055236008, ymax=0.0173060787351642,
ytick style={color=black},
ytick={-0.0025,0,0.0025,0.005,0.0075,0.01,0.0125,0.015,0.0175},
yticklabels={−0.25,0.00,0.25,0.50,0.75,1.00,1.25,1.50,1.75}
]
\path [draw=color0, fill=color0, opacity=0.2]
(axis cs:1,1.53902596291674e-05)
--(axis cs:1,2.66226851258948e-06)
--(axis cs:2,1.11351921472824e-05)
--(axis cs:3,4.40221991176032e-05)
--(axis cs:4,0.00014359983933723)
--(axis cs:5,0.00038771649488702)
--(axis cs:6,0.000871032248657139)
--(axis cs:7,0.00164252806247223)
--(axis cs:8,0.00263179545278577)
--(axis cs:9,0.00363474016678337)
--(axis cs:10,0.00439696511388856)
--(axis cs:11,0.00475781482149379)
--(axis cs:12,0.00476704349688634)
--(axis cs:13,0.00466019643860711)
--(axis cs:13,0.0130474327990112)
--(axis cs:13,0.0130474327990112)
--(axis cs:12,0.0144609516034074)
--(axis cs:11,0.0160997268083117)
--(axis cs:10,0.0164819836537824)
--(axis cs:9,0.0148380452523704)
--(axis cs:8,0.011503239324214)
--(axis cs:7,0.00757640234967034)
--(axis cs:6,0.00419284278449118)
--(axis cs:5,0.00193270957121377)
--(axis cs:4,0.000738121397133778)
--(axis cs:3,0.000233181503658105)
--(axis cs:2,6.10811174860585e-05)
--(axis cs:1,1.53902596291674e-05)
--cycle;

\path [draw=color1, fill=color1, opacity=0.2]
(axis cs:1,1.37079868835902e-07)
--(axis cs:1,8.20261458199086e-08)
--(axis cs:2,4.86742732813909e-07)
--(axis cs:3,2.62028316249816e-06)
--(axis cs:4,1.12131447222204e-05)
--(axis cs:5,4.03308783883031e-05)
--(axis cs:6,0.000124141775072579)
--(axis cs:7,0.000330041618538635)
--(axis cs:8,0.000759639693864529)
--(axis cs:9,0.0015122630734359)
--(axis cs:10,0.00260410966887317)
--(axis cs:11,0.00388910178162672)
--(axis cs:12,0.00506067560052088)
--(axis cs:13,0.00576459428216655)
--(axis cs:13,0.00964982429365226)
--(axis cs:13,0.00964982429365226)
--(axis cs:12,0.00844018087728082)
--(axis cs:11,0.00645185501980959)
--(axis cs:10,0.00430207732271953)
--(axis cs:9,0.00249504103616969)
--(axis cs:8,0.00125531771330198)
--(axis cs:7,0.000547193224414612)
--(axis cs:6,0.000206541597406848)
--(axis cs:5,6.72620406203852e-05)
--(axis cs:4,1.87209568113716e-05)
--(axis cs:3,4.38996728667519e-06)
--(axis cs:2,8.13344055094636e-07)
--(axis cs:1,1.37079868835902e-07)
--cycle;

\addplot [line width=2.0pt, color0]
table {%
1 9.02626407087846e-06
2 3.61081548166704e-05
3 0.000138601851387854
4 0.000440860618235504
5 0.00116021303305039
6 0.00253193751657416
7 0.00460946520607129
8 0.00706751738849989
9 0.00923639270957686
10 0.0104394743838355
11 0.0104287708149028
12 0.00961399755014686
13 0.00885381461880914
};
\addlegendentry{well trained}
\addplot [line width=2.0pt, color1]
table {%
1 1.09553007327905e-07
2 6.50043393954273e-07
3 3.50512522458667e-06
4 1.4967050766796e-05
5 5.37964595043441e-05
6 0.000165341686239714
7 0.000438617421476623
8 0.00100747870358325
9 0.0020036520548028
10 0.00345309349579635
11 0.00517047840071815
12 0.00675042823890085
13 0.0077072092879094
};
\addlegendentry{corrupted}
\end{axis}

\end{tikzpicture}
		\vspace{-1em}
	  	\caption{\scriptsize{\nbshort $\nb_j^{(k)}$}}\label{fig:compare_tnb}
	\end{subfigure}
	\hfill
	\begin{subfigure}[t]{0.225\columnwidth}
		\centering
\begin{tikzpicture}[scale=0.3]

\definecolor{color0}{rgb}{0.12156862745098,0.466666666666667,0.705882352941177}
\definecolor{color1}{rgb}{1,0.498039215686275,0.0549019607843137}

\begin{axis}[
height=6cm,
legend cell align={left},
legend style={at={(0.03,0.97)}, anchor=north west, draw=white!80.0!black},
tick align=outside,
tick pos=left,
width=7cm,
x grid style={white!69.01960784313725!black},
xlabel={layers},
xmin=0.4, xmax=13.6,
xtick style={color=black},
y grid style={white!69.01960784313725!black},
ymin=0.0348557833582163, ymax=1.21884182803333,
ytick style={color=black},
ytick={0,0.2,0.4,0.6,0.8,1,1.2,1.4},
yticklabels={0.0,0.2,0.4,0.6,0.8,1.0,1.2,1.4}
]
\path [draw=color0, fill=color0, opacity=0.2]
(axis cs:1,0.269719034433365)
--(axis cs:1,0.0886733308434486)
--(axis cs:2,0.101778842508793)
--(axis cs:3,0.117614582180977)
--(axis cs:4,0.12878143787384)
--(axis cs:5,0.141245678067207)
--(axis cs:6,0.170574516057968)
--(axis cs:7,0.22951602935791)
--(axis cs:8,0.318743884563446)
--(axis cs:9,0.426498681306839)
--(axis cs:10,0.534554243087769)
--(axis cs:11,0.624752044677734)
--(axis cs:12,0.685113430023193)
--(axis cs:13,0.71358597278595)
--(axis cs:13,1.16064357757568)
--(axis cs:13,1.16064357757568)
--(axis cs:12,1.1650242805481)
--(axis cs:11,1.13013088703156)
--(axis cs:10,1.0173534154892)
--(axis cs:9,0.832127630710602)
--(axis cs:8,0.619668424129486)
--(axis cs:7,0.435828536748886)
--(axis cs:6,0.319092273712158)
--(axis cs:5,0.274275302886963)
--(axis cs:4,0.274758607149124)
--(axis cs:3,0.283312439918518)
--(axis cs:2,0.27944341301918)
--(axis cs:1,0.269719034433365)
--cycle;

\path [draw=color1, fill=color1, opacity=0.2]
(axis cs:1,0.150583237409592)
--(axis cs:1,0.113911338150501)
--(axis cs:2,0.124691992998123)
--(axis cs:3,0.133493185043335)
--(axis cs:4,0.130237728357315)
--(axis cs:5,0.119351603090763)
--(axis cs:6,0.11540262401104)
--(axis cs:7,0.130473926663399)
--(axis cs:8,0.167038708925247)
--(axis cs:9,0.21879443526268)
--(axis cs:10,0.273710668087006)
--(axis cs:11,0.317655593156815)
--(axis cs:12,0.342094480991364)
--(axis cs:13,0.350226551294327)
--(axis cs:13,0.473463535308838)
--(axis cs:13,0.473463535308838)
--(axis cs:12,0.463163495063782)
--(axis cs:11,0.429913192987442)
--(axis cs:10,0.368514031171799)
--(axis cs:9,0.291701257228851)
--(axis cs:8,0.220419570803642)
--(axis cs:7,0.171347752213478)
--(axis cs:6,0.151937887072563)
--(axis cs:5,0.157805278897285)
--(axis cs:4,0.172511711716652)
--(axis cs:3,0.176814839243889)
--(axis cs:2,0.165004163980484)
--(axis cs:1,0.150583237409592)
--cycle;

\addplot [line width=2.0pt, color0]
table {%
1 0.179196193814278
2 0.190611124038696
3 0.200463518500328
4 0.201770007610321
5 0.207760497927666
6 0.244833394885063
7 0.332672268152237
8 0.469206154346466
9 0.629313170909882
10 0.775953829288483
11 0.877441465854645
12 0.925068795681
13 0.937114775180817
};
\addlegendentry{well trained}
\addplot [line width=2.0pt, color1]
table {%
1 0.132247284054756
2 0.144848078489304
3 0.155154019594193
4 0.151374712586403
5 0.138578444719315
6 0.133670255541801
7 0.150910839438438
8 0.193729132413864
9 0.255247831344604
10 0.321112334728241
11 0.373784393072128
12 0.402628988027573
13 0.411845058202744
};
\addlegendentry{corrupted}
\end{axis}

\end{tikzpicture}
		\vspace{-1em}	  	
	  	\caption{\scriptsize{\lcshort $\lc^{(k)}$}}\label{fig:compare_lc}
	\end{subfigure}
	\hfill
	\vspace{-0.25em}
  	\caption{Comparison of our proposed properties across layers between well-trained and corrupted \ourcpnn-VGG-16. The statistics are obtained from 200 models trained under the same optimization settings.} 
	\label{fig:properties}
	\vspace{-0.25em}
\end{figure}

We further apply Algorithm~\ref{alg:cnn_compress} to these well-trained and corrupted models to investigate the consistency between the compression performance of Algorithm~\ref{alg:cnn_compress} and our theoretical results: on average, Algorithm~\ref{alg:cnn_compress} achieves a 31.83\% compression rate on the well-trained models, but only an 89.7\% compression rate on the corrupted models (lower compression rate is better as it implies a smaller generalization error bound). Clearly, the low-rank structures in well-trained models allow them to be compressed much further, consistent with our theoretical analysis of Algorithm~\ref{alg:cnn_compress}.

\subsection{Generalization Improvement on Real Data Experiments}
\label{sec:exp_gen_imp}

\textbf{Expressive Power of Neural Networks with \ourNNlong.} As shown in Table~\ref{table:exp_power}, neural networks equipped with \ourNNlong maintain competitive training and test accuracies. 

\textbf{Generalization Improvements under Label Noise.}
The memorization effect is directly linked to the deteriorated generalization performance of the network~\citep{zhang2016understanding}. Therefore we study how our proposed \ourNN structure affects the generalizability of a neural network with presence of strong memorization effect --- under label noise setting.
We assign random labels to a proportion of the training data and train the neural network until convergence. Then we test the network's performance on the uncorrupted test data.
As shown in Table~\ref{tbl:memorization},  \ourcpnn-VGG consistently achieves better generalization performance compared to the traditional VGG under various label corruption ratios.

Our \ourNN, combined with co-teaching (CT)~\citep{han2018co} (the SOTA method for defeating label noise) further improves its performance as shown in Table~\ref{tbl:mnist} where we also compare our method CT+\ourNN against other different label-noise methods~\citep{han2018co}.
Besides, in Figure~\ref{fig:pf_sym_cr}, our method CT+\ourNN consistently outperforms the SOTA method (CT) with various choices of number of components.

\subsection{
\ourNN Is Natural for Compression}
\label{sec:exp_cpl_compression}

Applying \ourNN for neural network compression is extensively studied in~\cite{su2019tensorial}, therefore we focus on explaining why \ourNN is natural for compression and analyzing the compressibility of {\ourNN}s.

\textbf{Low Rankness in Neural Networks with \ourNN vs Traditional Neural Networks}. The low rankness of a \ourcpnn-VGG and a traditional VGG is demonstrated by Figure~\ref{eig_spectrum} where we display the ratios of the number of components with amplitudes above a given threshold 0.2. We clearly see that VGG with \ourNN exhibits low rankness consistently for all layers while the traditional VGG is not low-rank. Notice that the CP spectrum in each \ourNN is normalized by dividing the largest amplitude and the CP components of traditional VGG are obtained via explicit CP decompositions with reconstruction error set to 1e-3.

\begin{figure}[!htbp]
\centering
\vspace{-0.25em}
\begin{subfigure}[b]{0.495 \columnwidth}
\centering
\begin{tikzpicture}[scale=0.45]


\definecolor{color1}{rgb}{0.921875,0.5859375,0.640625}
\definecolor{color0}{rgb}{0.36328125, 0.32421875, 0.3671875}

\begin{axis}[
height=5cm,
legend cell align={left},
legend style={draw=white!80.0!black},
tick align=outside,
tick pos=left,
width=9cm,
x grid style={white!69.01960784313725!black},
xlabel={layers $k$},
xmin=0.0149999999999999, xmax=13.985,
xtick style={color=black},
y grid style={white!69.01960784313725!black},
ylabel={\% of components $\ge$ 0.2},
ymin=0, ymax=1.05,
ytick style={color=black}
]
\draw[fill=color0,draw opacity=0] (axis cs:0.65,0) rectangle (axis cs:1,0.791666666666667);
\addlegendimage{ybar,ybar legend,fill=color0,draw opacity=0};
\addlegendentry{VGG}

\draw[fill=color0,draw opacity=0] (axis cs:1.65,0) rectangle (axis cs:2,1);
\draw[fill=color0,draw opacity=0] (axis cs:2.65,0) rectangle (axis cs:3,1);
\draw[fill=color0,draw opacity=0] (axis cs:3.65,0) rectangle (axis cs:4,1);
\draw[fill=color0,draw opacity=0] (axis cs:4.65,0) rectangle (axis cs:5,1);
\draw[fill=color0,draw opacity=0] (axis cs:5.65,0) rectangle (axis cs:6,1);
\draw[fill=color0,draw opacity=0] (axis cs:6.65,0) rectangle (axis cs:7,1);
\draw[fill=color0,draw opacity=0] (axis cs:7.65,0) rectangle (axis cs:8,1);
\draw[fill=color0,draw opacity=0] (axis cs:8.65,0) rectangle (axis cs:9,1);
\draw[fill=color0,draw opacity=0] (axis cs:9.65,0) rectangle (axis cs:10,1);
\draw[fill=color0,draw opacity=0] (axis cs:10.65,0) rectangle (axis cs:11,1);
\draw[fill=color0,draw opacity=0] (axis cs:11.65,0) rectangle (axis cs:12,1);
\draw[fill=color0,draw opacity=0] (axis cs:12.65,0) rectangle (axis cs:13,1);
\draw[fill=color1,draw opacity=0] (axis cs:1,0) rectangle (axis cs:1.35,0.545454545454545);
\addlegendimage{ybar,ybar legend,fill=color1,draw opacity=0};
\addlegendentry{\ourcpnn-VGG}

\draw[fill=color1,draw opacity=0] (axis cs:2,0) rectangle (axis cs:2.35,0.104089219330855);
\draw[fill=color1,draw opacity=0] (axis cs:3,0) rectangle (axis cs:3.35,0.448087431693989);
\draw[fill=color1,draw opacity=0] (axis cs:4,0) rectangle (axis cs:4.35,0.474820143884892);
\draw[fill=color1,draw opacity=0] (axis cs:5,0) rectangle (axis cs:5.35,0.513333333333333);
\draw[fill=color1,draw opacity=0] (axis cs:6,0) rectangle (axis cs:6.35,0.226148409893993);
\draw[fill=color1,draw opacity=0] (axis cs:7,0) rectangle (axis cs:7.35,0.147526501766784);
\draw[fill=color1,draw opacity=0] (axis cs:8,0) rectangle (axis cs:8.35,0.0322793148880105);
\draw[fill=color1,draw opacity=0] (axis cs:9,0) rectangle (axis cs:9.35,0.00306614104248795);
\draw[fill=color1,draw opacity=0] (axis cs:10,0) rectangle (axis cs:10.35,0.00700832238282961);
\draw[fill=color1,draw opacity=0] (axis cs:11,0) rectangle (axis cs:11.35,0.011826544021025);
\draw[fill=color1,draw opacity=0] (axis cs:12,0) rectangle (axis cs:12.35,0.00569426193604906);
\draw[fill=color1,draw opacity=0] (axis cs:13,0) rectangle (axis cs:13.35,0.0135786246167324);
\end{axis}

\end{tikzpicture}
	\caption{VGG-16}
	\label{fig:vgg_eig_layers}
\end{subfigure}%
\hfill
\begin{subfigure}[b]{0.495\columnwidth}
\begin{tikzpicture}[scale=0.45]


\definecolor{color1}{rgb}{0.921875,0.5859375,0.640625}
\definecolor{color0}{rgb}{0.36328125, 0.32421875, 0.3671875}

\begin{axis}[
height=5cm,
legend cell align={left},
legend style={draw=white!80.0!black},
tick align=outside,
tick pos=left,
width=9cm,
x grid style={white!69.01960784313725!black},
xlabel={layers $k$},
xmin=-0.735, xmax=29.735,
xtick style={color=black},
y grid style={white!69.01960784313725!black},
ylabel={\% of components $\ge$ 0.2},
ymin=0, ymax=1.05,
ytick style={color=black}
]
\draw[fill=color0,draw opacity=0] (axis cs:0.65,0) rectangle (axis cs:1,0.944444444444444);
\addlegendimage{ybar,ybar legend,fill=color0,draw opacity=0};
\addlegendentry{WRN}

\draw[fill=color0,draw opacity=0] (axis cs:1.65,0) rectangle (axis cs:2,0.937007874015748);
\draw[fill=color0,draw opacity=0] (axis cs:2.65,0) rectangle (axis cs:3,1);
\draw[fill=color0,draw opacity=0] (axis cs:3.65,0) rectangle (axis cs:4,0.8);
\draw[fill=color0,draw opacity=0] (axis cs:4.65,0) rectangle (axis cs:5,1);
\draw[fill=color0,draw opacity=0] (axis cs:5.65,0) rectangle (axis cs:6,1);
\draw[fill=color0,draw opacity=0] (axis cs:6.65,0) rectangle (axis cs:7,1);
\draw[fill=color0,draw opacity=0] (axis cs:7.65,0) rectangle (axis cs:8,1);
\draw[fill=color0,draw opacity=0] (axis cs:8.65,0) rectangle (axis cs:9,1);
\draw[fill=color0,draw opacity=0] (axis cs:9.65,0) rectangle (axis cs:10,1);
\draw[fill=color0,draw opacity=0] (axis cs:10.65,0) rectangle (axis cs:11,1);
\draw[fill=color0,draw opacity=0] (axis cs:11.65,0) rectangle (axis cs:12,1);
\draw[fill=color0,draw opacity=0] (axis cs:12.65,0) rectangle (axis cs:13,1);
\draw[fill=color0,draw opacity=0] (axis cs:13.65,0) rectangle (axis cs:14,1);
\draw[fill=color0,draw opacity=0] (axis cs:14.65,0) rectangle (axis cs:15,1);
\draw[fill=color0,draw opacity=0] (axis cs:15.65,0) rectangle (axis cs:16,1);
\draw[fill=color0,draw opacity=0] (axis cs:16.65,0) rectangle (axis cs:17,1);
\draw[fill=color0,draw opacity=0] (axis cs:17.65,0) rectangle (axis cs:18,1);
\draw[fill=color0,draw opacity=0] (axis cs:18.65,0) rectangle (axis cs:19,1);
\draw[fill=color0,draw opacity=0] (axis cs:19.65,0) rectangle (axis cs:20,1);
\draw[fill=color0,draw opacity=0] (axis cs:20.65,0) rectangle (axis cs:21,1);
\draw[fill=color0,draw opacity=0] (axis cs:21.65,0) rectangle (axis cs:22,1);
\draw[fill=color0,draw opacity=0] (axis cs:22.65,0) rectangle (axis cs:23,1);
\draw[fill=color0,draw opacity=0] (axis cs:23.65,0) rectangle (axis cs:24,1);
\draw[fill=color0,draw opacity=0] (axis cs:24.65,0) rectangle (axis cs:25,1);
\draw[fill=color0,draw opacity=0] (axis cs:25.65,0) rectangle (axis cs:26,1);
\draw[fill=color0,draw opacity=0] (axis cs:26.65,0) rectangle (axis cs:27,0.998604813393791);
\draw[fill=color0,draw opacity=0] (axis cs:27.65,0) rectangle (axis cs:28,0.951517265434252);
\draw[fill=color1,draw opacity=0] (axis cs:1,0) rectangle (axis cs:1.35,0.266666666666667);
\addlegendimage{ybar,ybar legend,fill=color1,draw opacity=0};
\addlegendentry{\ourcpnn-WRN}

\draw[fill=color1,draw opacity=0] (axis cs:2,0) rectangle (axis cs:2.35,0.137096774193548);
\draw[fill=color1,draw opacity=0] (axis cs:3,0) rectangle (axis cs:3.35,0.07);
\draw[fill=color1,draw opacity=0] (axis cs:4,0) rectangle (axis cs:4.35,0.714285714285714);
\draw[fill=color1,draw opacity=0] (axis cs:5,0) rectangle (axis cs:5.35,0.00142857142857143);
\draw[fill=color1,draw opacity=0] (axis cs:6,0) rectangle (axis cs:6.35,0.00142857142857143);
\draw[fill=color1,draw opacity=0] (axis cs:7,0) rectangle (axis cs:7.35,0.00142857142857143);
\draw[fill=color1,draw opacity=0] (axis cs:8,0) rectangle (axis cs:8.35,0.00142857142857143);
\draw[fill=color1,draw opacity=0] (axis cs:9,0) rectangle (axis cs:9.35,0.00142857142857143);
\draw[fill=color1,draw opacity=0] (axis cs:10,0) rectangle (axis cs:10.35,0.00571428571428571);
\draw[fill=color1,draw opacity=0] (axis cs:11,0) rectangle (axis cs:11.35,0.0222929936305732);
\draw[fill=color1,draw opacity=0] (axis cs:12,0) rectangle (axis cs:12.35,0.0380281690140845);
\draw[fill=color1,draw opacity=0] (axis cs:13,0) rectangle (axis cs:13.35,0.0849056603773585);
\draw[fill=color1,draw opacity=0] (axis cs:14,0) rectangle (axis cs:14.35,0.00985915492957747);
\draw[fill=color1,draw opacity=0] (axis cs:15,0) rectangle (axis cs:15.35,0.0105633802816901);
\draw[fill=color1,draw opacity=0] (axis cs:16,0) rectangle (axis cs:16.35,0.0133802816901408);
\draw[fill=color1,draw opacity=0] (axis cs:17,0) rectangle (axis cs:17.35,0.0176056338028169);
\draw[fill=color1,draw opacity=0] (axis cs:18,0) rectangle (axis cs:18.35,0.0359154929577465);
\draw[fill=color1,draw opacity=0] (axis cs:19,0) rectangle (axis cs:19.35,0.0147887323943662);
\draw[fill=color1,draw opacity=0] (axis cs:20,0) rectangle (axis cs:20.35,0.0846477392218717);
\draw[fill=color1,draw opacity=0] (axis cs:21,0) rectangle (axis cs:21.35,0.0863938440013991);
\draw[fill=color1,draw opacity=0] (axis cs:22,0) rectangle (axis cs:22.35,0.187793427230047);
\draw[fill=color1,draw opacity=0] (axis cs:23,0) rectangle (axis cs:23.35,0.0108429520811473);
\draw[fill=color1,draw opacity=0] (axis cs:24,0) rectangle (axis cs:24.35,0.00909408884225254);
\draw[fill=color1,draw opacity=0] (axis cs:25,0) rectangle (axis cs:25.35,0.0101434067855894);
\draw[fill=color1,draw opacity=0] (axis cs:26,0) rectangle (axis cs:26.35,0.0101434067855894);
\draw[fill=color1,draw opacity=0] (axis cs:27,0) rectangle (axis cs:27.35,0.0118922700244841);
\draw[fill=color1,draw opacity=0] (axis cs:28,0) rectangle (axis cs:28.35,0.00384749912556838);
\end{axis}

\end{tikzpicture}
	\caption{WRN-28-10}
	\label{fig:wrn_eig_layers}
\end{subfigure}%
\vspace{-0.25em}
\caption{Comparison of low rankness (compressibility) across layers between neural networks with \ourNN and standard neural networks  
}
\label{eig_spectrum}
\vspace{-0.75em}
\end{figure}
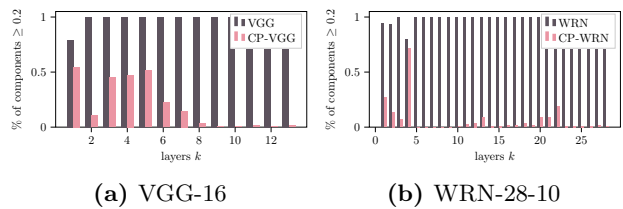

\textbf{No Fine-tuning Needed for \ourNN.} Many works using tensor methods for neural network compression require computationally expensive fine-tuning (e.g. 200 epochs end-to-end training on the compressed networks) to recover the compressed network's test performance~\cite{jaderberg2014speeding, denton2014exploiting, lebedev2014speeding, kim2015compression, garipov2016ultimate, wang2018wide, su2019tensorial}. 
However, the compression we perform does not require any fine tuning since it directly prunes out the components with amplitudes below some given threshold. 
\emph{In experiments, we compress a \ourcpnn-WRN-28-10, which has the same number of parameters as WRN-28-10, by 8$\times$ with only 0.56\% performance drop on CIFAR10 image classification.}
The full compression results for \ourcpnn-WRN-28-10 under different cutting-off thresholds are shown in Table~\ref{tbl:compression}, where components whose amplitudes are under the cutting-off threshold are pruned.

\section{Conclusion and Discussion}
In this work, we derive a practical compression-based generalization bound via the proposed layerwise structure \ourNNlong, and demonstrate the effectiveness of using tensor methods in theoretical analyses of deep neural networks.
With a series of benchmark experiments, we show the practical usage of our generalization bound and the effectiveness of our proposed structure \ourNN in terms of compression and generalization.
A possible future direction is studying the effectiveness of other tensor decomposition methods such as Tucker or Tensor Train. 

\section*{Acknowledgement}
This research was supported by startup fund from Department of Computer Science of University of Maryland, National Science Foundation IIS-1850220 CRII Award 030742- 00001, DOD-DARPA-Defense Advanced Research Projects Agency Guaranteeing AI Robustness against Deception (GARD), Laboratory for Physical Sciences at University of Maryland.
This research was also supported in part by JSPS Kakenhi (26280009, 15H05707 and 18H03201), Japan Digital Design and JST-CREST.
Huang was also supported by Adobe, Capital One and JP Morgan faculty fellowships.
We thank Ziyin Liu for supporting this research with great advice and efforts.
We thank Jin-peng Liu, Kai Wang, and Dongruo Zhou for helpful discussions and comments.
We thank Jingxiao Zheng for supporting additional computing resources.

\bibliographystyle{plainnat}
\bibliography{supp_bib}

\begin{thebibliography}{58}
\providecommand{\natexlab}[1]{#1}
\providecommand{\url}[1]{\texttt{#1}}
\expandafter\ifx\csname urlstyle\endcsname\relax
  \providecommand{\doi}[1]{doi: #1}\else
  \providecommand{\doi}{doi: \begingroup \urlstyle{rm}\Url}\fi

\bibitem[Anandkumar et~al.(2014{\natexlab{a}})Anandkumar, Ge, and
  Janzamin]{anandkumar2014guaranteed}
Anima Anandkumar, Rong Ge, and Majid Janzamin.
\newblock Guaranteed non-orthogonal tensor decomposition via alternating rank-1
  updates.
\newblock \emph{arXiv preprint arXiv:1402.5180}, 2014{\natexlab{a}}.

\bibitem[Anandkumar et~al.(2014{\natexlab{b}})Anandkumar, Ge, Hsu, Kakade, and
  Telgarsky]{anandkumar2014tensor}
Animashree Anandkumar, Rong Ge, Daniel Hsu, Sham~M Kakade, and Matus Telgarsky.
\newblock Tensor decompositions for learning latent variable models.
\newblock \emph{The Journal of Machine Learning Research}, 15\penalty0
  (1):\penalty0 2773--2832, 2014{\natexlab{b}}.

\bibitem[Anandkumar et~al.(2015)Anandkumar, Ge, and
  Janzamin]{anandkumar2014provable}
Animashree Anandkumar, Rong Ge, and Majid Janzamin.
\newblock Learning overcomplete latent variable models through tensor methods.
\newblock In \emph{Conference on Learning Theory (COLT)}, June 2015.

\bibitem[Arora et~al.(2018)Arora, Ge, Neyshabur, and Zhang]{arora2018stronger}
Sanjeev Arora, Rong Ge, Behnam Neyshabur, and Yi~Zhang.
\newblock Stronger generalization bounds for deep nets via a compression
  approach.
\newblock 2018.

\bibitem[Bartlett and Mendelson(2002)]{bartlett2002rademacher}
Peter~L Bartlett and Shahar Mendelson.
\newblock Rademacher and gaussian complexities: Risk bounds and structural
  results.
\newblock \emph{Journal of Machine Learning Research}, 3\penalty0
  (Nov):\penalty0 463--482, 2002.

\bibitem[Bartlett et~al.(1999)Bartlett, Maiorov, and Meir]{bartlett1999almost}
Peter~L Bartlett, Vitaly Maiorov, and Ron Meir.
\newblock Almost linear vc dimension bounds for piecewise polynomial networks.
\newblock In \emph{Advances in Neural Information Processing Systems}, pages
  190--196, 1999.

\bibitem[Bartlett et~al.(2017)Bartlett, Foster, and
  Telgarsky]{bartlett2017spectrally}
Peter~L Bartlett, Dylan~J Foster, and Matus~J Telgarsky.
\newblock Spectrally-normalized margin bounds for neural networks.
\newblock In \emph{Advances in Neural Information Processing Systems}, pages
  6240--6249, 2017.

\bibitem[Denton et~al.(2014)Denton, Zaremba, Bruna, LeCun, and
  Fergus]{denton2014exploiting}
Emily~L Denton, Wojciech Zaremba, Joan Bruna, Yann LeCun, and Rob Fergus.
\newblock Exploiting linear structure within convolutional networks for
  efficient evaluation.
\newblock In \emph{Advances in neural information processing systems}, pages
  1269--1277, 2014.

\bibitem[Du and Lee(2018)]{du2018power}
Simon~S Du and Jason~D Lee.
\newblock On the power of over-parametrization in neural networks with
  quadratic activation.
\newblock \emph{arXiv preprint arXiv:1803.01206}, 2018.

\bibitem[Dziugaite and Roy(2017)]{dziugaite2017computing}
Gintare~Karolina Dziugaite and Daniel~M Roy.
\newblock Computing nonvacuous generalization bounds for deep (stochastic)
  neural networks with many more parameters than training data.
\newblock \emph{arXiv preprint arXiv:1703.11008}, 2017.

\bibitem[Garipov et~al.(2016)Garipov, Podoprikhin, Novikov, and
  Vetrov]{garipov2016ultimate}
Timur Garipov, Dmitry Podoprikhin, Alexander Novikov, and Dmitry Vetrov.
\newblock Ultimate tensorization: compressing convolutional and fc layers
  alike.
\newblock \emph{arXiv preprint arXiv:1611.03214}, 2016.

\bibitem[Golowich et~al.(2017)Golowich, Rakhlin, and Shamir]{golowich2017size}
Noah Golowich, Alexander Rakhlin, and Ohad Shamir.
\newblock Size-independent sample complexity of neural networks.
\newblock \emph{arXiv preprint arXiv:1712.06541}, 2017.

\bibitem[Golowich et~al.(2018)Golowich, Rakhlin, and Shamir]{golowich2018}
Noah Golowich, Alexander Rakhlin, and Ohad Shamir.
\newblock Size-independent sample complexity of neural networks.
\newblock In S\'ebastien Bubeck, Vianney Perchet, and Philippe Rigollet,
  editors, \emph{Proceedings of the 31st Conference On Learning Theory},
  volume~75 of \emph{Proceedings of Machine Learning Research}, pages 297--299.
  PMLR, 06--09 Jul 2018.
\newblock URL \url{http://proceedings.mlr.press/v75/golowich18a.html}.

\bibitem[Gonen and Shalev-Shwartz(2017)]{Gonen2017}
Alon Gonen and Shai Shalev-Shwartz.
\newblock Fast rates for empirical risk minimization of strict saddle problems.
\newblock In \emph{COLT}, 2017.

\bibitem[Han et~al.(2018)Han, Yao, Yu, Niu, Xu, Hu, Tsang, and
  Sugiyama]{han2018co}
Bo~Han, Quanming Yao, Xingrui Yu, Gang Niu, Miao Xu, Weihua Hu, Ivor Tsang, and
  Masashi Sugiyama.
\newblock Co-teaching: Robust training of deep neural networks with extremely
  noisy labels.
\newblock In \emph{Advances in Neural Information Processing Systems}, pages
  8536--8546, 2018.

\bibitem[Han et~al.(2015)Han, Mao, and Dally]{han2015deep}
Song Han, Huizi Mao, and William~J Dally.
\newblock Deep compression: Compressing deep neural networks with pruning,
  trained quantization and huffman coding.
\newblock \emph{arXiv preprint arXiv:1510.00149}, 2015.

\bibitem[Hardt et~al.(2016)Hardt, Recht, and Singer]{Hardt2016}
Moritz Hardt, Benjamin Recht, and Yoram Singer.
\newblock Train faster, generalize better: Stability of stochastic gradient
  descent.
\newblock In \emph{Proceedings of the 33rd International Conference on
  International Conference on Machine Learning - Volume 48}, ICML'16, pages
  1225--1234. JMLR.org, 2016.
\newblock URL \url{http://dl.acm.org/citation.cfm?id=3045390.3045520}.

\bibitem[Harvey et~al.(2017)Harvey, Liaw, and Mehrabian]{harvey2017nearly}
Nick Harvey, Christopher Liaw, and Abbas Mehrabian.
\newblock Nearly-tight vc-dimension bounds for piecewise linear neural
  networks.
\newblock In \emph{Conference on Learning Theory}, pages 1064--1068, 2017.

\bibitem[He et~al.(2016)He, Zhang, Ren, and Sun]{he2016deep}
Kaiming He, Xiangyu Zhang, Shaoqing Ren, and Jian Sun.
\newblock Deep residual learning for image recognition.
\newblock In \emph{Proceedings of the IEEE Conference on Computer Vision and
  Pattern Recognition}, pages 770--778, 2016.

\bibitem[Hornik et~al.(1989)Hornik, Stinchcombe, and
  White]{hornik1989multilayer}
Kurt Hornik, Maxwell Stinchcombe, and Halbert White.
\newblock Multilayer feedforward networks are universal approximators.
\newblock \emph{Neural networks}, 2\penalty0 (5):\penalty0 359--366, 1989.

\bibitem[Howard et~al.(2017)Howard, Zhu, Chen, Kalenichenko, Wang, Weyand,
  Andreetto, and Adam]{howard2017mobilenets}
Andrew~G Howard, Menglong Zhu, Bo~Chen, Dmitry Kalenichenko, Weijun Wang,
  Tobias Weyand, Marco Andreetto, and Hartwig Adam.
\newblock Mobilenets: Efficient convolutional neural networks for mobile vision
  applications.
\newblock \emph{arXiv preprint arXiv:1704.04861}, 2017.

\bibitem[Huang et~al.(2015)Huang, Niranjan, Hakeem, and
  Anandkumar]{huang2015online}
Furong Huang, UN~Niranjan, Mohammad~Umar Hakeem, and Animashree Anandkumar.
\newblock Online tensor methods for learning latent variable models.
\newblock \emph{Journal of Machine Learning Research}, 16:\penalty0 2797--2835,
  2015.

\bibitem[Jaderberg et~al.(2014)Jaderberg, Vedaldi, and
  Zisserman]{jaderberg2014speeding}
Max Jaderberg, Andrea Vedaldi, and Andrew Zisserman.
\newblock Speeding up convolutional neural networks with low rank expansions.
\newblock \emph{arXiv preprint arXiv:1405.3866}, 2014.

\bibitem[Jakubovitz et~al.(2018)Jakubovitz, Giryes, and
  Rodrigues]{jakubovitz2018generalization}
Daniel Jakubovitz, Raja Giryes, and Miguel~RD Rodrigues.
\newblock Generalization error in deep learning.
\newblock \emph{arXiv preprint arXiv:1808.01174}, 2018.

\bibitem[Jiang et~al.(2017)Jiang, Zhou, Leung, Li, and
  Fei-Fei]{jiang2017mentornet}
Lu~Jiang, Zhengyuan Zhou, Thomas Leung, Li-Jia Li, and Li~Fei-Fei.
\newblock Mentornet: Learning data-driven curriculum for very deep neural
  networks on corrupted labels.
\newblock \emph{arXiv preprint arXiv:1712.05055}, 2017.

\bibitem[Kawaguchi et~al.(2017)Kawaguchi, Kaelbling, and
  Bengio]{kawaguchi2017generalization}
Kenji Kawaguchi, Leslie~Pack Kaelbling, and Yoshua Bengio.
\newblock Generalization in deep learning.
\newblock \emph{arXiv preprint arXiv:1710.05468}, 2017.

\bibitem[Keskar et~al.(2016)Keskar, Mudigere, Nocedal, Smelyanskiy, and
  Tang]{keskar2016large}
Nitish~Shirish Keskar, Dheevatsa Mudigere, Jorge Nocedal, Mikhail Smelyanskiy,
  and Ping Tak~Peter Tang.
\newblock On large-batch training for deep learning: Generalization gap and
  sharp minima.
\newblock \emph{arXiv preprint arXiv:1609.04836}, 2016.

\bibitem[Kim et~al.(2015)Kim, Park, Yoo, Choi, Yang, and
  Shin]{kim2015compression}
Yong-Deok Kim, Eunhyeok Park, Sungjoo Yoo, Taelim Choi, Lu~Yang, and Dongjun
  Shin.
\newblock Compression of deep convolutional neural networks for fast and low
  power mobile applications.
\newblock \emph{arXiv preprint arXiv:1511.06530}, 2015.

\bibitem[Kolda and Bader(2009)]{kolda2009tensor}
Tamara~G Kolda and Brett~W Bader.
\newblock Tensor decompositions and applications.
\newblock \emph{SIAM review}, 51\penalty0 (3):\penalty0 455--500, 2009.

\bibitem[Kossaifi et~al.(2017)Kossaifi, Khanna, Lipton, Furlanello, and
  Anandkumar]{kossaifi2017tensor2}
Jean Kossaifi, Aran Khanna, Zachary Lipton, Tommaso Furlanello, and Anima
  Anandkumar.
\newblock Tensor contraction layers for parsimonious deep nets.
\newblock In \emph{Computer Vision and Pattern Recognition Workshops (CVPRW),
  2017 IEEE Conference on}, pages 1940--1946. IEEE, 2017.

\bibitem[Kossaifi et~al.(2019)Kossaifi, Panagakis, Anandkumar, and
  Pantic]{kossaifi2019tensorly}
Jean Kossaifi, Yannis Panagakis, Anima Anandkumar, and Maja Pantic.
\newblock Tensorly: Tensor learning in python.
\newblock \emph{The Journal of Machine Learning Research}, 20\penalty0
  (1):\penalty0 925--930, 2019.

\bibitem[Krizhevsky et~al.(2012)Krizhevsky, Sutskever, and
  Hinton]{krizhevsky2012imagenet}
Alex Krizhevsky, Ilya Sutskever, and Geoffrey~E Hinton.
\newblock Imagenet classification with deep convolutional neural networks.
\newblock In \emph{Advances in neural information processing systems}, pages
  1097--1105, 2012.

\bibitem[Kruskal(1989)]{kruskal1989rank}
Joseph~B Kruskal.
\newblock Rank, decomposition, and uniqueness for 3-way and n-way arrays.
\newblock \emph{Multiway data analysis}, pages 7--18, 1989.

\bibitem[Kuzborskij and Lampert(2018)]{Kuzborskij2018}
Ilja Kuzborskij and Christoph~H. Lampert.
\newblock Data-dependent stability of stochastic gradient descent.
\newblock In \emph{ICML}, 2018.

\bibitem[Langford and Caruana(2002)]{langford2002not}
John Langford and Rich Caruana.
\newblock (not) bounding the true error.
\newblock In \emph{Advances in Neural Information Processing Systems}, pages
  809--816, 2002.

\bibitem[Lebedev et~al.(2014)Lebedev, Ganin, Rakhuba, Oseledets, and
  Lempitsky]{lebedev2014speeding}
Vadim Lebedev, Yaroslav Ganin, Maksim Rakhuba, Ivan Oseledets, and Victor
  Lempitsky.
\newblock Speeding-up convolutional neural networks using fine-tuned
  cp-decomposition.
\newblock \emph{arXiv preprint arXiv:1412.6553}, 2014.

\bibitem[Li and Huang(2018)]{li2018guaranteed}
Jialin Li and Furong Huang.
\newblock Guaranteed simultaneous asymmetric tensor decomposition via
  orthogonalized alternating least squares.
\newblock \emph{arXiv preprint arXiv:1805.10348}, 2018.

\bibitem[McAllester(1999{\natexlab{a}})]{mcallester1999pac}
David~A McAllester.
\newblock Pac-bayesian model averaging.
\newblock In \emph{Proceedings of the twelfth annual conference on
  Computational learning theory}, pages 164--170. ACM, 1999{\natexlab{a}}.

\bibitem[McAllester(1999{\natexlab{b}})]{mcallester1999some}
David~A McAllester.
\newblock Some pac-bayesian theorems.
\newblock \emph{Machine Learning}, 37\penalty0 (3):\penalty0 355--363,
  1999{\natexlab{b}}.

\bibitem[Mhaskar and Poggio(2016)]{mhaskar2016deep}
Hrushikesh~N Mhaskar and Tomaso Poggio.
\newblock Deep vs. shallow networks: An approximation theory perspective.
\newblock \emph{Analysis and Applications}, 14\penalty0 (06):\penalty0
  829--848, 2016.

\bibitem[Neyshabur et~al.(2015{\natexlab{a}})Neyshabur, Salakhutdinov, and
  Srebro]{neyshabur2015path}
Behnam Neyshabur, Ruslan~R Salakhutdinov, and Nati Srebro.
\newblock Path-sgd: Path-normalized optimization in deep neural networks.
\newblock In \emph{Advances in Neural Information Processing Systems}, pages
  2422--2430, 2015{\natexlab{a}}.

\bibitem[Neyshabur et~al.(2015{\natexlab{b}})Neyshabur, Tomioka, and
  Srebro]{neyshabur2015norm}
Behnam Neyshabur, Ryota Tomioka, and Nathan Srebro.
\newblock Norm-based capacity control in neural networks.
\newblock In \emph{Conference on Learning Theory}, pages 1376--1401,
  2015{\natexlab{b}}.

\bibitem[Neyshabur et~al.(2017{\natexlab{a}})Neyshabur, Bhojanapalli,
  McAllester, and Srebro]{neyshabur2017pac}
Behnam Neyshabur, Srinadh Bhojanapalli, David McAllester, and Nathan Srebro.
\newblock A pac-bayesian approach to spectrally-normalized margin bounds for
  neural networks.
\newblock \emph{arXiv preprint arXiv:1707.09564}, 2017{\natexlab{a}}.

\bibitem[Neyshabur et~al.(2017{\natexlab{b}})Neyshabur, Bhojanapalli,
  McAllester, and Srebro]{neyshabur2017exploring}
Behnam Neyshabur, Srinadh Bhojanapalli, David McAllester, and Nati Srebro.
\newblock Exploring generalization in deep learning.
\newblock In \emph{Advances in Neural Information Processing Systems}, pages
  5947--5956, 2017{\natexlab{b}}.

\bibitem[Neyshabur et~al.(2018)Neyshabur, Li, Bhojanapalli, LeCun, and
  Srebro]{neyshabur2018towards}
Behnam Neyshabur, Zhiyuan Li, Srinadh Bhojanapalli, Yann LeCun, and Nathan
  Srebro.
\newblock Towards understanding the role of over-parametrization in
  generalization of neural networks.
\newblock \emph{arXiv preprint arXiv:1805.12076}, 2018.

\bibitem[Paszke et~al.(2017)Paszke, Gross, Chintala, Chanan, Yang, DeVito, Lin,
  Desmaison, Antiga, and Lerer]{paszke2017automatic}
Adam Paszke, Sam Gross, Soumith Chintala, Gregory Chanan, Edward Yang, Zachary
  DeVito, Zeming Lin, Alban Desmaison, Luca Antiga, and Adam Lerer.
\newblock Automatic differentiation in pytorch.
\newblock In \emph{NIPS-W}, 2017.

\bibitem[Sedghi et~al.(2018)Sedghi, Gupta, and Long]{sedghi2018singular}
Hanie Sedghi, Vineet Gupta, and Philip~M Long.
\newblock The singular values of convolutional layers.
\newblock \emph{arXiv preprint arXiv:1805.10408}, 2018.

\bibitem[Sermanet et~al.(2013)Sermanet, Eigen, Zhang, Mathieu, Fergus, and
  LeCun]{sermanet2013overfeat}
Pierre Sermanet, David Eigen, Xiang Zhang, Micha{\"e}l Mathieu, Rob Fergus, and
  Yann LeCun.
\newblock Overfeat: Integrated recognition, localization and detection using
  convolutional networks.
\newblock \emph{arXiv preprint arXiv:1312.6229}, 2013.

\bibitem[Simonyan and Zisserman(2014)]{simonyan2014very}
Karen Simonyan and Andrew Zisserman.
\newblock Very deep convolutional networks for large-scale image recognition.
\newblock \emph{arXiv preprint arXiv:1409.1556}, 2014.

\bibitem[Sokolic et~al.(2016)Sokolic, Giryes, Sapiro, and
  Rodrigues]{sokolic2016generalization}
Jure Sokolic, Raja Giryes, Guillermo Sapiro, and Miguel~RD Rodrigues.
\newblock Generalization error of invariant classifiers.
\newblock \emph{arXiv preprint arXiv:1610.04574}, 2016.

\bibitem[Su et~al.(2018)Su, Li, Bhattacharjee, and Huang]{su2019tensorial}
Jiahao Su, Jingling Li, Bobby Bhattacharjee, and Furong Huang.
\newblock Tensorial neural networks: Generalization of neural networks and
  application to model compression.
\newblock \emph{https://arxiv.org/pdf/1805.10352.pdf}, 2018.

\bibitem[Szegedy et~al.(2015)Szegedy, Liu, Jia, Sermanet, Reed, Anguelov,
  Erhan, Vanhoucke, and Rabinovich]{szegedy2015going}
Christian Szegedy, Wei Liu, Yangqing Jia, Pierre Sermanet, Scott Reed, Dragomir
  Anguelov, Dumitru Erhan, Vincent Vanhoucke, and Andrew Rabinovich.
\newblock Going deeper with convolutions.
\newblock In \emph{Proceedings of the IEEE Conference on Computer Vision and
  Pattern Recognition}, pages 1--9, 2015.

\bibitem[Wang et~al.(2018)Wang, Sun, Eriksson, Wang, and
  Aggarwal]{wang2018wide}
Wenqi Wang, Yifan Sun, Brian Eriksson, Wenlin Wang, and Vaneet Aggarwal.
\newblock Wide compression: Tensor ring nets.
\newblock \emph{learning}, 14\penalty0 (15):\penalty0 13--31, 2018.

\bibitem[Xu and Mannor(2012)]{xu2012robustness}
Huan Xu and Shie Mannor.
\newblock Robustness and generalization.
\newblock \emph{Machine learning}, 86\penalty0 (3):\penalty0 391--423, 2012.

\bibitem[Zagoruyko and Komodakis(2016)]{zagoruyko2016wide}
Sergey Zagoruyko and Nikos Komodakis.
\newblock Wide residual networks.
\newblock \emph{arXiv preprint arXiv:1605.07146}, 2016.

\bibitem[Zeiler and Fergus(2014)]{zeiler2014visualizing}
Matthew~D Zeiler and Rob Fergus.
\newblock Visualizing and understanding convolutional networks.
\newblock In \emph{European conference on computer vision}, pages 818--833.
  Springer, 2014.

\bibitem[Zhang et~al.(2017)Zhang, Bengio, Hardt, Recht, and
  Vinyals]{zhang2016understanding}
Chiyuan Zhang, Samy Bengio, Moritz Hardt, Benjamin Recht, and Oriol Vinyals.
\newblock Understanding deep learning requires rethinking generalization.
\newblock \emph{International Conference on Learning Representations}, 2017.

\bibitem[Zhou et~al.(2018)Zhou, Veitch, Austern, Adams, and
  Orbanz]{zhou2018non}
Wenda Zhou, Victor Veitch, Morgane Austern, Ryan~P Adams, and Peter Orbanz.
\newblock Non-vacuous generalization bounds at the imagenet scale: a
  pac-bayesian compression approach.
\newblock \emph{arXiv preprint arXiv:1804.05862}, 2018.

\end{thebibliography}

\newpage
\appendix
\onecolumn

\section*{Supplementary Material}
Supplementary material for the paper: ``Understanding Generalization in Deep Learning via Tensor Methods''.
This appendix is organized as follows:
\begin{itemize}
    \item Appendix A: Experimental details and additional results
    \item Appendix B: Technical definitions and propositions
    \item Appendix C: Main technical contributions 
    \item Appendix D, E, and F: Generalization bounds on three types of neural networks: convolutional neural networks, fully-connected neural networks, and neural networks with residual connections
    \item Appendix G: Additional algorithms and algorithmic details
\end{itemize}

\section{Additional Experimental Results}
\label{app:experiments}

\subsection{Architecture and optimization setting}
\label{app:opt_setting}
We train these four models (VGG-16, \ourcpnn-VGG-16, WRN-28-10 and \ourcpnn-WRN-28-10) using standard optimization settings with no dropouts and default initializations provided by PyTorch~\citep{paszke2017automatic}. We use a SGD optimizer with momentum=0.9, weight decay=5e-4, and initial learning rate=0.05 to start the training process. The learning rate is scheduled to be divided by 2 every 30 epochs for VGG-16 and \ourcpnn-VGG-16. While for WRN-28-10 and \ourcpnn-WRN-28-10, the learning rate is scheduled to be divided by 5 at the $60^{\tha}, 120^{\tha}$ and $160^{\tha}$ epoch. We run 300 epochs to train each VGG-16 and \ourcpnn-VGG-16, and we run 200 epochs to train each WRN-28-10 and \ourcpnn-WRN-28-10. 

\subsection{Generalization bounds comparison with~\citep{arora2018stronger}}
\label{sec:comp_arora}
The generalization bound we calculated for a well-trained \ourcpnn-VGG-16 (with the same \# of parameters as VGG-16) on CIFAR10 dataset is around $12$ (thus, of order $10^1$) according to the transformation $f(x) = x / 20 - 0.5$ applied in Figure~\ref{fig:bound}. Our evaluated bound is much better than naive counting of \# parameters. Although we may not be able to directly compare our calculated bound with that in~\citep{arora2018stronger}, which is roughly of order $10^5$ as~\citep{arora2018stronger} uses a VGG-19 to evaluate their generalization bound while our evaluation is done using a \ourcpnn-VGG-16, we present in Table~\ref{tbl:effective} the effective number of parameters identified by our proposed bound. Compared with the effective number of parameters in~\citep{arora2018stronger} (Table 1 of~\citep{arora2018stronger}), we can see that 
\textbf{(1)} our effective number of parameters is upper bounded by the total number of parameters in original network (thus, the compression ratio is bounded by $1$), while the effective number identified by~\citep{arora2018stronger} could be several times larger than the original number of parameters (e.g. based on Table 1 of~\citep{arora2018stronger}, their effective number of parameters in layer 4 and 6 are more than 4 times of the original number of parameters);
\textbf{(2)} the effective number of parameters in~\citep{arora2018stronger} ignores the dependence on depth, log factors and constants, while our effective number of parameters in Table~\ref{tbl:effective} is exactly the actual number of parameters in the compressed network without these dependences.
\begin{table}[!htbp]
	\centering
	\caption{Effective number of parameters identified by our proposed bound in Theorem~\ref{thm:cnn_thm}.}
    \label{tbl:effective}
	\begin{tabular}{ c | c | c c | c c}
		       layer & \tabincell{c}{original \\ \# of params} & \tabincell{c}{our effective \\ \# of params} &  \tabincell{c}{our \\ compression ratio} & 
		       \tabincell{c}{effective \# of params \\ in \cite{arora2018stronger}} & \tabincell{c}{compression ratio\\ in \cite{arora2018stronger}} \\
			\hline
1 & 1728 & 1694 & 0.980324 & 10728.622 & 6.208693 \\ \hline 
2 & 36864 & 36984 & 1.003255 & 63681.09 & 1.727460 \\ \hline 
3 & 73728 & 73932 & 1.002767 & 116967.945 & 1.586479 \\ \hline 
4 & 147456 & 147630 & 1.001180 & 910160.75 & 6.172423 \\ \hline 
5 & 294912 & 295106 & 1.000658 & 817337.9 & 2.771464 \\ \hline 
6 & 589824 & 590904 & 1.001831 & 3913927.2 & 6.635754 \\ \hline 
7 & 589824 & 590904 & 1.001831 & 15346982.0 & 26.019596 \\ \hline 
8 & 1179648 & 1177892 & 0.998511 & 367775.12 & 0.311767 \\ \hline 
9 & 2359296 & 2288242 & 0.969883 & 95893.41 & 0.040645 \\ \hline 
10 & 2359296 & 1774344 & 0.752065 & 87476.836 & 0.037078 \\ \hline 
11 & 2359296 & 350526 & 0.148572 & 42480.465 & 0.018006 \\ \hline 
12 & 2359296 & 42394 & 0.017969 & 40184.535 & 0.017032 \\ \hline 
13 & 2359296 & 124080 & 0.052592 & 137974.52 & 0.058481
	\end{tabular}
\end{table}

\subsection{Neural networks with \ourNN are natural for compression}
The compression results in Table~\ref{tbl:compression} are obtained directly without any fine tuning. 
\begin{table}[!htbp]
	\centering
	\caption{The compression results of a 28-layer Wide-ResNet equipped with \ourNN (\ourcpnn-WRN-28-10) on CIFAR10 dataset. The compression is done via normalizing the CP spectrum and then deleting the components in \ourNN which have amplitudes smaller than the given cut-off-threshold.}
	\label{tbl:compression}
	\begin{tabular}{ c | c l c l c}
		       Cut-off threshold & Compression ratio  & \# params & Test acc \% \\
			\hline
			0 & $1\times$ & 36.5M & 95.09  \\ \hline 
			1e-4 & $0.229$ ($4\times$) & 8.36M & 95.08  \\ \hline 
			1e-3 & $0.164$ ($5\times$) & 6.90M & 95.05  \\ \hline 
			1e-2 & $0.124$ ($8\times$) & 4.52M & \textbf{94.53}
	\end{tabular}
\end{table}

\subsection{Improved Generalization Achieved by \ourNN}
We provide additional experimental details in the improved generalization ability achieved by \ourNN under label noise setting. 
Our \ourNN combined with co-teaching (CT)~\citep{han2018co} outperforms SOTA method. Co-teaching~\citep{han2018co} is a training procedure for defeating label noise: it avoids overfitting to noisy labels by selecting clean samples out of the noisy ones and using them to update the network. Given the experimental results that neural networks with \ourNN tend to overfit less to noisy labels (Table 3), we combine Co-teaching to train networks with \ourNN on three different types of corrupted data (Table~\ref{tbl:mnist}). The hyperparameters we use in these experiments are the same as the ones in Co-teaching [2].

As shown in Table~\ref{tbl:mnist}, we compare our method CT+\ourNN against various label-noise methods~\citep{han2018co} under standard label noise setting~\citep{han2018co}.
\textbf{(1)} As shown in Figure~\ref{fig:pf_sym_cr}, our method (CT+\ourNN) consistently outperforms the SOTA method with various choices of the number of components. \textbf{(1.1)} Specifically, according to Table~\ref{tbl:mnist}, we see that combining \ourNN with co-teaching achieves the SOTA results on MNIST for PairFlip\footnote{PairFlip denotes that the label mistakes can only happen within very similar classes~\citep{han2018co}} with corruption rate $45\%$ and Symmetric\footnote{Symmetric denotes that the label mistakes may happen across different classes uniformly~\citep{han2018co}} with corruption rate $50\%$. \textbf{(1.2)} We also investigate the learning curve of our method compared with the SOTA (see Figure~\ref{fig:learning curve}.).  
The models first reach best test accuracy early in the training, and then the test accuracy deteriorates as training goes on due to memorization effect. We see that our method always dominates the vanilla CT method when generalizability of the model starts to deteriorate due to memorization effect. This clearly shows that a neural network with \ourNN has better generalizability property than the plain neural network under this label noise setting.  
\textbf{(2)} For the Symmetric-$20\%$ in Table~\ref{tbl:mnist}, as the label corruption rate is low, our method has a low effect in improving the generalization, which is expected. 
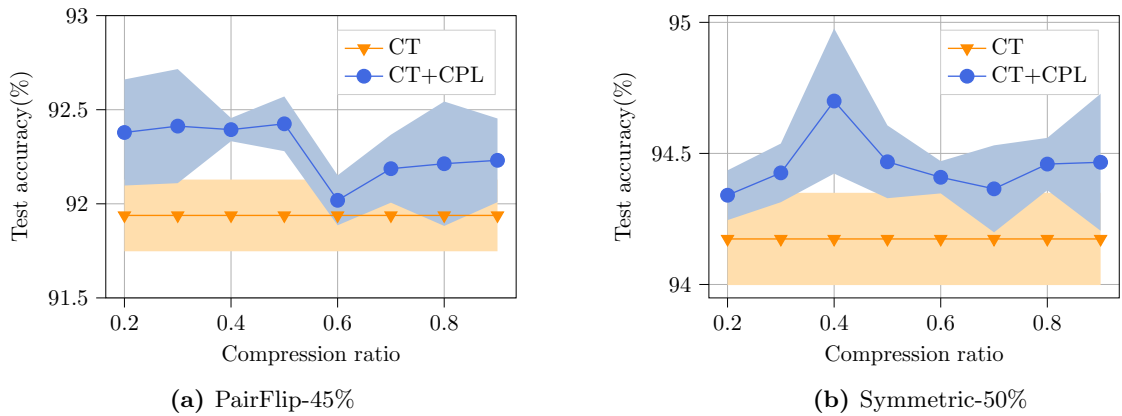
\begin{figure}[!htbp]
\centering
\begin{subfigure}[t]{0.49 \columnwidth}
\centering
\begin{tikzpicture}[scale=0.85]

\definecolor{color0}{rgb}{1,0.870588235294118,0.67843137254902}
\definecolor{color1}{rgb}{0.690196078431373,0.768627450980392,0.870588235294118}
\definecolor{color2}{rgb}{1,0.549019607843137,0}
\definecolor{color3}{rgb}{0.254901960784314,0.411764705882353,0.882352941176471}

\begin{axis}[
height=6cm,
legend cell align={left},
legend style={at={(0.95,0.95)}, draw=white!80.0!black},
tick align=outside,
tick pos=left,
width=8cm,
x grid style={white!69.01960784313725!black},
xlabel={Compression ratio},
xmajorgrids,
xmin=0.165, xmax=0.935,
xtick style={color=black},
y grid style={white!69.01960784313725!black},
ylabel={Test accuracy(\%)},
ymajorgrids,
ymin=91.5, ymax=93,
ytick style={color=black}
]
\path [fill=color0]
(axis cs:0.2,92.1287741335948)
--(axis cs:0.2,91.7477616570185)
--(axis cs:0.3,91.7477616570185)
--(axis cs:0.4,91.7477616570185)
--(axis cs:0.5,91.7477616570185)
--(axis cs:0.6,91.7477616570185)
--(axis cs:0.7,91.7477616570185)
--(axis cs:0.8,91.7477616570185)
--(axis cs:0.9,91.7477616570185)
--(axis cs:0.9,92.1287741335948)
--(axis cs:0.9,92.1287741335948)
--(axis cs:0.8,92.1287741335948)
--(axis cs:0.7,92.1287741335948)
--(axis cs:0.6,92.1287741335948)
--(axis cs:0.5,92.1287741335948)
--(axis cs:0.4,92.1287741335948)
--(axis cs:0.3,92.1287741335948)
--(axis cs:0.2,92.1287741335948)
--cycle;

\path [fill=color1]
(axis cs:0.2,92.6607779291573)
--(axis cs:0.2,92.0971681178561)
--(axis cs:0.3,92.1096945737014)
--(axis cs:0.4,92.3326054880042)
--(axis cs:0.5,92.2796992885223)
--(axis cs:0.6,91.885862663825)
--(axis cs:0.7,92.0061236327976)
--(axis cs:0.8,91.8825378190951)
--(axis cs:0.9,92.0087861416052)
--(axis cs:0.9,92.4533532814748)
--(axis cs:0.9,92.4533532814748)
--(axis cs:0.8,92.5435439116816)
--(axis cs:0.7,92.3672070295924)
--(axis cs:0.6,92.1512634045383)
--(axis cs:0.5,92.5703941943811)
--(axis cs:0.4,92.4557225034358)
--(axis cs:0.3,92.7160265801586)
--(axis cs:0.2,92.6607779291573)
--cycle;

\addplot [semithick, color2, mark=triangle*, mark size=3, mark options={solid,rotate=180}]
table {%
0.2 91.9382678953067
0.3 91.9382678953067
0.4 91.9382678953067
0.5 91.9382678953067
0.6 91.9382678953067
0.7 91.9382678953067
0.8 91.9382678953067
0.9 91.9382678953067
};
\addlegendentry{CT}
\addplot [semithick, color3, mark=*, mark size=3, mark options={solid}]
table {%
0.2 92.3789730235067
0.3 92.41286057693
0.4 92.39416399572
0.5 92.4250467414517
0.6 92.0185630341817
0.7 92.186665331195
0.8 92.2130408653883
0.9 92.23106971154
};
\addlegendentry{CT+CPL}
\end{axis}

\end{tikzpicture}
	\caption{PairFlip-45\%}
	\label{fig:pairflip_cr_0.2_0.9}
\end{subfigure}
\hfill
\begin{subfigure}[t]{0.49\columnwidth}
\begin{tikzpicture}[scale=0.85]

\definecolor{color0}{rgb}{1,0.870588235294118,0.67843137254902}
\definecolor{color1}{rgb}{0.690196078431373,0.768627450980392,0.870588235294118}
\definecolor{color2}{rgb}{1,0.549019607843137,0}
\definecolor{color3}{rgb}{0.254901960784314,0.411764705882353,0.882352941176471}

\begin{axis}[
height=6cm,
legend cell align={left},
legend style={at={(0.95,0.95)}, draw=white!80.0!black},
tick align=outside,
tick pos=left,
width=8cm,
x grid style={white!69.01960784313725!black},
xlabel={Compression ratio},
xmajorgrids,
xmin=0.165, xmax=0.935,
xtick style={color=black},
y grid style={white!69.01960784313725!black},
ylabel={Test accuracy(\%)},
ymajorgrids,
ymin=93.9487603669785, ymax=95.0256484464235,
ytick style={color=black}
]
\path [fill=color0]
(axis cs:0.2,94.3496459441016)
--(axis cs:0.2,93.9977098251351)
--(axis cs:0.3,93.9977098251351)
--(axis cs:0.4,93.9977098251351)
--(axis cs:0.5,93.9977098251351)
--(axis cs:0.6,93.9977098251351)
--(axis cs:0.7,93.9977098251351)
--(axis cs:0.8,93.9977098251351)
--(axis cs:0.9,93.9977098251351)
--(axis cs:0.9,94.3496459441016)
--(axis cs:0.9,94.3496459441016)
--(axis cs:0.8,94.3496459441016)
--(axis cs:0.7,94.3496459441016)
--(axis cs:0.6,94.3496459441016)
--(axis cs:0.5,94.3496459441016)
--(axis cs:0.4,94.3496459441016)
--(axis cs:0.3,94.3496459441016)
--(axis cs:0.2,94.3496459441016)
--cycle;

\path [fill=color1]
(axis cs:0.2,94.4353752900561)
--(axis cs:0.2,94.2455141330305)
--(axis cs:0.3,94.3138522991806)
--(axis cs:0.4,94.4230072082998)
--(axis cs:0.5,94.3291288342621)
--(axis cs:0.6,94.3474493522935)
--(axis cs:0.7,94.1988784168527)
--(axis cs:0.8,94.3593489231341)
--(axis cs:0.9,94.2048729520247)
--(axis cs:0.9,94.7274187146386)
--(axis cs:0.9,94.7274187146386)
--(axis cs:0.8,94.5592541751392)
--(axis cs:0.7,94.530755664334)
--(axis cs:0.6,94.4703257545565)
--(axis cs:0.5,94.6068353751445)
--(axis cs:0.4,94.9766989882669)
--(axis cs:0.3,94.5376434273094)
--(axis cs:0.2,94.4353752900561)
--cycle;

\addplot [semithick, color2, mark=triangle*, mark size=3, mark options={solid,rotate=180}]
table {%
0.2 94.1736778846183
0.3 94.1736778846183
0.4 94.1736778846183
0.5 94.1736778846183
0.6 94.1736778846183
0.7 94.1736778846183
0.8 94.1736778846183
0.9 94.1736778846183
};
\addlegendentry{CT}
\addplot [semithick, color3, mark=*, mark size=3, mark options={solid}]
table {%
0.2 94.3404447115433
0.3 94.425747863245
0.4 94.6998530982833
0.5 94.4679821047033
0.6 94.408887553425
0.7 94.3648170405933
0.8 94.4593015491367
0.9 94.4661458333317
};
\addlegendentry{CT+CPL}
\end{axis}

\end{tikzpicture}
	\caption{Symmetric-50\%}
	\label{fig:symmetric_cr_0.2_0.9}
\end{subfigure}
\caption{Test accuracy vs. different compression ratios}
\label{fig:pf_sym_cr}
\end{figure}

\begin{figure}[!htbp]
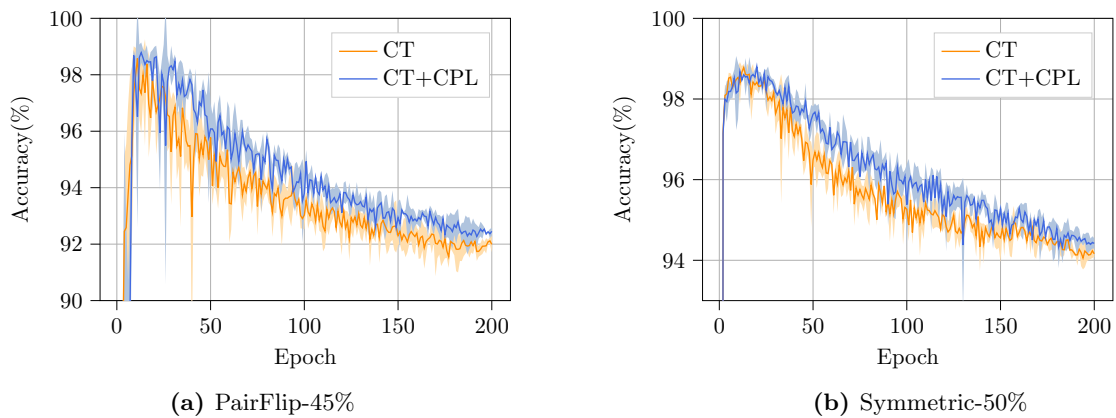

\centering
\begin{subfigure}[t]{0.49 \columnwidth}
\centering
	\input{\fighome/label_noise_mnist_cr_0.5_comparison_pairflip_0.5.tex}
	\caption{\small PairFlip-45\% }
	\label{fig:convergence}
\end{subfigure}
\hfill
\begin{subfigure}[t]{0.49\columnwidth}
	\input{\fighome/label_noise_mnist_cr_0.5_comparison_symmetric_0.5.tex}
	\caption{\small Symmetric-50\%}
	\label{fig:property1-setting2}
\end{subfigure}
\caption{Convergence plots of test accuracy vs. number of epochs on MNIST data}
\label{fig:learning curve}
\end{figure}

\textit{Remark.} The results displayed in Figure~\ref{fig:pf_sym_cr} and Figure~\ref{fig:learning curve} are based on our implementation of the CT method in order to achieve a fair comparison, while the results displayed in Table~\ref{tbl:mnist} are based on the reported accuracies by~\citep{han2018co} as we would like to compare our CT+\ourNN with other different label-noise methods as well.

\subsection{Compressibility of CPL: Property Evaluation \ourNN}
\begin{figure}[!htbp]
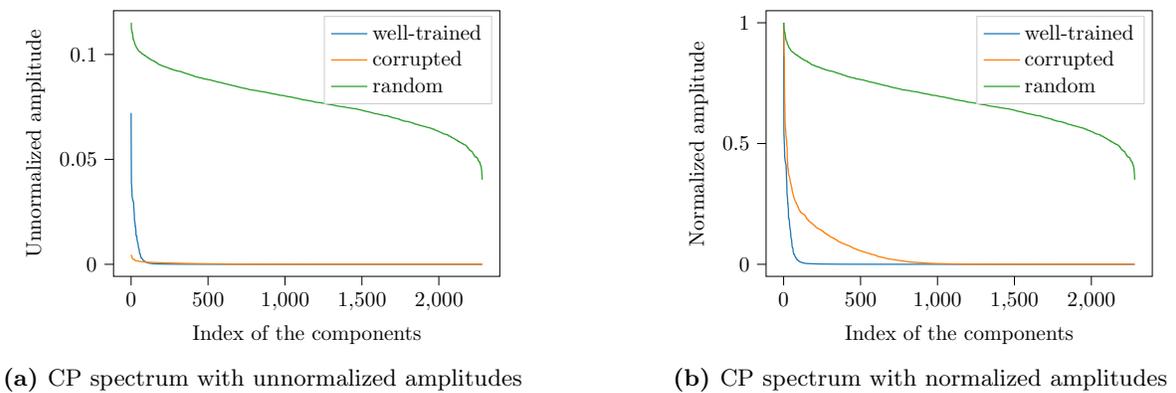

\centering
\begin{subfigure}[t]{0.49\columnwidth}
\centering
	\input{\fighome/VGG_conv12_CP_spectrum_compare_all_unnormalize.tex}
	\caption{CP spectrum with unnormalized amplitudes}
	\label{fig:cp_spec_all}
\end{subfigure}
\hfill
\begin{subfigure}[t]{0.49\columnwidth}
\centering
	\input{\fighome/VGG_conv12_CP_spectrum_compare_all_normalize.tex}
	\caption{CP spectrum with normalized amplitudes}
	\label{fig:cp_spec_all_n}
\end{subfigure}
\caption{Comparison of the CP spectra of a well-trained, a corrupted, and a randomly initialized \ourcpnn-VGG-16 a the $13^{\tha}$ convolutional layer  
}
\label{fig:eig_spectrum_all}
\end{figure}

Figure~\ref{fig:cp_spec_all_n} displays the CP spectral of a well-trained, a corrupted, and a randomly initialized \ourcpnn-VGG-16 (at the $13^{\tha}$ convolutional layer). For the unnormalized CP spectra of three models in Figure~\ref{fig:cp_spec_all_n}(a), we can see that the largest amplitude in the CP spectrum of the corrupted \ourcpnn-VGG-16 is much smaller than that of well-trained and random models. Yet, a smaller leading value in the CP spectrum does not necessarily mean that the corrupted is more low rank. As shown in Figure~\ref{fig:cp_spec_all_n}(b), after normalizing the CP spectrum of each model by its largest amplitude, well-trained \ourcpnn-VGG-16 still has the most low-rank CP spectrum (the blue curve) than that of corrupted or random models. Notice that the random model has the least low rankness since its weight tensors are the closest to random noise and thus it is hard to compress them.

We also compare our proposed properties among the three different sets of \ourcpnn-VGG-16: well-trained, corrupted, and randomly initialized. As shown in Figure~\ref{fig:properties_full}, since random models have the least compressibility as their weight tensors are closest to random noise, properties that focus more on the compressibility of the model are larger on random models (e.g \nblong), which will lead to larger generalization bounds. In the meantime, properties that focus more on measuring the information loss after compression as well as the expressive power of the models (e.g. \fflong) are smaller for random models. 
\begin{figure*}[h]
\centering
	\begin{subfigure}[t]{0.25\textwidth}
		\centering
\begin{tikzpicture}[scale=0.6]

\definecolor{color0}{rgb}{0.12156862745098,0.466666666666667,0.705882352941177}
\definecolor{color1}{rgb}{1,0.498039215686275,0.0549019607843137}
\definecolor{color2}{rgb}{0.235294117647059,0.701960784313725,0.443137254901961}

\begin{axis}[
height=6cm,
legend cell align={left},
legend style={at={(0.03,0.97)}, anchor=north west, draw=white!80.0!black},
tick align=outside,
tick pos=left,
title={rank across layers},
width=8cm,
x grid style={white!69.01960784313725!black},
xlabel={layers},
xmin=0.4, xmax=13.6,
xtick style={color=black},
y grid style={white!69.01960784313725!black},
ylabel={values},
ymin=79.8711587487176, ymax=2373.31989696268,
ytick style={color=black},
ytick={0,500,1000,1500,2000,2500},
yticklabels={0.0,0.5,1.0,1.5,2.0,2.5}
]
\path [draw=color0, fill=color0, opacity=0.2]
(axis cs:1,249.470496460325)
--(axis cs:1,247.869121893568)
--(axis cs:2,312.367284545135)
--(axis cs:3,424.881274534894)
--(axis cs:4,557.305196064375)
--(axis cs:5,675.425604087054)
--(axis cs:6,745.807645836515)
--(axis cs:7,746.749152240661)
--(axis cs:8,678.049017870757)
--(axis cs:9,560.967624987871)
--(axis cs:10,427.978350345089)
--(axis cs:11,309.765283850459)
--(axis cs:12,226.083423934901)
--(axis cs:13,184.118828667534)
--(axis cs:13,395.637516544983)
--(axis cs:13,395.637516544983)
--(axis cs:12,476.725097964097)
--(axis cs:11,617.631165586481)
--(axis cs:10,777.164163179594)
--(axis cs:9,906.700819491146)
--(axis cs:8,970.558755034376)
--(axis cs:7,957.163052053696)
--(axis cs:6,874.909323834581)
--(axis cs:5,743.466908273738)
--(axis cs:4,588.335088783523)
--(axis cs:3,437.195105450879)
--(axis cs:2,316.684698483768)
--(axis cs:1,249.470496460325)
--cycle;

\path [draw=color1, fill=color1, opacity=0.2]
(axis cs:1,250.322124852704)
--(axis cs:1,250.318819923541)
--(axis cs:2,319.794176686693)
--(axis cs:3,448.198737643791)
--(axis cs:4,621.380514402018)
--(axis cs:5,828.041008650377)
--(axis cs:6,1060.20985265151)
--(axis cs:7,1306.49078846146)
--(axis cs:8,1542.04898213765)
--(axis cs:9,1728.80330465388)
--(axis cs:10,1835.37706502227)
--(axis cs:11,1861.45406238791)
--(axis cs:12,1840.01679559375)
--(axis cs:13,1814.90694011668)
--(axis cs:13,1935.21527372302)
--(axis cs:13,1935.21527372302)
--(axis cs:12,1947.77116100006)
--(axis cs:11,1947.7772408588)
--(axis cs:10,1896.98126869121)
--(axis cs:9,1767.6808817796)
--(axis cs:8,1563.54287026963)
--(axis cs:7,1316.79734730177)
--(axis cs:6,1064.45331148018)
--(axis cs:5,829.524040866063)
--(axis cs:4,621.814945297828)
--(axis cs:3,448.304102447599)
--(axis cs:2,319.81438310002)
--(axis cs:1,250.322124852704)
--cycle;

\path [draw=color2, fill=color2, opacity=0.2]
(axis cs:1,250.329321898876)
--(axis cs:1,250.329321898876)
--(axis cs:2,319.86024551627)
--(axis cs:3,448.548902825917)
--(axis cs:4,622.852369554504)
--(axis cs:5,833.147549137253)
--(axis cs:6,1075.0722105309)
--(axis cs:7,1343.2530683415)
--(axis cs:8,1620.14988711627)
--(axis cs:9,1872.48693463644)
--(axis cs:10,2065.95417305382)
--(axis cs:11,2186.66181365902)
--(axis cs:12,2246.4250521283)
--(axis cs:13,2268.22388096713)
--(axis cs:13,2269.07222704386)
--(axis cs:13,2269.07222704386)
--(axis cs:12,2247.0960814877)
--(axis cs:11,2187.08145076518)
--(axis cs:10,2066.16146625203)
--(axis cs:9,1872.56772355122)
--(axis cs:8,1620.174693531)
--(axis cs:7,1343.25906039552)
--(axis cs:6,1075.07334749846)
--(axis cs:5,833.147700313281)
--(axis cs:4,622.852369554504)
--(axis cs:3,448.548902825917)
--(axis cs:2,319.86024551627)
--(axis cs:1,250.329321898876)
--cycle;

\addplot [line width=2.0pt, color0]
table {%
1 248.669809176947
2 314.525991514452
3 431.038189992886
4 572.820142423949
5 709.446256180396
6 810.358484835548
7 851.956102147179
8 824.303886452567
9 733.834222239509
10 602.571256762342
11 463.69822471847
12 351.404260949499
13 289.878172606259
};
\addlegendentry{well trained}
\addplot [line width=2.0pt, color1]
table {%
1 250.320472388123
2 319.804279893357
3 448.251420045695
4 621.597729849923
5 828.78252475822
6 1062.33158206585
7 1311.64406788162
8 1552.79592620364
9 1748.24209321674
10 1866.17916685674
11 1904.61565162335
12 1893.8939782969
13 1875.06110691985
};
\addlegendentry{corrupted}
\addplot [line width=2.0pt, color2]
table {%
1 250.329321898876
2 319.86024551627
3 448.548902825917
4 622.852369554504
5 833.147624725267
6 1075.07277901468
7 1343.25606436851
8 1620.16229032363
9 1872.52732909383
10 2066.05781965292
11 2186.8716322121
12 2246.760566808
13 2268.6480540055
};
\addlegendentry{random}
\end{axis}

\end{tikzpicture}	
		\vspace{-1.25em}  	
	  	\caption{\scriptsize{Rank $R$}}\label{fig:compare_rank_full}
	\end{subfigure}
	\hfill
	\begin{subfigure}[t]{0.25\textwidth}
		\centering
\begin{tikzpicture}[scale=0.6]

\definecolor{color0}{rgb}{0.12156862745098,0.466666666666667,0.705882352941177}
\definecolor{color1}{rgb}{1,0.498039215686275,0.0549019607843137}
\definecolor{color2}{rgb}{0.235294117647059,0.701960784313725,0.443137254901961}

\begin{axis}[
height=6cm,
legend cell align={left},
legend style={at={(0.03,0.97)}, anchor=north west, draw=white!80.0!black},
tick align=outside,
tick pos=left,
title={tensorization factor across layers},
width=8cm,
x grid style={white!69.01960784313725!black},
xlabel={layers},
xmin=0.4, xmax=13.6,
xtick style={color=black},
y grid style={white!69.01960784313725!black},
ylabel={values},
ymin=-4.44489397342257, ymax=115.535156184245,
ytick style={color=black},
ytick={-20,0,20,40,60,80,100,120},
yticklabels={−0.2,0.0,0.2,0.4,0.6,0.8,1.0,1.2}
]
\path [draw=color0, fill=color0, opacity=0.2]
(axis cs:1,3.11147761506754)
--(axis cs:1,1.76958935982834)
--(axis cs:2,1.98659098076038)
--(axis cs:3,2.35782043796456)
--(axis cs:4,2.75422472572999)
--(axis cs:5,3.01035138131269)
--(axis cs:6,2.9981468650283)
--(axis cs:7,2.7001148471897)
--(axis cs:8,2.22158972814575)
--(axis cs:9,1.72626081702029)
--(axis cs:10,1.34475073573918)
--(axis cs:11,1.12304547168705)
--(axis cs:12,1.03152404895329)
--(axis cs:13,1.00874467010779)
--(axis cs:13,1.12412095391132)
--(axis cs:13,1.12412095391132)
--(axis cs:12,1.15425093012)
--(axis cs:11,1.26203475292156)
--(axis cs:10,1.51023890509413)
--(axis cs:9,1.92768339441211)
--(axis cs:8,2.46450097054077)
--(axis cs:7,2.98500364448577)
--(axis cs:6,3.32772713794017)
--(axis cs:5,3.4113136919585)
--(axis cs:4,3.30142091800641)
--(axis cs:3,3.15968233386837)
--(axis cs:2,3.10210337896098)
--(axis cs:1,3.11147761506754)
--cycle;

\path [draw=color1, fill=color1, opacity=0.2]
(axis cs:1,4.05151376909588)
--(axis cs:1,3.67865293372726)
--(axis cs:2,4.11017633255532)
--(axis cs:3,4.9438695482516)
--(axis cs:4,6.0013619608106)
--(axis cs:5,6.9108002890078)
--(axis cs:6,7.26847633256024)
--(axis cs:7,6.8929050872591)
--(axis cs:8,5.93665147278224)
--(axis cs:9,4.76980534159105)
--(axis cs:10,3.76200596259813)
--(axis cs:11,3.11916905657531)
--(axis cs:12,2.83560414353281)
--(axis cs:13,2.76343212348929)
--(axis cs:13,3.09022607450169)
--(axis cs:13,3.09022607450169)
--(axis cs:12,3.14647842570808)
--(axis cs:11,3.4135653506836)
--(axis cs:10,4.05768706125219)
--(axis cs:9,5.09030174880748)
--(axis cs:8,6.29265380654819)
--(axis cs:7,7.27346257896795)
--(axis cs:6,7.64985024862528)
--(axis cs:5,7.27457009373261)
--(axis cs:4,6.34682453355547)
--(axis cs:3,5.28652030501051)
--(axis cs:2,4.46711643593077)
--(axis cs:1,4.05151376909588)
--cycle;

\path [draw=color2, fill=color2, opacity=0.2]
(axis cs:1,14.9773363226597)
--(axis cs:1,14.3369366275349)
--(axis cs:2,17.8316906306252)
--(axis cs:3,24.293176231109)
--(axis cs:4,33.0160624556476)
--(axis cs:5,43.492005565454)
--(axis cs:6,55.4625910958894)
--(axis cs:7,68.5149219757918)
--(axis cs:8,81.544644135253)
--(axis cs:9,92.8397820289554)
--(axis cs:10,101.000883243268)
--(axis cs:11,105.791433147679)
--(axis cs:12,108.034073555391)
--(axis cs:13,108.817592582014)
--(axis cs:13,110.081517540715)
--(axis cs:13,110.081517540715)
--(axis cs:12,109.293739025265)
--(axis cs:11,107.040047354228)
--(axis cs:10,102.224211872036)
--(axis cs:9,94.0145582964942)
--(axis cs:8,82.6463563195459)
--(axis cs:7,69.5292526574481)
--(axis cs:6,56.3881939515255)
--(axis cs:5,44.3339367601693)
--(axis cs:4,33.781788499771)
--(axis cs:3,24.9959088014886)
--(axis cs:2,18.4920043497714)
--(axis cs:1,14.9773363226597)
--cycle;

\addplot [line width=2.0pt, color0]
table {%
1 2.44053348744794
2 2.54434717986068
3 2.75875138591647
4 3.0278228218682
5 3.2108325366356
6 3.16293700148423
7 2.84255924583773
8 2.34304534934326
9 1.8269721057162
10 1.42749482041665
11 1.19254011230431
12 1.09288748953664
13 1.06643281200956
};
\addlegendentry{well trained}
\addplot [line width=2.0pt, color1]
table {%
1 3.86508335141157
2 4.28864638424305
3 5.11519492663105
4 6.17409324718303
5 7.09268519137021
6 7.45916329059276
7 7.08318383311352
8 6.11465263966521
9 4.93005354519926
10 3.90984651192516
11 3.26636720362945
12 2.99104128462044
13 2.92682909899549
};
\addlegendentry{corrupted}
\addplot [line width=2.0pt, color2]
table {%
1 14.6571364750973
2 18.1618474901983
3 24.6445425162988
4 33.3989254777093
5 43.9129711628116
6 55.9253925237074
7 69.0220873166199
8 82.0955002273995
9 93.4271701627248
10 101.612547557652
11 106.415740250954
12 108.663906290328
13 109.449555061364
};
\addlegendentry{random}
\end{axis}

\end{tikzpicture}	  
		\vspace{-1.25em}  	
	  	\caption{\scriptsize{\tfshort $\tf_j^{(k)}$}}\label{fig:compare_tf_full}
	\end{subfigure}
	\hfill
	\begin{subfigure}[t]{0.25\textwidth}
		\centering
\begin{tikzpicture}[scale=0.6]

\definecolor{color0}{rgb}{0.12156862745098,0.466666666666667,0.705882352941177}
\definecolor{color1}{rgb}{1,0.498039215686275,0.0549019607843137}
\definecolor{color2}{rgb}{0.235294117647059,0.701960784313725,0.443137254901961}

\begin{axis}[
height=6cm,
legend cell align={left},
legend style={at={(0.03,0.97)}, anchor=north west, draw=white!80.0!black},
tick align=outside,
tick pos=left,
title={tensor noise bound across layers},
width=8cm,
x grid style={white!69.01960784313725!black},
xlabel={layers},
xmin=0.4, xmax=13.6,
xtick style={color=black},
y grid style={white!69.01960784313725!black},
ylabel={values},
ymin=-0.000824099182689119, ymax=0.0173060828364715,
ytick style={color=black},
ytick={-0.0025,0,0.0025,0.005,0.0075,0.01,0.0125,0.015,0.0175},
yticklabels={−0.25,0.00,0.25,0.50,0.75,1.00,1.25,1.50,1.75}
]
\path [draw=color0, fill=color0, opacity=0.2]
(axis cs:1,1.53902596291674e-05)
--(axis cs:1,2.66226851258948e-06)
--(axis cs:2,1.11351921472824e-05)
--(axis cs:3,4.40221991176032e-05)
--(axis cs:4,0.00014359983933723)
--(axis cs:5,0.00038771649488702)
--(axis cs:6,0.000871032248657139)
--(axis cs:7,0.00164252806247223)
--(axis cs:8,0.00263179545278577)
--(axis cs:9,0.00363474016678337)
--(axis cs:10,0.00439696511388856)
--(axis cs:11,0.00475781482149379)
--(axis cs:12,0.00476704349688634)
--(axis cs:13,0.00466019643860711)
--(axis cs:13,0.0130474327990112)
--(axis cs:13,0.0130474327990112)
--(axis cs:12,0.0144609516034074)
--(axis cs:11,0.0160997268083117)
--(axis cs:10,0.0164819836537824)
--(axis cs:9,0.0148380452523704)
--(axis cs:8,0.011503239324214)
--(axis cs:7,0.00757640234967034)
--(axis cs:6,0.00419284278449118)
--(axis cs:5,0.00193270957121377)
--(axis cs:4,0.000738121397133778)
--(axis cs:3,0.000233181503658105)
--(axis cs:2,6.10811174860585e-05)
--(axis cs:1,1.53902596291674e-05)
--cycle;

\path [draw=color1, fill=color1, opacity=0.2]
(axis cs:1,1.37079868835902e-07)
--(axis cs:1,8.20261458199086e-08)
--(axis cs:2,4.86742732813909e-07)
--(axis cs:3,2.62028316249816e-06)
--(axis cs:4,1.12131447222204e-05)
--(axis cs:5,4.03308783883031e-05)
--(axis cs:6,0.000124141775072579)
--(axis cs:7,0.000330041618538635)
--(axis cs:8,0.000759639693864529)
--(axis cs:9,0.0015122630734359)
--(axis cs:10,0.00260410966887317)
--(axis cs:11,0.00388910178162672)
--(axis cs:12,0.00506067560052088)
--(axis cs:13,0.00576459428216655)
--(axis cs:13,0.00964982429365226)
--(axis cs:13,0.00964982429365226)
--(axis cs:12,0.00844018087728082)
--(axis cs:11,0.00645185501980959)
--(axis cs:10,0.00430207732271953)
--(axis cs:9,0.00249504103616969)
--(axis cs:8,0.00125531771330198)
--(axis cs:7,0.000547193224414612)
--(axis cs:6,0.000206541597406848)
--(axis cs:5,6.72620406203852e-05)
--(axis cs:4,1.87209568113716e-05)
--(axis cs:3,4.38996728667519e-06)
--(axis cs:2,8.13344055094636e-07)
--(axis cs:1,1.37079868835902e-07)
--cycle;

\path [draw=color2, fill=color2, opacity=0.2]
(axis cs:1,0)
--(axis cs:1,0)
--(axis cs:2,0)
--(axis cs:3,0)
--(axis cs:4,0)
--(axis cs:5,8.15859842736295e-07)
--(axis cs:6,6.13593430476065e-06)
--(axis cs:7,3.23376418891844e-05)
--(axis cs:8,0.000133874119468258)
--(axis cs:9,0.000435997904055374)
--(axis cs:10,0.00111871040958425)
--(axis cs:11,0.00226467825748865)
--(axis cs:12,0.00362138042142467)
--(axis cs:13,0.00457831513601737)
--(axis cs:13,0.00686169970779245)
--(axis cs:13,0.00686169970779245)
--(axis cs:12,0.00542750427641177)
--(axis cs:11,0.0033941617551411)
--(axis cs:10,0.00167665498387386)
--(axis cs:9,0.00065344708740546)
--(axis cs:8,0.000200642371515606)
--(axis cs:7,4.84656868978081e-05)
--(axis cs:6,9.19616439130373e-06)
--(axis cs:5,1.22276101102403e-06)
--(axis cs:4,0)
--(axis cs:3,0)
--(axis cs:2,0)
--(axis cs:1,0)
--cycle;

\addplot [line width=2.0pt, color0]
table {%
1 9.02626407087846e-06
2 3.61081548166704e-05
3 0.000138601851387854
4 0.000440860618235504
5 0.00116021303305039
6 0.00253193751657416
7 0.00460946520607129
8 0.00706751738849989
9 0.00923639270957686
10 0.0104394743838355
11 0.0104287708149028
12 0.00961399755014686
13 0.00885381461880914
};
\addlegendentry{well trained}
\addplot [line width=2.0pt, color1]
table {%
1 1.09553007327905e-07
2 6.50043393954273e-07
3 3.50512522458667e-06
4 1.4967050766796e-05
5 5.37964595043441e-05
6 0.000165341686239714
7 0.000438617421476623
8 0.00100747870358325
9 0.0020036520548028
10 0.00345309349579635
11 0.00517047840071815
12 0.00675042823890085
13 0.0077072092879094
};
\addlegendentry{corrupted}
\addplot [line width=2.0pt, color2]
table {%
1 0
2 0
3 0
4 0
5 1.01931042688016e-06
6 7.66604934803219e-06
7 4.04016643934963e-05
8 0.000167258245491932
9 0.000544722495730417
10 0.00139768269672905
11 0.00282942000631488
12 0.00452444234891822
13 0.00572000742190491
};
\addlegendentry{random}
\end{axis}

\end{tikzpicture}
		\vspace{-1.25em} 
	  	\caption{\scriptsize{\nbshort $\nb_j^{(k)}$}}\label{fig:compare_tnb_full}
	\end{subfigure}
	\hfill
	
	\begin{subfigure}[t]{0.25\textwidth}
		\centering
\begin{tikzpicture}[scale=0.6]

\definecolor{color0}{rgb}{0.12156862745098,0.466666666666667,0.705882352941177}
\definecolor{color1}{rgb}{1,0.498039215686275,0.0549019607843137}
\definecolor{color2}{rgb}{0.235294117647059,0.701960784313725,0.443137254901961}

\begin{axis}[
height=6cm,
legend cell align={left},
legend style={at={(0.03,0.97)}, anchor=north west, draw=white!80.0!black},
tick align=outside,
tick pos=left,
width=8cm,
x grid style={white!69.01960784313725!black},
xlabel={layers},
xmin=0.4, xmax=13.6,
xtick style={color=black},
y grid style={white!69.01960784313725!black},
ymin=-0.054899944702629, ymax=1.22311591032194,
ytick style={color=black},
ytick={-0.2,0,0.2,0.4,0.6,0.8,1,1.2,1.4},
yticklabels={−0.2,0.0,0.2,0.4,0.6,0.8,1.0,1.2,1.4}
]
\path [draw=color0, fill=color0, opacity=0.2]
(axis cs:1,0.269719034433365)
--(axis cs:1,0.0886733308434486)
--(axis cs:2,0.101778842508793)
--(axis cs:3,0.117614582180977)
--(axis cs:4,0.12878143787384)
--(axis cs:5,0.141245678067207)
--(axis cs:6,0.170574516057968)
--(axis cs:7,0.22951602935791)
--(axis cs:8,0.318743884563446)
--(axis cs:9,0.426498681306839)
--(axis cs:10,0.534554243087769)
--(axis cs:11,0.624752044677734)
--(axis cs:12,0.685113430023193)
--(axis cs:13,0.71358597278595)
--(axis cs:13,1.16064357757568)
--(axis cs:13,1.16064357757568)
--(axis cs:12,1.1650242805481)
--(axis cs:11,1.13013088703156)
--(axis cs:10,1.0173534154892)
--(axis cs:9,0.832127630710602)
--(axis cs:8,0.619668424129486)
--(axis cs:7,0.435828536748886)
--(axis cs:6,0.319092273712158)
--(axis cs:5,0.274275302886963)
--(axis cs:4,0.274758607149124)
--(axis cs:3,0.283312439918518)
--(axis cs:2,0.27944341301918)
--(axis cs:1,0.269719034433365)
--cycle;

\path [draw=color1, fill=color1, opacity=0.2]
(axis cs:1,0.150583237409592)
--(axis cs:1,0.113911338150501)
--(axis cs:2,0.124691992998123)
--(axis cs:3,0.133493185043335)
--(axis cs:4,0.130237728357315)
--(axis cs:5,0.119351603090763)
--(axis cs:6,0.11540262401104)
--(axis cs:7,0.130473926663399)
--(axis cs:8,0.167038708925247)
--(axis cs:9,0.21879443526268)
--(axis cs:10,0.273710668087006)
--(axis cs:11,0.317655593156815)
--(axis cs:12,0.342094480991364)
--(axis cs:13,0.350226551294327)
--(axis cs:13,0.473463535308838)
--(axis cs:13,0.473463535308838)
--(axis cs:12,0.463163495063782)
--(axis cs:11,0.429913192987442)
--(axis cs:10,0.368514031171799)
--(axis cs:9,0.291701257228851)
--(axis cs:8,0.220419570803642)
--(axis cs:7,0.171347752213478)
--(axis cs:6,0.151937887072563)
--(axis cs:5,0.157805278897285)
--(axis cs:4,0.172511711716652)
--(axis cs:3,0.176814839243889)
--(axis cs:2,0.165004163980484)
--(axis cs:1,0.150583237409592)
--cycle;

\path [draw=color2, fill=color2, opacity=0.2]
(axis cs:1,0.0242142248898745)
--(axis cs:1,0.0161362588405609)
--(axis cs:2,0.0141245359554887)
--(axis cs:3,0.0112250940874219)
--(axis cs:4,0.00866283476352692)
--(axis cs:5,0.00694219395518303)
--(axis cs:6,0.0059116967022419)
--(axis cs:7,0.00524378288537264)
--(axis cs:8,0.00472062267363071)
--(axis cs:9,0.00425695907324553)
--(axis cs:10,0.00384829612448812)
--(axis cs:11,0.00352222588844597)
--(axis cs:12,0.00330125400796533)
--(axis cs:13,0.00319168507121503)
--(axis cs:13,0.00358926481567323)
--(axis cs:13,0.00358926481567323)
--(axis cs:12,0.00372587400488555)
--(axis cs:11,0.00400323327630758)
--(axis cs:10,0.00441949628293514)
--(axis cs:9,0.00495862262323499)
--(axis cs:8,0.0056027821265161)
--(axis cs:7,0.00638120993971825)
--(axis cs:6,0.0074461093172431)
--(axis cs:5,0.00915412977337837)
--(axis cs:4,0.0120156528428197)
--(axis cs:3,0.0162326209247112)
--(axis cs:2,0.0209565851837397)
--(axis cs:1,0.0242142248898745)
--cycle;

\addplot [line width=2.0pt, color0]
table {%
1 0.179196193814278
2 0.190611124038696
3 0.200463518500328
4 0.201770007610321
5 0.207760497927666
6 0.244833394885063
7 0.332672268152237
8 0.469206154346466
9 0.629313170909882
10 0.775953829288483
11 0.877441465854645
12 0.925068795681
13 0.937114775180817
};
\addlegendentry{well trained}
\addplot [line width=2.0pt, color1]
table {%
1 0.132247284054756
2 0.144848078489304
3 0.155154019594193
4 0.151374712586403
5 0.138578444719315
6 0.133670255541801
7 0.150910839438438
8 0.193729132413864
9 0.255247831344604
10 0.321112334728241
11 0.373784393072128
12 0.402628988027573
13 0.411845058202744
};
\addlegendentry{corrupted}
\addplot [line width=2.0pt, color2]
table {%
1 0.0201752409338951
2 0.0175405610352755
3 0.0137288579717278
4 0.010339243337512
5 0.0080481618642807
6 0.0066789030097425
7 0.00581249641254544
8 0.00516170216724277
9 0.00460779061540961
10 0.00413389643654227
11 0.0037627296987921
12 0.00351356389001012
13 0.00339047494344413
};
\addlegendentry{random}
\end{axis}

\end{tikzpicture}
		\vspace{-1.25em} 	  	
	  	\caption{\scriptsize{\lcshort $\lc^{(k)}$}}\label{fig:compare_lc_full}
	\end{subfigure}
	\hfill
	\begin{subfigure}[t]{0.25\textwidth}
		\centering
\begin{tikzpicture}[scale=0.6]

\definecolor{color0}{rgb}{0.12156862745098,0.466666666666667,0.705882352941177}
\definecolor{color1}{rgb}{1,0.498039215686275,0.0549019607843137}
\definecolor{color2}{rgb}{0.235294117647059,0.701960784313725,0.443137254901961}

\begin{axis}[
height=6cm,
legend cell align={left},
legend style={at={(0.03,0.97)}, anchor=north west, draw=white!80.0!black},
tick align=outside,
tick pos=left,
title={norm across layers},
width=8cm,
x grid style={white!69.01960784313725!black},
xlabel={layers},
xmin=0.4, xmax=13.6,
xtick style={color=black},
y grid style={white!69.01960784313725!black},
ylabel={values},
ymin=0.259939684424103, ymax=7.88762995397984,
ytick style={color=black}
]
\path [draw=color0, fill=color0, opacity=0.2]
(axis cs:1,1.60861100510454)
--(axis cs:1,0.837724205125173)
--(axis cs:2,0.833891261883632)
--(axis cs:3,0.83592674082408)
--(axis cs:4,0.846233033438614)
--(axis cs:5,0.850768130427016)
--(axis cs:6,0.830019604723903)
--(axis cs:7,0.777030487776936)
--(axis cs:8,0.705202413658147)
--(axis cs:9,0.640772266896689)
--(axis cs:10,0.606652878494819)
--(axis cs:11,0.608704066851846)
--(axis cs:12,0.632802822061768)
--(axis cs:13,0.654269209343345)
--(axis cs:13,0.691189326526635)
--(axis cs:13,0.691189326526635)
--(axis cs:12,0.668018212880103)
--(axis cs:11,0.641876358290143)
--(axis cs:10,0.638835006400607)
--(axis cs:9,0.673364676191563)
--(axis cs:8,0.739361678035985)
--(axis cs:7,0.815641497884022)
--(axis cs:6,0.884399551806022)
--(axis cs:5,0.952688677591885)
--(axis cs:4,1.05722909364123)
--(axis cs:3,1.23167125416328)
--(axis cs:2,1.4494092377057)
--(axis cs:1,1.60861100510454)
--cycle;

\path [draw=color1, fill=color1, opacity=0.2]
(axis cs:1,1.92452668083742)
--(axis cs:1,1.79414890350294)
--(axis cs:2,1.77223078583635)
--(axis cs:3,1.76725883955574)
--(axis cs:4,1.80172107718851)
--(axis cs:5,1.84695169185468)
--(axis cs:6,1.84890765826342)
--(axis cs:7,1.77640994273292)
--(axis cs:8,1.64207450198184)
--(axis cs:9,1.48877002533571)
--(axis cs:10,1.36262263782975)
--(axis cs:11,1.28964409477502)
--(axis cs:12,1.26555073900297)
--(axis cs:13,1.26493429323197)
--(axis cs:13,1.32703435526486)
--(axis cs:13,1.32703435526486)
--(axis cs:12,1.32069692774067)
--(axis cs:11,1.33547350354057)
--(axis cs:10,1.40173361509727)
--(axis cs:9,1.52541370013006)
--(axis cs:8,1.67897380949724)
--(axis cs:7,1.81430904995821)
--(axis cs:6,1.88858724362443)
--(axis cs:5,1.89201609593833)
--(axis cs:4,1.86016529469353)
--(axis cs:3,1.84919332743526)
--(axis cs:2,1.88248700436756)
--(axis cs:1,1.92452668083742)
--cycle;

\path [draw=color2, fill=color2, opacity=0.2]
(axis cs:1,3.17511011335172)
--(axis cs:1,3.03292496973157)
--(axis cs:2,3.26079320182151)
--(axis cs:3,3.67456663014467)
--(axis cs:4,4.20936826192481)
--(axis cs:5,4.80548972425391)
--(axis cs:6,5.42424242338369)
--(axis cs:7,6.03451354876476)
--(axis cs:8,6.58742904256368)
--(axis cs:9,7.02180384530525)
--(axis cs:10,7.30380979371004)
--(axis cs:11,7.45073811267229)
--(axis cs:12,7.51109204771044)
--(axis cs:13,7.52967917966157)
--(axis cs:13,7.54091675990913)
--(axis cs:13,7.54091675990913)
--(axis cs:12,7.52246244684736)
--(axis cs:11,7.46247140450084)
--(axis cs:10,7.31626897172671)
--(axis cs:9,7.03544484440024)
--(axis cs:8,6.60278352348985)
--(axis cs:7,6.05250104800435)
--(axis cs:6,5.44715743201733)
--(axis cs:5,4.83864992976907)
--(axis cs:4,4.26216981173801)
--(axis cs:3,3.75787360409863)
--(axis cs:2,3.37887532817814)
--(axis cs:1,3.17511011335172)
--cycle;

\addplot [line width=2.0pt, color0]
table {%
1 1.22316760511486
2 1.14165024979467
3 1.03379899749368
4 0.951731063539922
5 0.901728404009451
6 0.857209578264963
7 0.796335992830479
8 0.722282045847066
9 0.657068471544126
10 0.622743942447713
11 0.625290212570995
12 0.650410517470935
13 0.67272926793499
};
\addlegendentry{well trained}
\addplot [line width=2.0pt, color1]
table {%
1 1.85933779217018
2 1.82735889510196
3 1.8082260834955
4 1.83094318594102
5 1.86948389389651
6 1.86874745094393
7 1.79535949634556
8 1.66052415573954
9 1.50709186273288
10 1.38217812646351
11 1.3125587991578
12 1.29312383337182
13 1.29598432424841
};
\addlegendentry{corrupted}
\addplot [line width=2.0pt, color2]
table {%
1 3.10401754154165
2 3.31983426499982
3 3.71622011712165
4 4.23576903683141
5 4.82206982701149
6 5.43569992770051
7 6.04350729838456
8 6.59510628302676
9 7.02862434485274
10 7.31003938271838
11 7.45660475858656
12 7.5167772472789
13 7.53529796978535
};
\addlegendentry{random}
\end{axis}

\end{tikzpicture}
		\vspace{-1.25em} 	  	
	  	\caption{\scriptsize{$\| \cnnweight\|_F$}}\label{fig:compare_norm_full}
	\end{subfigure}
	\hfill
	\begin{subfigure}[t]{0.25\textwidth}
		\centering
\begin{tikzpicture}[scale=0.6]

\definecolor{color0}{rgb}{0.12156862745098,0.466666666666667,0.705882352941177}
\definecolor{color1}{rgb}{1,0.498039215686275,0.0549019607843137}
\definecolor{color2}{rgb}{0.235294117647059,0.701960784313725,0.443137254901961}

\begin{axis}[
height=6cm,
legend cell align={left},
legend style={at={(0.03,0.97)}, anchor=north west, draw=white!80.0!black},
tick align=outside,
tick pos=left,
title={right hand side across layers},
width=8cm,
x grid style={white!69.01960784313725!black},
xlabel={layers},
xmin=0.4, xmax=13.6,
xtick style={color=black},
y grid style={white!69.01960784313725!black},
ylabel={values},
ymin=-0.00169396865621594, ymax=0.0355733444145008,
ytick style={color=black},
ytick={-0.005,0,0.005,0.01,0.015,0.02,0.025,0.03,0.035,0.04},
yticklabels={−0.5,0.0,0.5,1.0,1.5,2.0,2.5,3.0,3.5,4.0}
]
\path [draw=color0, fill=color0, opacity=0.2]
(axis cs:1,0.000101730870933412)
--(axis cs:1,1.96271189421916e-05)
--(axis cs:2,4.27917733511355e-05)
--(axis cs:3,0.000109941297283075)
--(axis cs:4,0.000268924910424118)
--(axis cs:5,0.000597071668712975)
--(axis cs:6,0.00117883987260182)
--(axis cs:7,0.00204989603112582)
--(axis cs:8,0.0031354122584105)
--(axis cs:9,0.00423462118707886)
--(axis cs:10,0.00508426802598728)
--(axis cs:11,0.00549510102622146)
--(axis cs:12,0.00550168002896586)
--(axis cs:13,0.00537037849170034)
--(axis cs:13,0.0142894231207172)
--(axis cs:13,0.0142894231207172)
--(axis cs:12,0.0156078723862379)
--(axis cs:11,0.0170247124726032)
--(axis cs:10,0.0171250693391927)
--(axis cs:9,0.0153062312149207)
--(axis cs:8,0.0119930968560415)
--(axis cs:7,0.00819224166461073)
--(axis cs:6,0.00487315186951737)
--(axis cs:5,0.00253868930823604)
--(axis cs:4,0.00117666017911986)
--(axis cs:3,0.000499394671738476)
--(axis cs:2,0.00020577262231601)
--(axis cs:1,0.000101730870933412)
--cycle;

\path [draw=color1, fill=color1, opacity=0.2]
(axis cs:1,0.000318636783974072)
--(axis cs:1,0.000191602264895965)
--(axis cs:2,0.000319614420956772)
--(axis cs:3,0.000661813480253461)
--(axis cs:4,0.00142388373189941)
--(axis cs:5,0.0029261826154943)
--(axis cs:6,0.00545533291456229)
--(axis cs:7,0.0090068331679096)
--(axis cs:8,0.0131315421865161)
--(axis cs:9,0.0170236878510498)
--(axis cs:10,0.0197724374874255)
--(axis cs:11,0.0207659329225993)
--(axis cs:12,0.0202408986354753)
--(axis cs:13,0.0193819254221389)
--(axis cs:13,0.0315113807063427)
--(axis cs:13,0.0315113807063427)
--(axis cs:12,0.03295966896203)
--(axis cs:11,0.0338793756385591)
--(axis cs:10,0.0322813711536247)
--(axis cs:9,0.0277483871346609)
--(axis cs:8,0.021315435155914)
--(axis cs:7,0.0145350881979989)
--(axis cs:6,0.00875076607213996)
--(axis cs:5,0.00467309421684007)
--(axis cs:4,0.00227321693652571)
--(axis cs:3,0.00106389898834645)
--(axis cs:2,0.00052188057672078)
--(axis cs:1,0.000318636783974072)
--cycle;

\path [draw=color2, fill=color2, opacity=0.2]
(axis cs:1,1.88712844484039e-10)
--(axis cs:1,1.19725731621237e-10)
--(axis cs:2,1.03704445592821e-09)
--(axis cs:3,1.14362609041035e-08)
--(axis cs:4,1.37130050896594e-07)
--(axis cs:5,1.71120429960652e-06)
--(axis cs:6,1.05861346898875e-05)
--(axis cs:7,5.01986027010191e-05)
--(axis cs:8,0.000192229305706235)
--(axis cs:9,0.000591435331456011)
--(axis cs:10,0.00145691708834582)
--(axis cs:11,0.00286774260091616)
--(axis cs:12,0.00450582201457596)
--(axis cs:13,0.00564844671954007)
--(axis cs:13,0.00798588062280282)
--(axis cs:13,0.00798588062280282)
--(axis cs:12,0.00637184754469401)
--(axis cs:11,0.00405724815058762)
--(axis cs:10,0.00206271345712485)
--(axis cs:9,0.000838209532391121)
--(axis cs:8,0.00027280886370169)
--(axis cs:7,7.13694127182469e-05)
--(axis cs:6,1.50862186287761e-05)
--(axis cs:5,2.44851180043275e-06)
--(axis cs:4,1.99997993228185e-07)
--(axis cs:3,1.68752669369147e-08)
--(axis cs:2,1.57565404893428e-09)
--(axis cs:1,1.88712844484039e-10)
--cycle;

\addplot [line width=2.0pt, color0]
table {%
1 6.06789949378017e-05
2 0.000124282197833573
3 0.000304667984510776
4 0.00072279254477199
5 0.00156788048847451
6 0.0030259958710596
7 0.00512106884786827
8 0.00756425455722599
9 0.00977042620099976
10 0.01110466868259
11 0.0112599067494123
12 0.0105547762076019
13 0.00982990080620876
};
\addlegendentry{well trained}
\addplot [line width=2.0pt, color1]
table {%
1 0.000255119524435019
2 0.000420747498838776
3 0.000862856234299954
4 0.00184855033421256
5 0.00379963841616718
6 0.00710304949335113
7 0.0117709606829542
8 0.0172234886712151
9 0.0223860374928553
10 0.0260269043205251
11 0.0273226542805792
12 0.0266002837987527
13 0.0254466530642408
};
\addlegendentry{corrupted}
\addplot [line width=2.0pt, color2]
table {%
1 1.54219288052638e-10
2 1.30634925243125e-09
3 1.41557639205091e-08
4 1.6856402206239e-07
5 2.07985805001964e-06
6 1.28361766593318e-05
7 6.0784007709633e-05
8 0.000232519084703963
9 0.000714822431923566
10 0.00175981527273534
11 0.00346249537575189
12 0.00543883477963499
13 0.00681716367117144
};
\addlegendentry{random}
\end{axis}

\end{tikzpicture}	
		\vspace{-1.25em} 
	  	\caption{\scriptsize{RHS}}\label{fig:compare_rhs_full}
	\end{subfigure}
  	\caption{Comparison of proposed properties among well-trained, corrupted and randomly-initialized \ourcpnn-VGG-16 models} 
	\label{fig:properties_full}
\end{figure*}
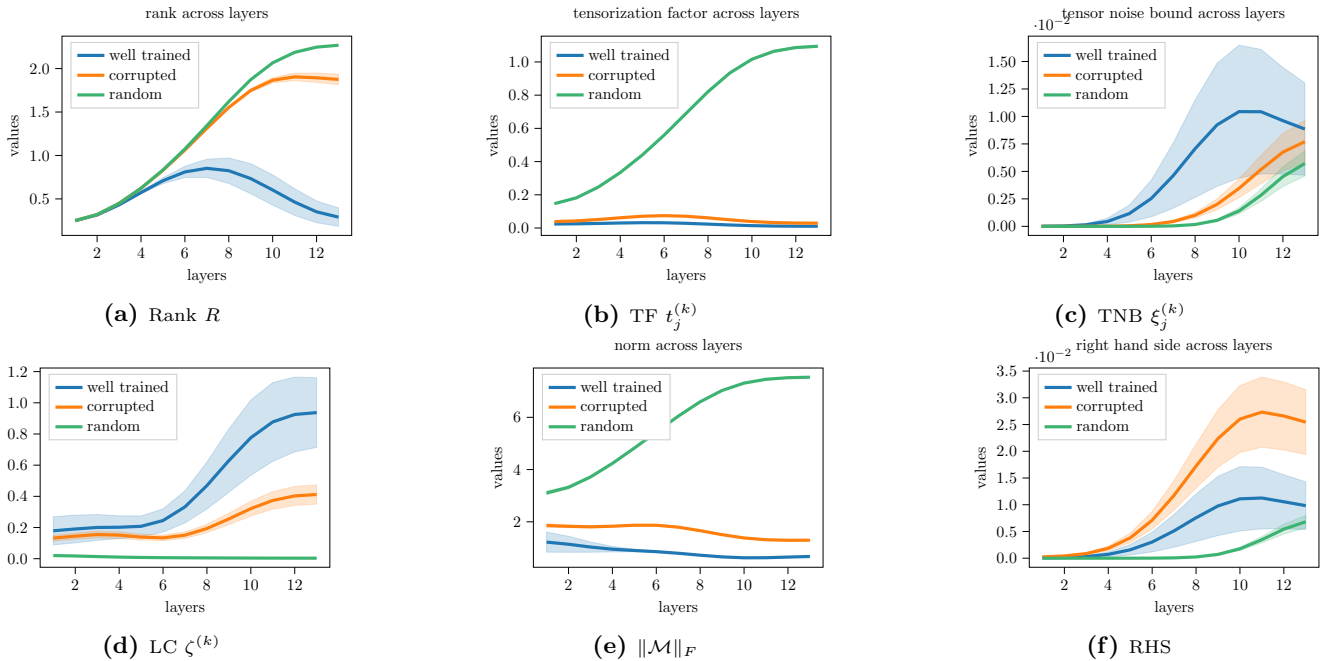

The reason why well-trained models have the largest \tflong is in Figure~\ref{fig:cp_spec_all_n} as the corrupted model usually has a very small leading value in its CP spectrum of later layers; yet as explained before, this does not necessarily indicate that corrupted models have more compressibility or low-rankness. The reason why the CP spectrum of corrupted models tend to have a small leading value is still a interesting question to study and we defer this to future work.

\textbf{Optimization settings for obtaining the well-trained, corrupted, and randomly initialized models of \ourcpnn-VGG-16.}  We obtain well-trained \ourcpnn-VGG-16 using the same hyperparameter settings as mentioned in Appendix Section~\ref{app:opt_setting}. For corrupted \ourcpnn-VGG-16, we train the model under 50\% of label noise but using the same set of hyperparameters as the well-trained models. For \ourcpnn-VGG-16 with random initialization, we just train the models for less than $1$ epoch.  For each set of these models, we obtain 200 instances using different random seeds. 


\section{Common Definitions and Propositions}
\label{app:definitions}

In this section, we will briefly review three key concepts underlying all analysis in this work, 
including {\em (multidimensional discrete Fourier transform)}, 
{\em CP decomposition} and {\em 2D-convolutional layer} in neural networks. 

\subsection{Multidimensional Discrete Fourier Transform (MDFT)}
\label{sub:Fourier}

\begin{definition} \label{def:mdft}
(Multidimensional discrete Fourier transform, MDFT)
An $m$-dimensional MDFT $\ourfft{}{m}$ defines a mapping 
from an $m$-order tensor $\mytensor{X} \in \Rbb^{N_1 \times \cdots \times N_m}$
to another {\em complex} $m$-order tensor $\Tilde{\mytensor{X}} \in \Cbb^{N_1 \times \cdots \times N_m}$ such that
\begin{equation}
\Tilde{\mytensor{X}}_{f_1,\dots, f_m}
= \left( \prod_{l = 1}^{m} N_l \right)^{-\frac{1}{2}} 
\sum_{n_1=1}^{N_1} \cdots \sum_{n_m = 1}^{N_m} 
\mytensor{X}_{n_1, \cdots, n_m} 
\left( \prod_{l = 1}^{m} \omega_{N_l}^{f_l n_l} \right)
\label{eq:mdft}
\end{equation}
where $\omega_{N_l}=\exp \left(-\mathrm{j} 2\pi/N_l\right)$ 
and $\left( \prod_{l = 1}^{m} N_l \right)^{-\frac{1}{2}}$ is the {\em normalization factor} that makes $\ourfft{}{m}$ unitary.
Through out the paper, we will use symbols with tilde (e.g. $\Tilde{\mytensor{X}}$) to denote tensors after MDFT.

MDFT can also be applied on a subset of the dimensions $\mathcal{I} \subseteq [m]$, 
and in this case we denote the mapping as $\ourfft{\mathcal{I}}{m}$.
\begin{equation}
\Tilde{\mytensor{X}}_{i_1,\dots, i_m}
= \left( \prod_{l \in \mathcal{I}} N_l \right)^{-\frac{1}{2}} 
\sum_{\forall l \in \mathcal{I}} 
\mytensor{X}_{n_1, \cdots, n_m} 
\left( \prod_{l \in \mathcal{I}} \omega_{N_l}^{f_l n_l} \right)
\label{eq:mdft-partial}
\end{equation}
where $i_l = f_l$ if $l \in \mathcal{I}$ and $i_l = n_l$ for $l \notin \mathcal{I}$. 

\end{definition}

\begin{fact} \label{fact:mdft-separability}
(Separability of MDFT)
An $m$-dimensional MDFT $\ourfft{}{m}$ is equivalent to a composition of $m$ unidimensional {DFT}s, i.e.\
\begin{equation}
\ourfft{}{m}=\ourfft{1}{m} \circ \ourfft{2}{m} \circ \cdots \circ \ourfft{m}{m}
\end{equation}
Similarly, $\ourfft{\mathcal{I}}{m}$ is identical to a composition of $|\mathcal{I}|$ unidimensional {DFT}s over corresponding dimensions. 
\end{fact}

\begin{fact} \label{fact:mdft-unitary}
(MDFT is unitary) For an MDFT $\ourfft{}{}$, its adjoint $\ourfft{\ast}{}$ is equal to its inverse $\ourfft{-1}{}$, i.e.\ $\ourfft{\ast}{} = \ourfft{-1}{}$. An immediate corollary of this property is that the operator norm is invariant to MDFT: Given an operator $\mathcal{A}$, its operator norm of $\mathcal{A}$ is equal to $\ourfft{\ast}{} \mathcal{A} \ourfft{}{}$, i.e.\ $\|\mathcal{A}\| = \| \ourfft{\ast}{} \mathcal{A} \ourfft{}{} \|$.
\end{fact}

\subsection{CP decomposition}
\label{sub:cp-decomposition}

\begin{definition} \label{def:cp-decomposition}
(CP decomposition) Given an $m$-order tensor $\mytensor{T} \in \Rbb^{N_1 \times \cdots \times N_m}$, a CP decomposition factorizes $\mytensor{T}$ into $m$ {\em core factors} $\{\mymatrix{K}^{l} \}_{l = 1}^{m}$ with $\mymatrix{K}^{l} \in \Rbb^{R \times N_l}$ (with its $r^{\tha}$ column as $\myvector{k}^{l}_{r} \in \Rbb^{N_l}$) such that
\begin{subequations}
\begin{gather}
\mytensor{T} =  \sum_{r = 1}^{R} \lambda_r \myvector{k}^{1}_{r} \otimes \cdots \otimes \myvector{k}^{m}_{r} 
\label{eq:cp-decomposition} \\
\mytensor{T}_{n_1, \cdots, n_m} = \sum_{r = 1}^{R} \lambda_r \mymatrix{K}^{1}_{r, n_1} \cdots \mymatrix{K}^{m}_{r, n_m}
\end{gather}
\end{subequations}
where each column $\myvector{k}^{l}_{r}$ has unit $\ell_2$ norm, 
i.e.\ $\|\myvector{k}^{l}_{r}\|_2 = 1, \forall r \in [R], l \in [m]$. 
Without loss of generality, we assume the {\em CP eigenvalues} are positive and sorted in decreasing order, 
i.e.\ $\lambda_1 \geq \lambda_2 \geq \cdots \geq \lambda_m > 0$. 
If the columns in $\mymatrix{K}^{l}$ are orthogonal, i.e.\ $\langle \myvector{k}^{l}_{r}, \myvector{k}^{l}_{r^{\prime}} \rangle = 1$ for $r \neq r^{\prime}$, the factorization is further named as {\em orthogonal CP decomposition}.
\end{definition}  

\begin{lemma} \label{lem:cp-mdft}
(MDFT of CP decomposition) If an $m$-order tensor $\mytensor{T} \in \Rbb^{N_1 \times \cdots \times N_m}$ takes a CP decomposition as in Eq.~\eqref{eq:cp-decomposition}, its (all-dimensional) MDFT $\mytensor{\Tilde{T}} = \ourfft{}{m}(\mytensor{T}) \in \Cbb^{N_1 \times \cdots \times N_m}$ also takes a CP format as
\begin{gather}
\Tilde{\mytensor{T}} = \sum_{r = 1}^{R} \lambda_r 
\tilde{\myvector{k}}^{1}_r \otimes \cdots \otimes \tilde{\myvector{k}}^{m}_r
\label{eq:cp-mdft} \\
\Tilde{\mytensor{T}}_{f_1, \cdots, f_m} = \sum_{r = 1}^{R} \lambda_r \tilde{\mymatrix{K}}^{1}_{r, f_1} \cdots \tilde{\mymatrix{K}}^{m}_{r, f_m}
\end{gather}
where $\tilde{\mymatrix{K}}^{l} = \ourfft{2}{2}(\mymatrix{K}^{l}), \forall l \in [m]$. The result can be extended to MDFT where a subset of dimensions are transformed.  
\end{lemma}

\begin{proof} (of Lemma~\ref{lem:cp-mdft}) 
According to the definition of multidimensional discrete Fourier transform, we have
\begin{align}
\Tilde{\mytensor{T}}_{n_1, \cdots, n_m} 
& = \left( \prod_{l = 1}^{m} N_l \right)^{-\frac{1}{2}} 
\sum_{n_1=1}^{N_1} \cdots \sum_{n_m = 1}^{N_m} 
\mytensor{T}_{n_1, \cdots, n_m} 
\left( \prod_{l = 1}^{m} \omega_{N_l}^{f_l n_l} \right) \\
& = \left( \prod_{l = 1}^{m} N_l \right)^{-\frac{1}{2}} 
\sum_{n_1=1}^{N_1} \cdots \sum_{n_m = 1}^{N_m} 
\left( \sum_{r = 1}^{R} \lambda_r \tilde{\mymatrix{K}}^{1}_{r, n_1} \cdots \tilde{\mymatrix{K}}^{m}_{r, n_m} \right)
\left( \prod_{l = 1}^{m} \omega_{N_l}^{f_l n_l} \right) \\
& = \sum_{r = 1}^{R} \lambda_r 
\left( N_1^{-1/2} \sum_{n_1 = 1}^{N_1} \mymatrix{K}^{1}_{r, n_1} \omega_{N_1}^{f_1 n_1} \right) \cdots
\left( N_m^{-1/2} \sum_{n_m = 1}^{N_m} \mymatrix{K}^{m}_{r, n_m} \omega_{N_m}^{f_m n_m} \right) \\
& = \sum_{r = 1}^{R} \lambda_r \tilde{\mymatrix{K}}^{1}_{r, f_1} \cdots \tilde{\mymatrix{K}}^{m}_{r, f_m}
\end{align}
which completes the proof.
\end{proof}

\subsection{2D-Convolutional Layer in Neural Networks}
\label{sub:cnn}

\begin{definition} \label{def:conv-layer}
(2D-convolutional layer) 
In {CNN}s, a 2D-convolutional layer is 
parametrized by a $4^{\tha}$-order tensor $\mytensor{M} \in \Rbb^{k_x \times k_y \times T \times S}$ (with $k_x \times k_y$ kernels). 
It defines a mapping from a $3^{{\mbox{\tiny rd}}}$-order input tensor $\mytensor{X} \in \Rbb^{H \times W \times S}$ (with $S$ channels)
to another $3^{{\mbox{\tiny rd}}}$-order output tensor $\mytensor{Y} \in \Rbb^{H \times W \times T}$ (with $T$ channels).
\begin{gather}
\mytensor{Y}_{\bm{:}, \bm{:}, t} 
= \sum_{s = 1}^{S} \mytensor{M}_{\bm{:}, \bm{:}, t, s} \ast \mytensor{X}_{\bm{:}, \bm{:}, s} \\
\mytensor{Y}_{i, j, t} = \sum_{s = 1}^{S} \sum_{p, q} \mytensor{M}_{i - p, j - q, t, s} \mytensor{X}_{p, q, s}
\label{eq:conv-layer}
\end{gather}
where $\ast$ represents a 2D-convolution operator.
\end{definition}

\begin{lemma} \label{lem:cnn-mdft}
(Convolutional theorem of 2D-convolutional layer) 
Suppose 
$\mytensor{\Tilde{X}} =  \ourfft{1,2}{3}(\mytensor{X}) \in \Cbb^{H \times W \times S}$, 
$\mytensor{\Tilde{M}} = \ourfft{1,2}{4}(\mytensor{M}) \in \Cbb^{H \times W \times T \times S}$ and 
$\mytensor{\Tilde{Y}} =  \ourfft{1,2}{3}(\mytensor{Y}) \in \Cbb^{H \times W \times T}$ 
are the MDFT of input, weights and outputs tensors 
$\mytensor{X}$, $\mytensor{W}$ and $\mytensor{Y}$ respectively, then 
these three tensors satisfy the following equation:
\begin{equation}
\mytensor{\Tilde{Y}}_{f, g, t} = \sqrt{H W} \sum_{s = 1}^{S} \mytensor{\Tilde{M}}_{f, g, t, s} \mytensor{\Tilde{X}}_{f, g, s}
\label{eq:cnn-mdft}
\end{equation}
Notice that the equation has a constant $\sqrt{H W}$ since we use a normalized MDFT. 
\end{lemma}

\begin{proof} (of Lemma~\ref{lem:cnn-mdft})
The theorem can be easily proved by applying MDFT on both sides of Eq.~\eqref{eq:conv-layer}.
\begin{align}
\mytensor{\Tilde{Y}}_{f, g, t} 
& = \frac{1}{\sqrt{H W}} \sum_{i, j} \mytensor{Y}_{i, j, t} \omega_{H}^{i f} \omega_{W}^{j g} \\
& = \frac{1}{\sqrt{H W}} \sum_{i, j} \left( \sum_{s = 1}^{S} \sum_{p, q} \mytensor{M}_{i - p, j - q, t, s} \mytensor{X}_{p, q, s} \right)  \omega_{H}^{i f} \omega_{W}^{j g} \\
& = \sqrt{H W} \sum_{s = 1}^{S} 
\left( \frac{1}{\sqrt{H W}} \sum_{i, j} \mytensor{M}_{i - p, j - q, t, s} \omega_{H}^{(i - p)f} \omega_{W}^{(j - q)g} \right)  
\left( \frac{1}{\sqrt{H W}} \sum_{p, q} \mytensor{X}_{p, q, s} \omega_{H}^{p f} \omega_{W}^{q g} \right) \\
& = \sqrt{H W} \sum_{s = 1}^{S} \mytensor{\Tilde{M}}_{f, g, t, s} \mytensor{\Tilde{X}}_{f, g, s}
\end{align}
\end{proof}

\begin{lemma} \label{lem:cnn-operator-norm}
(Operator norm of 2D-convolutional layer) Suppose we rewrite the tensors in matrix/vector form, i.e.\
$\mytensor{\tilde{X}}_{f, g, s} = \myvector{\tilde{x}}^{(f, g)}_s$, 
$\mytensor{\tilde{M}}_{f, g, t} = \mymatrix{\tilde{M}}^{(f, g)}_{t, s}$, 
$\mytensor{\tilde{Y}}_{f, g, t} = \myvector{\tilde{y}}^{(f, g)}_t$, then Eq.~\eqref{eq:cnn-mdft} can be written using matrix/vector products:
\begin{equation}
\myvector{\Tilde{y}}^{(f, g)}_{t} = 
\sum_{s = 1}^{S} \mymatrix{\tilde{M}}^{(f, g)}_{t, s} \myvector{\tilde{y}}^{(f, g)}_t, ~ \forall f, g
\end{equation}
The operator norm of $\mytensor{M}$, 
defined as $\|\mytensor{M}\| = \max_{\|\mytensor{X}\|_{F} = 1} \|\mytensor{Y}\|_{F}$, 
can be obtained by spectral norms of $\mymatrix{\Tilde{M}}^{(f, g)}$ as:
\begin{equation}
\|\mytensor{M}\| = \sqrt{H W} \max_{f, g} \left\| \mymatrix{M}^{(f, g)} \right\|_{2} 
\end{equation}
\end{lemma}

{\it Remarks.} The bound is first given by~\citet{sedghi2018singular}. 
In this work, we provide a much simpler proof compared to the original one in~\citet{sedghi2018singular}. 
In the next section, we show that the bound can be computed without evaluating the spectral norm
if the weights tensor $\mytensor{M}$ takes a CP format similar to Eq.~\eqref{eq:cp-decomposition}. 

\begin{proof} (of Lemma~\ref{lem:cnn-operator-norm})
From Fact~\ref{fact:mdft-unitary}, we know that $\|\mytensor{M}\| = \|\mytensor{\Tilde{M}}\|$, 
where $\|\mytensor{\Tilde{M}}\| = \max_{\|\mytensor{\Tilde{X}}\|_F = 1} \|\mytensor{\Tilde{Y}}\|_F$. 
Next, we bound $\|\mytensor{\Tilde{Y}}\|_{F}^{2}$ (i.e.\ $\sum_{f, g} \left\|\myvector{\tilde{y}}^{(f, g)}\right\|_{F}^{2}$)
assuming $\|\mytensor{\Tilde{X}}\|_{2}^{2} = 1$ (i.e.\ $\sum_{f, g} \left\|\myvector{\tilde{x}}^{(f, g)}\right\|_{2}^{2} = 1$).  
\begin{align}
\|\mytensor{\Tilde{Y}}\|_{F}^{2} 
& = \sum_{f, g} \left\| \myvector{\tilde{y}}^{(f, g)} \right\|_{2}^{2} \\
& \leq {H W} \sum_{f, g} \left\| \mymatrix{\Tilde{M}}^{(f, g)} \right\|^2 \left\| \myvector{\tilde{x}}^{(f, g)} \right\|_2^2 \\ 
& \leq {H W} \max_{f, g} \left\| \mymatrix{\Tilde{M}}^{(f, g)} \right\|^2
\sum_{f, g} \left\| \myvector{\tilde{x}}^{(f, g)} \right\|_2^2 \\ 
& = H W \max_{f, g} \left\| \mymatrix{\Tilde{M}}^{(f, g)} \right\|^2 \\
\|\mytensor{\Tilde{Y}}\|_{F} & \leq \sqrt{H W} \max_{f, g} \left\| \mymatrix{\Tilde{M}}^{(f, g)} \right\|
\end{align} 
We complete the proof by observing all inequalities can achieve equality simultaneously.
\end{proof}

\begin{definition} \label{def:tensor_product} 
(Tensor product) 
For vectors $\myvector{a} \in \mathbb{R}^n$, $\myvector{b} \in \mathbb{R}^m$, 
and $\myvector{c} \in \mathbb{R}^p$, their tensor product $\myvector{a} \otimes \myvector{b} \otimes \myvector{c}$ is a 3-way tensor in $\mathbb{R}^{m \times n \times p}$, with the $(i,j,k)^{\tha}$ entry being $\myvector{a}_i \myvector{b}_j \myvector{c}_k$. 
Similarly, for a matrix $\mymatrix{A} \in \mathbb{R}^{n \times m}$ and a vector $\myvector{c} \in \mathbb{R}^p$, their tensor product $\mymatrix{A} \otimes \myvector{c}$ is a $m \times n \times p$ tensor with the  $(i,j,k)^{\tha}$ entry being $\mymatrix{A}_{ij} \myvector{c}_k$.
 \end{definition}

\begin{definition} 
\label{def:cnn_kp}
(Kronecker product). Let $\mymatrix{A}$ be an $n\times p$ matrix and $\mymatrix{B}$ an $m\times q$ matrix. The $mn\times pq$ matrix 
 \begin{equation*}
 \mymatrix{A}\otimes\mymatrix{B}=
 \left[
 \begin{matrix}
    a_{1,1}\mymatrix{B} & a_{1,2}\mymatrix{B} & \cdots & a_{1,p}\mymatrix{B} \\
    a_{2,1}\mymatrix{B} & a_{2,2}\mymatrix{B} & \cdots & a_{2,p}\mymatrix{B} \\
    \vdots & \vdots & \vdots & \vdots \\
    a_{n,1}\mymatrix{B} & a_{n,2}\mymatrix{B} & \cdots & a_{n,p}\mymatrix{B}
  \end{matrix}
  \right]
  \end{equation*}
 is called the \textit{Kronecker product} of A and B. The outer product is an instance of \textit{Kronecker products}.
 \end{definition}


\section{CP Layers in Tensorial Neural Networks}
\label{app:cp-layers}

In this section, we will introduce three types of neural network layers, 
whose parameters are factorized in CP format 
as in Eq.~\eqref{eq:cp-decomposition} (with small variations). 
For brevity, we omit the layer superscript and 
denote the input, layer parameters and output
as $\mytensor{X}$, $\mytensor{M}$ and $\mytensor{Y}$, 
and we use $\mytensor{Y} = \mytensor{M}\left( \mytensor{X} \right)$ to denote the relations between $\mytensor{X}$, $\mytensor{M}$ and $\mytensor{Y}$.

\subsection{CP 2D-convolutional Layer}
\label{sub:cpl-conv}

\begin{definition} \label{def:cpl-conv}
(CP 2D-convolutional layer) 
For a given 2D-convolutional layer in Eq.~\eqref{eq:conv-layer}, 
a CP decomposition factorizes the weights tensor
$\mytensor{M} \in \Rbb^{H \times W \times T \times S}$ into three {\em core factors}
$\mytensor{C} \in \Rbb^{R \times k_x \times k_y}$, $\mymatrix{U} \in \Rbb^{R \times T}$, $\mymatrix{V} \in \Rbb^{R \times S}$ 
and a vector of {\em CP eigenvalues} $\lambda \in \Rbb^{R}$ such that
\begin{gather}
\mytensor{M} = \sum_{r = 1}^{R} \lambda_r \mytensor{C}_{r} \otimes \myvector{u}_{r} \otimes \myvector{v}_r \\ 
\mytensor{M}_{i, j, t, s} = \sum_{r = 1}^{R} \lambda_r \mytensor{C}_{r, i, j} \mymatrix{U}_{r, t} \mymatrix{V}_{r, s} 
\label{eq:cpl-conv}
\end{gather}
where $\lambda_r > 0$, $\|\mytensor{C}_r\|_F = 1$, $\|\myvector{u}_r\|_2 = 1$ and $\|\myvector{v}_r\|_2 = 1$ for all $r \in [R]$.
\end{definition}

\begin{lemma} \label{lem:cpl-conv-operator-norm}
(Operator norm of CP 2D-convolutional layer)
For a 2D-convolutional layer whose weights tensor takes 
a CP format as in Eq.~\eqref{eq:cpl-conv},
the operator norm $\|\mytensor{M}\|$ is bound by the CP eigenvalues $\lambda$ as   
\begin{equation}
\|\mytensor{M}\| \leq \sqrt{H W} \sum_{r = 1}^{R} |\lambda_r| \max_{f, g} \left| \Tilde{\mytensor{C}}^{(f, g)}_r \right|
\end{equation}
\end{lemma}

\begin{proof} (of Lemma~\ref{lem:cpl-conv-operator-norm})
From Fact~\ref{fact:mdft-unitary}, the operator norm of $\mytensor{M}$ 
is equal to the one of its MDFT $\mytensor{\Tilde{M}} = \ourfft{1, 2}{4}(\mytensor{M})$, 
i.e.\ $\|\mytensor{M}\| = \|\Tilde{\mytensor{M}}\|$. 
According to Lemma~\ref{lem:cnn-operator-norm}, 
it is sufficient to compute the spectral norm for each matrix $\Tilde{\mymatrix{M}}^{(f, g)}$ individually. 
Notice that if $\mytensor{M}$ takes a CP format, 
each $\Tilde{\mymatrix{M}}^{(f, g)}$ has a decomposed form as follows
\begin{subequations}
\begin{gather}
\mymatrix{\Tilde{M}}^{(f, g)} 
= \sum_{r = 1}^{R}  \lambda_r 
\mytensor{\Tilde{C}}^{(f, g)}_{r} \myvector{u}_{r} \myvector{v}_{r}^{\top} \\
\mymatrix{\Tilde{M}}^{(f, g)}_{t, s} = 
\sum_{r = 1}^{R} \lambda_r 
\mytensor{\Tilde{C}}^{(f, g)}_{r} \mymatrix{U}_{r, t} \mymatrix{V}_{r, s}
\end{gather}
\end{subequations}
where $\mytensor{\Tilde{C}} = \ourfft{2, 3}{3}(\mytensor{C})$ and $\mytensor{\Tilde{C}}^{(f, g)}_{r} = \mytensor{\Tilde{C}}_{r, f, g}$. The rest of the proof follows the definition of spectral norm of $\mymatrix{\Tilde{M}}$, i.e.\ $\|\Tilde{\mymatrix{M}}^{(f, g)}\|_2 = \max_{\|\myvector{a}\| = 1} \|\Tilde{\mymatrix{M}}^{(f, g)} \myvector{a}\|$. Let $\myvector{b} = \Tilde{\mymatrix{M}}^{(f, g)} \myvector{a}$, we can bound the $\ell_2$ norm of $\myvector{b}$:
\begin{align}
\left\|\myvector{b}\right\|_2 = \left\| \mymatrix{\Tilde{M}}^{(f, g)} \myvector{a} \right\|_2
& = \left\| \sum_{r = 1}^{R} \lambda_r \mytensor{\Tilde{C}}^{(f, g)}_{r} \myvector{u}_{r} \myvector{V}_{r}^{\top} \myvector{a} \right\|_2 \\
& \leq \sum_{r = 1}^{R} \left| \lambda_r \mytensor{\Tilde{C}}^{(f, g)}_r 
\left( \myvector{v}_{r}^{\top} \myvector{a} \right) \right| \left\|\myvector{\myvector{u}_r} \right\|_2 \\
& = \sum_{r = 1}^{R} \left| \lambda_r \mytensor{\Tilde{C}}^{(f, g)}_r \left( \myvector{v}_{r}^{\top} \myvector{a} \right) \right| \\
& \leq \sum_{r = 1}^{R} |\lambda_r| \left| \mytensor{\Tilde{C}}^{(f, g)}_r \right|
\end{align}
Therefore, $\|\mytensor{M}\| = \|\Tilde{\mytensor{M}}\| = \sqrt{H W} \max_{f, g} \|\mymatrix{\Tilde{M}}^{(f, g)}\| \leq \sqrt{H W} \sum_{r = 1}^{R} |\lambda_r| \max_{f, g} \left| \mytensor{\Tilde{C}}^{(f, g)}_r \right|$.
\end{proof}

\subsection{Higher-order CP Fully-connected Layer}
\label{sub:rcpl-fc}

\begin{definition} \label{def:reshaped-fc} 
(Higher-order fully-connected layer)
The layer is parameterized by a $2m^{\tha}$-order tensor 
$\mytensor{M} \in \Rbb^{T_1 \times \cdots \times T_m \times S_1 \times \cdots \times S_m}$. 
It maps an $m^{\tha}$-order input tensor $\mytensor{X} \in \Rbb^{S_1 \times \cdots \times S_m}$ to
another $m^{\tha}$-order output tensor $\mytensor{Y} \in \Rbb^{T_1 \times \cdots \times S_m}$ 
with the following equation:
\begin{equation}
\mytensor{Y}_{t_1, \cdots, t_m} = \sum_{\forall l: S_l}
\mytensor{M}_{t_1, \cdots, t_m, s_1, \cdots, s_m} \mytensor{X}_{s_1, \cdots, s_m}
\label{eq:reshaped-fc}
\end{equation}
\end{definition}

\begin{definition} \label{def:rcpl-fc}
(Higher-order CP fully-connected layer) 
Given a higher-order fully-connected layer in Eq.~\eqref{eq:reshaped-fc}, 
a CP decomposition factorizes the weights tensor 
$\mytensor{M} \in \Rbb^{T_1 \times \cdots \times T_m \times S_1 \times \cdots \times S_m}$ 
into $m$ core factors $\mytensor{K}^{m} \in \Rbb^{R \times T_m \times S_m}$.
\begin{equation}
\mytensor{M}_{t_1, \cdots, t_m, s_1, \cdots, s_m} = \sum_{r = 1}^{R} \lambda_r 
\mytensor{K}^{1}_{r, t_1, s_1} \cdots \mytensor{K}^{m}_{r, t_m, s_m}
\label{eq:rcpl-fc}
\end{equation}
For simplicity, we denote the $r^{\tha}$ slice of $\mytensor{K}^{l}$ as $\mymatrix{K}^{l}_{r} = \mytensor{K}^{l}_{r, \bm{:}, \bm{:}}$. We assume $\mymatrix{K}^{l}_r$ has unit Frobenius norm, i.e.\ $\|\mymatrix{K}^{l}_r\|_F = 1$ and $\lambda_r > 0$ for all $r \in [R]$.
\end{definition}

\begin{lemma} \label{lem:rcpl-fc-operator-norm}
(Operator norm of higher-order CP fully-connected layer) 
For a higher-order fully layer whose weights tensor takes 
a CP format as in Eq.~\eqref{eq:rcpl-fc}, 
the operator norm $\|\mytensor{M}\|$ is bound by the CP eigenvalues $\lambda$ as   
\begin{equation}
\|\mytensor{M}\| \leq \sum_{r = 1}^{R} |\lambda_r|
\end{equation}
\end{lemma}

\begin{proof} (of Lemma~\ref{lem:rcpl-fc-operator-norm})
The proof follows directly the definition of operator norm 
$\|\mytensor{M}\| = \max_{\|\mytensor{X}\|_F = 1} \|\mytensor{Y}\|_{F}$.
\begin{align}
\|\mytensor{Y}\|_{F} & \leq \sum_{r = 1}^{R} |\lambda_r| 
\left\| \mymatrix{K}^{l}_m \right\|_2 \cdots \left\| \mymatrix{K}^{l}_1\right\|_2 \|\mytensor{X}\|_{F} 
\\
& \leq \sum_{r = 1}^{R} |\lambda_r| 
\left\|\mymatrix{K}^{l}_m\right\|_F \cdots \left\| \mymatrix{K}^{l}_1\right\|_F \|\mytensor{X}\|_{F} \\
& =  \sum_{r = 1}^{R} |\lambda_r| \|\mytensor{X}\|_{F} = \sum_{r = 1}^{R} |\lambda_r|
\end{align} 
\end{proof}

\subsection{Higher-order 2D-convolutional layer}
\label{sub:rcpl-conv}

\begin{definition} \label{def:reshaped-conv-layer}
(Higher-order 2D-convolutional layer)
The layer is parameterized by a $(2m+2)^{\tha}$-order tensor 
$\mytensor{M} \in \Rbb^{k \times k \times T_1 \times \cdots \times T_m \times S_1 \times \cdots \times S_m}$. 
It maps an $(m+2)^{\tha}$-order input tensor $\mytensor{X} \in \Rbb^{H \times W \times S_1 \times \cdots \times S_m}$ to
another $(m+2)^{\tha}$-order output tensor $\mytensor{Y} \in \Rbb^{H \times W \times T_1 \times \cdots \times S_m}$ as:
\begin{subequations}
\begin{gather}
\mytensor{Y}_{\bm{:}, \bm{:}, t_1, \cdots, t_m} = 
\sum_{\forall l: s_l = 1}^{S_l} 
\mytensor{M}_{\bm{:}, \bm{:}, t_1, \cdot, t_m, s_1, \cdots, s_m} 
\ast \mytensor{X}_{\bm{:}, \bm{:}, s_1, \cdots, s_m} 
\label{eq:reshaped-conv-layer} \\
\mytensor{Y}_{i, j, t_1, \cdots, t_m} = \sum_{\forall l: s_l = 1}^{S_l} \sum_{p, q}
\mytensor{M}_{i - p, j - q, t_1, \cdots, t_m, s_1, \cdots, s_m} \mytensor{X}_{p, q, s_1, \cdots, s_m} 
\end{gather}
\end{subequations}
\end{definition}

\begin{definition} \label{def:rcpl-conv}
(CP decomposition of higher-order 2D-convolutional layer)
Given a higher-order 2D-convolutional layer in Eq.~\eqref{eq:reshaped-fc}, 
a CP decomposition factorizes the weights tensor 
$\mytensor{M} \in \Rbb^{H \times W \times T_1 \times \cdots \times T_m \times S_1 \times \cdots \times S_m}$ 
into $(m+1)$ core factors $\mytensor{C} \in \Rbb^{R \times H \times W}$ 
and $\mytensor{K}^{l} \in \Rbb^{R \times T_l \times S_l}, \forall l \in [m]$.
\begin{equation}
\mytensor{M}_{i, j, t_1, \cdots, t_m, s_1, \cdots, s_m} = 
\sum_{r = 1}^{R} \lambda_r \mytensor{C}_{r, i, j} 
\mytensor{K}^{1}_{r, t_1, s_1} \cdots \mytensor{K}^{m}_{r, t_m, s_m}
\label{eq:rcpl-conv}
\end{equation}
where we we assume $\mytensor{C}_r$ and 
$\mymatrix{K}^{l}_r = \mytensor{K}^{l}_{r, \bm{:}, \bm{:}}$ has unit Frobenius norm, 
i.e.\ $\|\mymatrix{K}^{l}_r\|_F = 1$ and $\|\mytensor{C}_r\|_F = 1$
\end{definition}

\begin{lemma} \label{lem:rcpl-conv-operator-norm}
(Operator norm of Higher-order CP 2D-convolutional layer)
For a higher-order 2D-convolutional layer layer whose weights tensor takes 
a CP format as in Eq.~\eqref{eq:rcpl-conv}, 
the operator norm $\|\mytensor{M}\|$ is bound by the CP eigenvalues $\lambda$ as   
\begin{equation}
\|\mytensor{M}\| \leq \sqrt{H W} \sum_{r = 1}^{R} |\lambda_r| 
\max_{f, g} \left| \Tilde{\mytensor{C}}^{(f, g)}_r \right|
\end{equation}
\end{lemma}

\begin{proof} (of Lemma~\ref{lem:rcpl-conv-operator-norm})
The proof is a combination of Lemmas~\ref{lem:cpl-conv-operator-norm} and~\ref{lem:rcpl-fc-operator-norm}. Let $\Tilde{\mytensor{M}} = \ourfft{1, 2}{m}(\mytensor{M})$, we have
\begin{gather}
\|\mytensor{M}\| = \|\mytensor{\Tilde{M}}\| = \sqrt{HW} \max_{f, g} \|\mytensor{\Tilde{W}}^{(f, g)}\| \\
\mytensor{\Tilde{M}}^{(f, g)} = \sum_{r = 1}^{R} \lambda_r \mytensor{\Tilde{C}}^{(f, g)}_r 
\mytensor{K}^{1}_{r, t_1, s_1} \cdots \mytensor{K}^{m}_{r, t_m, s_m}
\end{gather}
The operator norm is bounded using Lemma~\ref{lem:rcpl-fc-operator-norm}: 
$\|\mytensor{\Tilde{M}}^{(f, g)}\| \leq \sum_{r = 1}^{R} |\lambda_r| \max_{f, g} \left| \Tilde{\mytensor{C}}^{(f, g)}_r \right|$.
\end{proof}

\section{Convolutional Neural Networks: Compressibility and Generalization}
\label{app:cnn_proofs}

\subsection{Complete Proofs of Convolutional Neural Networks}
\label{sub:cnn_perturbation}

\begin{definition} \label{def:cnn_tf_full}
[\tflong $\tf^{(k)}_j$] 
The {\em {\tflong}s} $\left\{ \tf^{(k)}_j \right\}_{j=1}^{R^{(k)}}$ of the $k^{\tha}$ layer is defined as
\begin{equation}
\tf^{(k)}_j := \sum_{r = 1}^{j} \left| \lambda^{(k)}_{r} \right| \max_{f, g} \left| \Tilde{C}^{(f, g)}_r \right|
\label{eq:cnn_tf_full}
\end{equation}
where $\lambda_r^{(k)}$ is the $r^{\tha}$ largest value in the CP spectrum of $\mytensor{M}^{(k)}$.
\end{definition}

\begin{definition} \label{def:cnn_tnb_full}  
[\nblong $\nb^{(k)}_j$] 
The {\em \nblong} $\left\{ \nb^{(k)}_j \right\}_{j=1}^{R^{(k)}}$ of the $k^{\tha}$ layer 
measures the amplitudes of the remaining components after pruning the ones
with amplitudes smaller than the $\lambda^{(k)}_j$:
\begin{equation}
\nb^{(k)}_j := \sum_{r=j+1}^{R^{(k)}} \left| \lambda_{r}^{(k)} \right| \max_{f, g} \left| \Tilde{C}^{(f, g)}_r \right|
\label{eq:cnn_tnb_full}
\end{equation}
\end{definition}

\begin{definition} \label{def:cnn_lc_full} 
[layer cushion $\lc^{(k)}$]
As introduced in~\cite{arora2018stronger}, the layer cushion of the $k^{\tha}$ layer is defined to be the largest value $\lc^{(k)}$ such that for any $\mytensor{X}^{(k)} \in S$, 
\begin{equation}
\lc^{(k)} \frac{\fronorm{\mytensor{M}^{(k)}}}{\sqrt{H^{(k)}W^{(k)}}} \fronorm{\mytensor{X}^{(k)}}  \leq \fronorm{ \mytensor{M}^{(k+1)}}
\label{eq:cnn_lc_full}
\end{equation}
Following~\citet{arora2018stronger}, the layer cushion considers 
how much smaller the output $\fronorm{\mytensor{M}^{(k+1)} }$ of the $k^{\tha}$ layer (after activation) 
compared with the product between the weight tensor $\fronorm{M^{(k)} }$ and the input $\fronorm{ \mytensor{X}^{(k)}}$. Note that $H^{(k)}$ and $W^{(k)}$ are constants and will not influence the results of the theorem and the lemmas. For simplicity, we use $H$ and $W$ to denote the maximum $H^{(k)}$ and $W^{(k)}$ over the $n$ layers for the following proofs where upper bounds are desired.
\end{definition}

Given these definitions, we can bound the difference of outputs from 
a given model and its compressed counterpart. The following lemma characterizes 
the relation between the difference and the factors $\tf^{(k)}_j$, $\nb^{(k)}_j$, $\lc^{(k)}$.
\begin{lemma} \label{lem:cnn_compression}
(Compression bound of convolutional neural networks)
Suppose a convolutional neural network $\origcnn$ has $n$ layers, 
and each convolutional layer takes a CP format as in Eq.~\eqref{eq:cpl-conv} with rank $R^{(k)}$.
If an algorithm generates a compressed network $\hat{\origcnn}$ such that 
only $\hat{R}^{(k)}$ components with largest $\lambda^{(k)}_r$'s are retained at the $k^{\tha}$ layer, 
the difference of their outputs at the $m^{\tha}$ is bounded by $\mytensor{X}^{(m+1)}$ as
\begin{equation}
\fronorm{\mytensor{X}^{(m)} - \hat{\mytensor{X}}^{(m)})}
\leq \left( 
\sum_{k=1}^{m-1} \frac{\nb^{(k)}}{\lc^{(k)} \fronorm{\mytensor{M}^{(k)}}}
\prod_{l=k+1}^{m-1}\frac{\tf^{(l)}}{\lc^{(l)} \fronorm{\mytensor{M}^{(l)}}}
\right)
\fronorm{\mytensor{X}^{(m)})}
\label{eq:cnn_compression_layer}
\end{equation}
Therefore for the whole network with $n$ layers, 
the difference between $\origcnn(\mytensor{X})$ and $\hat{\origcnn}(\mytensor{X})$ is bounded by
\begin{equation}
\fronorm{\origcnn(\mytensor{\tinput}) - \hat{\origcnn}(\mytensor{\tinput})}
\leq \left( 
\sum_{k=1}^{n} \frac{\nb^{(k)}}{\lc^{(k)} \fronorm{\mytensor{M}^{(k)}}} 
\prod_{l=k+1}^{n}\frac{\tf^{(l)}}{\lc^{(l)} \fronorm{\mytensor{M}^{(l)}}}
\right)
\fronorm{\origcnn(\mytensor{\tinput})}
\label{eq:cnn_compression}
\end{equation}
\end{lemma}

\begin{proof} (of Lemma~\ref{lem:cnn_compression})
We prove this lemma by induction. For $m = 2$, the lemma holds since 
\begin{align}
\fronorm{\mytensor{X}^{(2)} - \hat{\mytensor{X}}^{(2)}} 
& = \fronorm{\relu{\mytensor{Y}^{(1)}} - \relu{\hat{\mytensor{Y}}^{(1)}}} \\
& \leq \fronorm{\mytensor{Y}^{(1)} - \hat{\mytensor{Y}}^{(1)}}
= \fronorm{\left(\mytensor{M}^{(1)} - \hat{\mytensor{M}}^{(1)} \right) \left(\mytensor{X}^{(1)} \right)} \\
& \leq \sqrt{HW} \nb^{(1)} \fronorm{\mytensor{X}^{(1)}} 
\leq \frac{\nb^{(1)}}{\lc^{(1)}\fronorm{\mytensor{M}^{(1)}}} \fronorm{\mytensor{M}^{(2)}}
\end{align}
where $\mytensor{Y} = \mytensor{M}(\mytensor{X})$ denotes the computation of a convolutional layer.
(1) The first inequality follows the Lipschitzness of the ReLU activations; 
(2) The second inequality uses Lemma~\ref{lem:cpl-conv-operator-norm}; 
and (3) the last inequality holds by the definition of $\lc^{(1)}$. 
For $m + 1> 2$, we assume the lemma already holds for $m$
\begin{align}
& \quad \fronorm{\mytensor{X}^{(m+1)} - \hat{\mytensor{X}}^{(m+1)}} 
= \fronorm{\relu{\mytensor{Y}^{(m)}} - \relu{\hat{\mytensor{Y}}^{(m)}}} \\
& \leq \fronorm{\mytensor{Y}^{(m)} - \hat{\mytensor{Y}}^{(m)}} 
= \fronorm{\mytensor{M}^{(m)}\left(\mytensor{X}^{(m)}\right) - 
\hat{\mytensor{M}}^{(m)}\left(\hat{\mytensor{X}}^{(m)}\right)} \\
& = \fronorm{\hat{\mytensor{M}^{(m)}}\left(\mytensor{X}^{(m)} - \hat{\mytensor{X}}^{(m)} \right) + 
\left(\mytensor{M}^{(m)} - \hat{\mytensor{M}}^{(m)}\right) \left(\mytensor{X}^{(m)}\right)} \\
& \leq \sqrt{HW} \left( \tf^{(m)} \fronorm{\mytensor{X}^{(m)} - \hat{\mytensor{X}}^{(m)}} + \nb^{(m)} \fronorm{\mytensor{X}^{(m)}} \right) \\
& \leq \tf^{(m)} \left( \sum_{k=1}^{m-1} \frac{\nb^{(k)}}{\lc^{(k)} \fronorm{\mytensor{M}^{(k)}}}
\prod_{l=k+1}^{m-1}\frac{\tf^{(l)}}{\lc^{(l)} \fronorm{\mytensor{M}^{(l)}}} \right) 
\fronorm{\mytensor{X}^{(m)}} + \frac{\nb^{(m)}}{\lc^{(m)}\fronorm{\mytensor{M}^{(m)}}} \fronorm{\mytensor{X}^{(m)}} \\
& \leq \left( \sum_{k=1}^{m} \frac{\nb^{(k)}}{\lc^{(k)} \fronorm{\mytensor{M}^{(k)}}}
\prod_{l=k+1}^{m}\frac{\tf^{(l)}}{\lc^{(l)} \fronorm{\mytensor{M}^{(l)}}} \right) \fronorm{\mytensor{M}^{(m+1)}}
\end{align}
which completes the induction.
\end{proof}

\begin{lemma} \label{lem:cnn_lemma}
For any convolutional neural network $\origcnn$ of $n$ layers satisfying the assumptions in section~\ref{sec:setup} and any error $0\leq\epsilon\leq 1$, Algorithm~\ref{alg:cnn_compress} generates a compressed tensorial neural network $\hat{\mytensor{\origcnn}}$ such that for any $\mytensor{\tinput}\in S$:
\begin{equation}
\fronorm{\origcnn(\mytensor{\tinput}) -  \hat{\origcnn}(\mytensor{\tinput})} \leq \epsilon \fronorm{\origcnn(\mytensor{\tinput})}
\end{equation}
The compressed convolutional neural network $\hat{\origcnn}$ has $\sum_{k=1}^n\hat{R}^{(k)}(s^{(k)} + o^{(k)}+ k_x^{(k)}k_y^{(k)}+1)$ total parameters, where each $\hat{R}^{(k)}$ satisfies: 
\begin{equation}
\hat{R}^{(k)} = \min \Big\{ j \in [R^{(k)}] | \nb^{(k)}_j \Pi_{i=k+1}^{n} \tf^{(i)}_j \leq \frac{\epsilon}{n} \Pi_{i=k}^n \lc^{(i)} \fronorm{\cnnweight^{(i)}} \Big\}
\label{eq:lemma_eq}
\end{equation}
\end{lemma}
\textit{Remark.} Equation~\eqref{eq:lemma_eq} is slightly different with equation~\ref{eq:main_eq}, as the margin $\gamma$ is replaced by a perturbation error $\epsilon$. Therefore, how well the compressed tensorial neural network can approximate the original network is related to the choice of $\hat{R}^{(k)}$. 
Notice that when $\hat{R^{(k)}} = R^{(k)}$, the inequality for the $k^{\tha}$ layer will be automatically satisfied as $\cfr^{(k)} = 0$ in this case by definition.

\begin{proof} (of Lemma~\ref{lem:cnn_lemma})
The proof is trivial by observing 
\begin{align}
& \frac{\nb^{(k)}}{\lc^{(k)} \fronorm{\mytensor{M}^{(k)}}} \leq \frac{\epsilon}{n} 
\prod_{i=k+1}^{n} \frac{\lc^{(i)} \fronorm{\mytensor{M}^{(i)}}}{\tf^{(i)}} \\
\implies ~ & \frac{\nb^{(k)}}{\lc^{(k)} \fronorm{\mytensor{M}^{(k)}}}
\prod_{i=k+1}^{n}\frac{\tf^{(i)}}{\lc^{(i)} \fronorm{\mytensor{M}^{(i)}}} \leq \frac{\epsilon}{n} \\
\implies ~ & \sum_{k=1}^{n} \frac{\nb^{(k)}}{\lc^{(k)} \fronorm{\mytensor{M}^{(k)}}}
\prod_{i=k+1}^{n} \frac{\tf^{(i)}}{\lc^{(l)} \fronorm{\mytensor{M}^{(i)}}} \leq \epsilon
\end{align}
\end{proof}

Before proving Theorem~\ref{thm:cnn_thm}, Lemma~\ref{lem:cnn_helper_lemma_3} (introduced below) is needed. 
\begin{lemma}
\label{lem:cnn_helper_lemma_3}
For any convolutional neural network $\origcnn$ of $n$ layers satisfying the assumptions in section~\ref{sec:setup} and any margin $\gamma \geq 0$, $\origcnn$ can be compressed to a tensorial convolutional neural network $\hat{\origcnn}$ with $\sum_{k=1}^{n} \hat{R}^{(k)}(s^{(k)}+t^{(k)}+k_x^{(k)} \times k_y^{(k)}+1)$ total parameters such that for any $\mytensor{\tinput} \in S$, $\hat{L_0}(\hat{\origcnn}) \leq \hat{L}_\gamma(\origcnn)$.
Here, for each layer $k$, 
\begin{equation}
\hat{R}^{(k)} = \min \Big\{ j \in [R^{(k)}] | \nb^{(k)}_j \Pi_{i=k+1}^{n} \tf^{(i)}_j \leq \frac{\epsilon}{n} \Pi_{i=k}^n \lc^{(i)} \fronorm{\cnnweight^{(i)}} \Big\}
\end{equation}
\end{lemma}

\begin{proof} (of Lemma~\ref{lem:cnn_helper_lemma_3}) 

If $\gamma \geq 2 \max_{\mytensor{\tinput}  \in S } \fronorm{\origcnn(\mytensor{\tinput})}$, for any pair $(\mytensor{\tinput}, y) \in S$, we have 
\begin{align*}
 |\origcnn(\mytensor{\tinput})[y] - \max_{j \neq y} \origcnn(\mytensor{\tinput})[j]|^2
 &\leq (| \origcnn(\mytensor{\tinput})[y]| + |\max_{j \neq y} \origcnn(\mytensor{\tinput})[j]|)^2 \\
 &\leq 4 \max_{\mytensor{\tinput}  \in S } \fronorm{\origcnn(\mytensor{\tinput})}^2 \\
 &\leq \gamma^2
\end{align*}
Then the output margin of $\origcnn$ cannot be greater than $\gamma$ for any $\mytensor{\tinput} \in S$. Thus $\hat{L}_\gamma(\origcnn) = 1$. 

If $\gamma < 2 \max_{\mytensor{\tinput}  \in S } \fronorm{\origcnn(\mytensor{\tinput})}$, setting $$\epsilon = \frac{\gamma}{2 \max_{\mytensor{\tinput}  \in S } \fronorm{\origcnn(\mytensor{\tinput})}}$$ in Lemma~\ref{lem:cnn_lemma}, we obtain a compressed fully-connected tensorial neural network $\hat{\tnn}$ with the desired number of parameters and $$\fronorm{\origcnn(\mytensor{\tinput}) -  \hat{\tnn}(\mytensor{\tinput})} < \frac{\gamma}{2} \Rightarrow \forall j, | \origcnn(\mytensor{\tinput})[j] - \hat{\tnn}(\mytensor{\tinput})[j] | < \frac{\gamma}{2}$$
Then for any pair $(\mytensor{\tinput}, y) \in S$, if $\origcnn(\mytensor{\tinput})[y] > \gamma + \max_{j \neq y} \origcnn(\mytensor{\tinput})[j]$, $\hat{\tnn}$ classifies $\mytensor{\tinput}$ correctly as well because:
$$\hat{\tnn}(\mytensor{\tinput})[y] > \origcnn(\mytensor{\tinput})[y] - \frac{\gamma}{2} >  \max_{j \neq y} \origcnn(\mytensor{\tinput})[j] + \frac{\gamma}{2} > \max_{j\neq y}\hat{\tnn}(\mytensor{\tinput})[j]$$
Thus, $\hat{L_0}(\hat{\tnn}) \leq \hat{L}_\gamma(\origcnn)$.
\end{proof}

Now we prove the main theorem~\ref{thm:cnn_thm} by bounding the covering number given any $\epsilon$.
\subsubsection{Covering Number Analysis for Convolutional Neural Network}
\begin{proof} (of Theorem~\ref{thm:cnn_thm}) 
To be more specific, let us bound the covering number of the compressed network $\hat{\origcnn}$ by approximating each parameter with accuracy $\acc$. 

\label{app:cnn_covering_number}
\begin{lemma}
\label{lem:cnn_covering_number}
For any given constant accuracy $\acc$, the covering number of the compressed convolutional network $\hat{\origcnn}$ is of order $\tilde{O}(d)$ where $d$ denotes the total number of parameters in $\hat{\origcnn}$: $d := \sum_{k=1}^{n} \hat{R}^{(k)}(s^{(k)} + o^{(k)} + k_x^{(k)} \times k_y^{(k)} + 1)$.

\end{lemma}
Let $\tilde{\cnnweight}$ denote the network after approximating each parameter in $\hat{\origcnn}$ with accuracy $\acc$ (and $\mytensor{\tilde{\cnnweight}}^{(k)}$ denote its weight tensor on the $k^{th}$ layer). Based on the given accuracy, we know that 
$\forall k, \ $
$\abs{\hat{\lambda}_{r}^{(k)} - \tilde{\lambda}_{r}^{(k)}} \leq \acc$,
$\norm{\hat{\myvector{a}}_{r}^{(k)} - \tilde{\myvector{a}}_{r}^{(k)}} \leq \sqrt{s^{(k)}}\acc$,
$\norm{\hat{\myvector{b}}_{r}^{(k)} - \tilde{\myvector{b}}_{r}^{(k)}} \leq \sqrt{o^{(k)}}\acc$,
$\norm{\hat{\mymatrix{C}}_{r}^{(k)} - \tilde{\mymatrix{C}}_{r}^{(k)}} \leq \sqrt{k_x^{(k)}k_y^{(k)}}\acc$, where $s$, $o$, $k_x$ and $k_y$ are the number of input channels, the number of output channels, the height of the kernel and the width of the kernel, as defined in Section~\ref{sec:setup}.
For simplicity, in this proof, let us just use $\mytensor{\tinput}^{(k)}, \mytensor{\tOut}^{(k)}, \myvector{a}_{r}^{(k)}, \myvector{b}_{r}^{(k)}, \mymatrix{C}_{r}^{(k)}$ to denote $\hat{\mytensor{\tinput}}^{(k)}, \hat{\mytensor{\tOut}}^{(k)}, \hat{\myvector{a}}_{r}^{(k)}, \hat{\myvector{b}}_{r}^{(k)}, \hat{\mymatrix{C}}_{r}^{(k)}$. $\mytensor{\tinput}^{(k)} \in \mathbb{R}^{H \times W \times s^{(k)}}$, $\mytensor{\tOut}^{(k)} \in \mathbb{R}^{H \times W \times o^{(k)}}$.

We have 
$$
\ourfft{1,2}{3}(\mytensor{\tOut}^{(k)})_{fgj} = \sqrt{HW} \sum_i[\ourfft{1,2}{3}(\mytensor{\tinput}^{(k)})_{fgi}\sum_{r=1}^{\hat{R}^{(k)}}\lambda_r^{(k)} a_{ri}^{(k)}b_{rj}^{(k)}\ourfft{}{2}(C_{r}^{(k)})_{fg}]
$$ 
$$\ourfft{1,2}{3}(\tilde{\mytensor{\tOut}}^{(k)})_{fgj} = \sqrt{HW} \sum_i[\ourfft{1,2}{3}(\tilde{\mytensor{\tinput}}^{(k)})_{fgi}\sum_{r=1}^{\hat{R}^{(k)}}\tilde{\lambda}_r^{(k)} \tilde{a}_{ri}^{(k)}\tilde{b}_{rj}^{(k)}\ourfft{}{2}(\tilde{C}_{r}^{(k)})_{fg}]$$ 
where $\sqrt{HW}$ is a normalization factor defined in Lemma~\ref{lem:cnn-mdft}

Let $\epsilon^{(k)} = \fronorm{\tilde{\mymatrix{Y}}^{(k)} - \mymatrix{Y}^{(k)}}$. Then for each $k$, let $\varphi=\sum_{f,g,i,j}\Big(\sum_r^{\hat{R}^{(k)}}\lambda_r^{(k)} a_{ri}^{(k)} b_{rj}^{(k)} \big(\ourfft{}{2}(C_r^{(k)})_{fg}-\ourfft{}{2}(\tilde{C}_r^{(k)})_{fg}\big)\Big)^2$ and $\psi=\sum_{f,g,i,j}\Big( \sum_r^{\hat{R}^{(k)}}\big(\lambda_r^{(k)}a_{ri}^{(k)}b_{rj}^{(k)}-\tilde{\lambda}_r^{(k)}\tilde{a}_{ri}^{(k)}\tilde{b}_{rj}^{(k)}\big) \ourfft{}{2}(\tilde{C}_r^{(k)})_{fg}\Big)^2$. We first bound $\varphi$ and $\psi$ as follows.

\textbf{Bound $\varphi=\sum_{f,g,i,j}\Big(\sum_r^{\hat{R}^{(k)}}\lambda_r^{(k)} a_{ri}^{(k)} b_{rj}^{(k)} \big(\ourfft{}{2}(C_r^{(k)})_{fg}-\ourfft{}{2}(\tilde{C}_r^{(k)})_{fg}\big)\Big)^2$:}  All calculations are based on the $k^{th}$ layer, we remove the layer number $(k)$ for ease of reading. So $a=a^{(k)}$ (the same for $b$, $c$, and $R$). Then

\begin{equation*}
\begin{split}
\MoveEqLeft \sum_{f,g,i,j}\Big(\sum_r^{\hat{R}^{(k)}}\lambda_r^{(k)} a_{ri}^{(k)} b_{rj}^{(k)} \big(\ourfft{}{2}(C_r^{(k)})_{fg}-\ourfft{}{2}(\tilde{C}_r^{(k)})_{fg}\big)\Big)^2 \\
&\leq \sum_{f,g,i,j}\Big( \sum_r^{\hat{R}}(\lambda_r a_{ri} b_{rj})^2 \sum_r^{\hat{R}}\big( \ourfft{}{2}(C_r)_{fg}-\ourfft{}{2}(\tilde{C}_r)_{fg} \big)^2 \Big) \\
&\leq \sum_r^{\hat{R}} (\lambda_r^2 \sum_i a_{ri}^2 \sum_j b_{rj}^2) \sum_r^{\hat{R}} \sum_{f,g} \big(\ourfft{}{2}(C_r)_{fg} - \ourfft{}{2}(\tilde{C}_r)_{fg} \big)^2  \\
&\leq \sum_r^{\hat{R}}\lambda_r^2 \hat{R} k_x k_y \acc^2
\end{split}
\end{equation*}

\textbf{Bound $\psi=\sum_{f,g,i,j}\Big( \sum_r^{\hat{R}^{(k)}}\big(\lambda_r^{(k)}a_{ri}^{(k)}b_{rj}^{(k)}-\tilde{\lambda}_r^{(k)}\tilde{a}_{ri}^{(k)}\tilde{b}_{rj}^{(k)}\big) \ourfft{}{2}(\tilde{C}_r^{(k)})_{fg}\Big)^2$:}  Similarly, we remove the layer number $(k)$ for ease of reading. Then we have

\begin{equation*}
\begin{split}
\MoveEqLeft \sum_{f,g,i,j}\Big( \sum_r^{\hat{R}^{(k)}}\big(\lambda_r^{(k)}a_{ri}^{(k)}b_{rj}^{(k)}-\tilde{\lambda}_r^{(k)}\tilde{a}_{ri}^{(k)}\tilde{b}_{rj}^{(k)}\big) \ourfft{}{2}(\tilde{C}_r^{(k)})_{fg}\Big)^2 \\
&\leq \sum_{f,g,i,j}\big( \sum_r^{\hat{R}}(\lambda_r a_{ri} b_{rj} - \tilde{\lambda}_r \tilde{a}_{ri} \tilde{b}_{rj})^2 \sum_r^{\hat{R}} \ourfft{}{2}(\tilde{C}_r)_{fg}^2 \big) \\
&= \sum_{f,g,i,j}\Big( \sum_r^{\hat{R}}\big(\lambda_r (a_{ri} b_{rj}-\tilde{a}_{ri} \tilde{b}_{rj}) + (\lambda_r-\tilde{\lambda}_r) \tilde{a}_{ri} \tilde{b}_{rj} \big)^2 \sum_r^{\hat{R}} \ourfft{}{2}(\tilde{C}_r)_{fg}^2 \Big) \\
&\leq \sum_{f,g,i,j} \Big( \big( 2\sum_r^{\hat{R}}\lambda_r^2(a_{ri} b_{rj}-\tilde{a}_{ri} \tilde{b}_{rj})^2 + 2\sum_r^{\hat{R}} (\lambda_r-\tilde{\lambda}_r)^2 \tilde{a}_{ri}^2 \tilde{b}_{rj}^2 \big)  \sum_r^{\hat{R}} \ourfft{}{2}(\tilde{C}_r)_{fg}^2\Big) \\
&= \sum_{f,g,i,j} \Bigg( \bigg( 2\sum_r^{\hat{R}}\lambda_r^2 \big( a_{ri} (b_{rj}-\tilde{b}_{rj}) + (a_{ri}-\tilde{a}_{ri}) \tilde{b}_{rj} \big)^2 + 2\sum_r^{\hat{R}} (\lambda_r-\tilde{\lambda}_r)^2 \tilde{a}_{ri}^2 \tilde{b}_{rj}^2 \bigg)  \sum_r^{\hat{R}} \ourfft{}{2}(\tilde{C}_r)_{fg}^2\Bigg) \\
&\leq \sum_{f,g,i,j} \Bigg( \bigg( 4\sum_r^{\hat{R}}\lambda_r^2 \big( a_{ri}^2 (b_{rj}-\tilde{b}_{rj})^2 + (a_{ri}-\tilde{a}_{ri})^2 \tilde{b}_{rj}^2 \big)^2 + 2\sum_r^{\hat{R}} (\lambda_r-\tilde{\lambda}_r)^2 \tilde{a}_{ri}^2 \tilde{b}_{rj}^2 \bigg)  \sum_r^{\hat{R}} \ourfft{}{2}(\tilde{C}_r)_{fg}^2\Bigg) \\
&= \bigg( 4\sum_r^{\hat{R}}\lambda_r^2 \big( \sum_i a_{ri}^2 \sum_j(b_{rj}-\tilde{b}_{rj})^2 + \sum_i(a_{ri}-\tilde{a}_{ri})^2 \sum_j\tilde{b}_{rj}^2 \big)^2 + 2\sum_r^{\hat{R}} (\lambda_r-\tilde{\lambda}_r)^2 \sum_i\tilde{a}_{ri}^2 \sum_j\tilde{b}_{rj}^2 \bigg)  \sum_r^{\hat{R}} \sum_{f.g}\ourfft{}{2}(\tilde{C}_r)_{fg}^2 \\
&\leq \big(4\sum_r^{\hat{R}}\lambda_r^2(o\acc^2+s\acc^2)+2\hat{R}\acc^2\big)\hat{R} \\
&= \big(4\sum_r^{\hat{R}}\lambda_r^2(o+s)\hat{R}+2\hat{R}^2\big) \acc^2
\end{split}
\end{equation*}

\textbf{Bound $\epsilon^{(k)}=\fronorm{\tilde{\mytensor{\tOut}}^{(k)}-\mytensor{\tOut}^{(k)}}$:}  Similarly, we remove the layer number $(k)$. And we let $w_i=\ourfft{1,2}{3}(\mytensor{\tinput}^{(k)})_{fgi}$, $\tilde{w}_i=\ourfft{1,2}{3}(\tilde{\mytensor{\tinput}}^{(k)})_{fgi}$, $u_i=\sum_r^{\hat{R}}\lambda_r^{(k)}a_{ri}^{(k)}b_{rj}^{(k)}\ourfft{}{2}(C_r^{(k)})_{fg}$ and $\tilde{u}_i=\sum_r^{\hat{R}}\tilde{\lambda}_r^{(k)}\tilde{a}_{ri}^{(k)}\tilde{b}_{rj}^{(k)}\ourfft{}{2}(\tilde{C}_r^{(k)})_{fg}$.

\begin{equation*}
\begin{split}
\MoveEqLeft \fronorm{\tilde{\mytensor{\tOut}}^{(k)}-\mytensor{\tOut}^{(k)}}^2 \\
&= \fronorm{\ourfft{1,2}{3}(\tilde{\mytensor{\tOut}}^{(k)})-\ourfft{1,2}{3}(\mytensor{\tOut}^{(k)})}^2 \\
&= \sum_{f,g,j} \absbig{[\ourfft{1,2}{3}(\tilde{\mytensor{\tOut}}^{(k)})]_{fgj}-[\ourfft{1,2}{3}(\mytensor{\tOut}^{(k)})]_{fgj}}^2 \\
&= \sum_{f,g,j} HW (\sum_i w_i u_i-\sum_i \tilde{w}_i \tilde{u}_i)^2 \\
&= HW \sum_{f,g,j} \big( \sum_i w_i(u_i-\tilde{u}_i) + \sum_i (w_i-\tilde{w}_i) \tilde{u}_i \big)^2 \\
&\leq 2HW \sum_{f,g,j} \big( \sum_i w_i(u_i-\tilde{u}_i)\big)^2 + 2 \sum_{f,g,j}\big(\sum_i (w_i-\tilde{w}_i) \tilde{u}_i \big)^2 \\
&\leq 2HW \sum_{f,g,j}\big((\sum_i w_i^2) \sum_i (u_i-\tilde{u}_i)^2\big) + 2\sum_{f,g,j}\big(\sum_i(w_i-\tilde{w}_i)^2(\sum_i \tilde{u}_i)^2\big) \\
&\leq 2HW (\sum_{f,g,i}w_i^2)\sum_{f,g,i,j}(u_i-\tilde{u}_i)^2 + 2\sum_{f,g,i}(w_i-\tilde{w}_i)^2\sum_{f,g,i,j}\tilde{u}_i^2 \\
&\leq 2HW (\sum_{f,g,i}w_i^2)(2\varphi + 2\psi)+ 2\sum_{f,g,i}(w_i-\tilde{w}_i)^2\sum_{f,g,i,j}\tilde{u}_i^2 \\
&\leq 4HW \fronorm{\mytensor{\tinput}^{(k)}}^2 \acc^2 \big( \sum_r^{\hat{R}}\lambda_r^2 \hat{R} k_x k_y + 4\sum_r^{\hat{R}}\lambda_r^2(o+s)\hat{R}+2\hat{R}^2\big) + 2\sum_{f,g,i}(w_i-\tilde{w}_i)^2\sum_{f,g,i,j}\tilde{u}_i^2 \\
&\leq 4HW \fronorm{\mytensor{\tinput}^{(k)}}^2 \acc^2 \big( \sum_r^{\hat{R}}\lambda_r^2 \hat{R} k_x k_y + 4\sum_r^{\hat{R}}\lambda_r^2(o+s)\hat{R}+2\hat{R}^2\big) + 2(\fronorm{\mytensor{\tinput}^{(k)}-\tilde{\mytensor{\tinput}}^{(k)}}^2 \fronorm{\tilde{\cnnweight}}^2) \\
\end{split}
\end{equation*}

When $k=1$, we know that $\mytensor{\tinput}^{(1)}=\tilde{\mytensor{\tinput}}^{(1)}$, so 

\begin{equation*}
\begin{split}
\MoveEqLeft \fronorm{\tilde{\mytensor{\tOut}}^{(1)}-\mytensor{\tOut}^{(1)}}^2 \\
&\leq 4HW \fronorm{\mytensor{\tinput}^{(1)}}^2 \acc^2 \big( \sum_r^{\hat{R}}\lambda_r^2 \hat{R} k_x k_y + 4\sum_r^{\hat{R}}\lambda_r^2(o+s)\hat{R}+2\hat{R}^2\big)\\
\end{split}
\end{equation*}

When $k>1$, we have

\begin{equation*}
\begin{split}
\MoveEqLeft \fronorm{\tilde{\mytensor{\tOut}}^{(k)}-\mytensor{\tOut}^{(k)}}^2 \\
&\leq 4HW \fronorm{\mytensor{\tinput}^{(k)}}^2 \acc^2 \big( \sum_r^{\hat{R}}\lambda_r^2 \hat{R} k_x k_y + 4\sum_r^{\hat{R}}\lambda_r^2(o+s)\hat{R}+2\hat{R}^2\big) + 2(\fronorm{\mytensor{\tinput}^{(k)}-\tilde{\mytensor{\tinput}}^{(k)}}^2 \fronorm{\tilde{\cnnweight}}^2) \\
&\leq 4HW \fronorm{\mytensor{\tinput}^{(k)}}^2 \acc^2 \big( \sum_r^{\hat{R}}\lambda_r^2 \hat{R} k_x k_y + 4\sum_r^{\hat{R}}\lambda_r^2(o+s)\hat{R}+2\hat{R}^2\big) + 2(\fronorm{\relu{\mytensor{\tOut}^{(k-1)}}-\relu{\tilde{\mytensor{\tOut}}^{(k-1)}}}^2 \fronorm{\tilde{\cnnweight}}^2) \\
&\leq 4HW \fronorm{\mytensor{\tinput}^{(k)}}^2 \acc^2 \big( \sum_r^{\hat{R}}\lambda_r^2 \hat{R} k_x k_y + 4\sum_r^{\hat{R}}\lambda_r^2(o+s)\hat{R}+2\hat{R}^2\big) + 2(\fronorm{\mytensor{\tOut}^{(k-1)}-\tilde{\mytensor{\tOut}}^{(k-1)}}^2 \fronorm{\tilde{\cnnweight}}^2) \\
&\leq 4HW \fronorm{\mytensor{\tinput}^{(k)}}^2 \acc^2 \big( \sum_r^{\hat{R}}\lambda_r^2 \hat{R} k_x k_y + 4\sum_r^{\hat{R}}\lambda_r^2(o+s)\hat{R}+2\hat{R}^2\big) + 2((\epsilon^{(k-1)})^2\fronorm{\tilde{\cnnweight}}^2) \\
\end{split}
\end{equation*}

Let $\alpha^{(k)}=4HW \fronorm{\mytensor{\tinput}^{(k)}}^2 \big( \sum_r^{\hat{R}^{(k)}}(\lambda_r^{(k)})^2 \hat{R}^{(k)} k_x^{(k)} k_y^{(k)} + 4\sum_r^{\hat{R}^{(k)}}(\lambda_r^{(k)})^2(o^{(k)}+s^{(k)})\hat{R}^{(k)}+2(\hat{R}^{(k)})^2 \big) \acc^2$, \\
and $\beta^{(k)}=2\fronorm{\tilde{\cnnweight}^{(k)}}^2$. Then the difference between the final output of the two networks are bounded by:
\begin{equation*}
\begin{split}
\MoveEqLeft \fronorm{\hat{\origcnn}(\mytensor{\tinput})-\tilde{\origcnn}(\mytensor{\tinput})}^2 \\
&= \fronorm{\relu{(\hat{\mytensor{\tOut}})}-\relu{(\tilde{\mytensor{\tOut}})}}^2 \\
&\leq \fronorm{\hat{\mytensor{\tOut}}-\tilde{\mytensor{\tOut}}}^2 \\
&\leq \sum_{k=1}^n\alpha^{(k)}\prod_{i=k+1}^n\beta^{(i)}
\end{split}
\end{equation*}

Since $\forall k \in [n], \norm{\mytensor{\tinput}^{(k)}} \leq \Pi_{i = k}^n \frac{\fronorm{\mytensor{\tinput}^{(n+1)}}}{\lc^{(i)} \fronorm{\cnnweight^{(i)}}} $, to obtain an $\epsilon$-cover of the compressed network, we can first assume $\beta^{(k)} \geq 1 \ \forall k \in [n]$. Then $\acc$ need to satisfy:

$$\acc \leq \frac{\epsilon}{ 2\sqrt{HW}n\fronorm{\mytensor{\tinput}^{(n+1)}}  \hat{R}^{(*)} (\frac{\sqrt{2}\fronorm{\tilde{\cnnweight}^{(*)}}}{\lc^{(*)}\fronorm{\cnnweight^{(*)}}})^n \sqrt{(\lambda^{(*)})^2 k_x^{(*)} k_y^{(*)} + 4(\lambda^{(*)})^2(o^{(*)}+s^{(*)})+2}}$$

where $\hat{R}^{(*)} = \max_{k} r^{(k)}$ $\lambda^{(*)} = \max_{r, k} \lambda_{r}^{(k)}$, $s^{(*)} = \max_{k} s^{(k)}$, $o^{(*)} = \max_{k} o^{(k)}$, $k_x^{(*)} = \max_{k} k_x^{(k)}$, $k_y^{(*)} = \max_{k} k_y^{(k)}$ and 
$\frac{\fronorm{\tilde{\cnnweight}^{(*)}}} { \mu^{(*)} \fronorm{\cnnweight^{(*)}}} = \max_{k} \frac{\fronorm{\tilde{\cnnweight}^{(k)}}} {\mu^{(k)} \fronorm{\cnnweight^{(*)}}}$

As when $\acc$ is fixed, the number of networks in our cover will at most be $(\frac{1}{\acc})^d$ where $d$ denote the number of parameters in the compressed network. Hence, the covering number w.r.t to a given $\epsilon$ is $\tilde{O}(n d)$ ($n$ is the number of layers in the given neural network). As for practical neural networks, the number of layers $n$ is usually much less than $O(\log(d))$, thus the covering number we obtained w.r.t to a given $\epsilon$ is just $\tilde{O}(d)$ for practical neural networks.

\end{proof}


\section{Fully Connected Networks: Compressibility and Generalization}
In this section, we derive generalization bounds for fully connected (FC) neural networks (denoted as $\orignn$) using tensor methods.

\subsection{Compression of a FC Network with \ourNN}
\label{sec:fcnn}

\textbf{Original Fully Connected Neural Network:} Let $\orignn$ denote an \textbf{$n$}-layer fully connected network with \reluname activations, where $\mymatrix{A}^{(k)} \in \Rbb^{h^{(k)} \times h^{(k+1)}}$ denotes the weight matrix of the $k^{\tha}$ layer, $\myvector{\originput}^{(k)} \in \Rbb^{h^{(k)}}$ denotes the input to $k^{\tha}$ layer, and $\myvector{\origOut}^{(k)}$ denotes the output of the $k^{\tha}$ layer before activation in $\orignn$. 
\textbf{Transform original FCN to a CP-FCN:} We transform the original fully connected network $\orignn$ to a network $\tnn$ with \ourNN. The $k^{\tha}$ layer of $\tnn$ is denoted by $\nnweight^{(k)}\in \Rbb^{s^{(k)}_1 \times s^{(k)}_2 \times s^{(k+1)}_1 \times s^{(k+1)}_2}$ is a 4-dimensional tensor reshaped from $\mymatrix{A}^{(k)}$ where $s^{(k)}_1 \times s^{(k)}_2 = h_{k}, \forall k \in [n]$.

\textbf{Input and Output of $\tnn$:}
The original input and output vectors of $\orignn$ are reshaped into matrices. 
The input to the $k^{\tha}$ layer of the $\tnn$, denoted by $\mymatrix{\tinput}^{(k)} \in \Rbb^{s^{(k)}_1 \times s^{(k)}_2}$, is a matrix reshaped from the input vector $\myvector{\originput}^{(k)}$ of the $k^{\tha}$ layer in the original network $\orignn$. 
Similarly, the output of the $k^{\tha}$ layer  before activation in $\tnn$, denoted by $\mymatrix{\tOut}^{(k)} \in \Rbb^{s^{(k)}_1 \times s^{(k)}_2}$,  is a matrix reshaped from the output vector $\myvector{\origOut}^{(k)}$ of the $k^{\tha}$ layer in the original network $\orignn$. 
For prediction purposes, we reshape the output $\mymatrix{\tOut}^{(n)}$ of the last layer in $\tnn$ back into a vector. So the final outputs of $\orignn$ and $\tnn$ are of the same dimension. 

\begin{assumption}[\textbf{Polyadic Form of $\tnn$}]\label{assum:CP-tnn}
For each layer $k$, assume the weight tensor $\mytensor{\nnweight}^{(k)}$ of $\tnn$ has a Polyadic form with rank $R^{(k)} \leq \min\{s_1^{(k)}, s_2^{(k)}, s_1^{(k+1)}, s_2^{(k+1)} \}$:
\begin{equation}\label{eq:cpform-tnn-fc}
\nnweight^{(k)} = \sum_{i = 1}^{R^{(k)}} \lambda_{i}^{(k)} a_{i}^{(k)} \otimes b_{i}^{(k)} \otimes c_{i}^{(k)} \otimes d_{i}^{(k)}
\end{equation}
where $\forall i, a_i, b_i, c_i, d_i$ are unit vectors in $\Rbb^{s_1^{(k)}},  \Rbb^{s_2^{(k)}}, \Rbb^{s_1^{(k+1)}}, \Rbb^{s_2^{(k+1)}}$ respectively, and $\forall 1 \leq i \leq R^{(k)}, \inp{a_i}{a_i} = 1,  \inp{b_i}{b_i} = 1,  \inp{c_i}{c_i} = 1,  \inp{d_i}{d_i} = 1 $.
Moreover, for each $\nnweight^{(k)}$, $\lambda_{i}^{(k)} \geq \lambda_{i+1}^{(k)},  \forall i$, and the absolute value of the smallest $| \lambda_{R^{(k)}}^{(k)} |$ can be arbitrarily small. 
\end{assumption}

The total number of parameters in $\tnn$ is $(s_1^{(k)} + s_2^{(k)} + s_1^{(k+1)} + s_2^{(k+1)}+1) R^{(k)}$ and a
smaller $R^{(k)}$ renders fewer number of parameters and thus leads to compression.
We introduce a compression mechanism that prunes out the smaller components of weight tensor of $\tnn$, i.e., a low rank approximation of each weight tensor $\nnweight^{(k)}$ of the $k^{\tha}$ layer, and generates a compressed CP-FCN $\hat{\tnn}$. The algorithm is depicted in Algorithm~\ref{alg:tfc_compress}.

\textbf{Compression of a FC Network with \ourNN:} In~\cite{li2018guaranteed}, a tensor decomposition algorithm (procedure 1 in~\cite{li2018guaranteed}) on tensors with asymmetric orthogonal components is guaranteed to recover the top-$r$ components with the largest singular values.  To compress $\tnn$, we apply top-$\hat{R}^{(k)}$ ($\hat{R}^{(k)} \leq R^{(k)}$) CP decomposition algorithm on each $\nnweight^{(k)}$, obtaining the components from CP decomposition $(\hat{\lambda}_{i}^{(k)}, \hat{a}_{i}^{(k)}, \hat{b}_{i}^{(k)}, \hat{c}_{i}^{(k)}, \hat{d}_{i}^{(k)}$), $i \in [\hat{R}^{(k)}]$.
Therefore, we achieve \textbf{a compressed network $\hat{\tnn}$} of $\tnn$, and the $j^{\tha}$ layer of the compressed network  $\hat{\tnn}$ has weight tensor as follows
\begin{equation}\label{eq:compressedtnn_fc}
\mytensor{\hat{T}}^{(k)} = \sum_{i = 1}^{\hat{R}^{(k)}} \hat{\lambda}_{i}^{(k)} \hat{a}_{i}^{(k)} \otimes \hat{b}_{i}^{(k)}\otimes \hat{c}_{i}^{(k)} \otimes \hat{d}_{i}^{(k)}.
\end{equation}
As each $\nnweight^{(k)}$ has a low rank orthogonal CP decomposition by our assumption, the returned results $\{\hat{\lambda}_{i}^{(k)}, \hat{a}_{i}^{(k)}, \hat{b}_{i}^{(k)}, \hat{c}_{i}^{(k)}, \hat{d}_{i}^{(k)}\}_{i=1}^{\hat{R}^{(k)}}$ from procedure 1 in~\cite{li2018guaranteed} are perfect recoveries of $\{\lambda_{i}^{(k)}, a_{i}^{(k)}, b_{i}^{(k)}, c_{i}^{(k)}, d_{i}^{(k)}\}_{i=1}^{\hat{R}^{(k)}}$
according to the robustness theorem in~\cite{li2018guaranteed}.  Our compression procedure is depicted in Algorithm~\ref{alg:tfc_compress}.

\begin{algorithm}[h!]
\caption{\textbf{Compression of Fully Connected Neural Networks} \\ 
\scriptsize{${}^\square$\FBR (in Appendix~\ref{app:algos}) denotes a sub-procedure which calculates $\hat{R}^{(k)}$ such that $\fronorm{\orignn(\mymatrix{\tinput})-  \hat{\tnn}(\mymatrix{\tinput})} \leq \epsilon \fronorm{\orignn(\mymatrix{\tinput})}$ holds for any input $\mymatrix{\tinput}$ in the training dataset and for any given $\epsilon$.\\
${}^{\triangle}$\TNNPROJECT (in Appendix~\ref{app:algos}) denotes a sub-procedure which returns a compressed network $\hat{\origcnn}$ by pruning out the smaller components in the Polyadic form of the weight tensors in the original CNN.\\
More intuitions of the sub-procedures \FBR and \TNNPROJECT are described in Section~\ref{sec:notations_fc}.}
}\label{alg:tfc_compress}
\begin{algorithmic}[1]
\REQUIRE A FCN $\orignn$ of $n$ layers 
and a margin $\gamma$
\ENSURE A compressed $\hat{\tnn}$ whose expected error $L_0(\hat{\tnn}) \le \hat{L}_{\gamma}(\tnn) + \tilde{O}\Big(\sqrt{\frac{ \sum_{k=1}^{n} \hat{R}^{(k)}(2 s^{(k)} + 2 s^{(k+1)} + 1)}{m}}\Big)$ 
\STATE Calculate all layer cushions $\{ \lc^{(k)} \}_{k=1}^{n}$ based on definition~\ref{def:lc}
\STATE Pick $R^{(k)} = \min\{s^{(k)}, s^{(k+1)}\}$ for each layer $k$
\STATE If $\origcnn$ does not have \ourNN, apply a CP-decomposition to the weight tensor of each layer $k$
\STATE Set the perturbation parameter $\epsilon := \frac{\gamma}{2\max_{\mymatrix{\tinput}} \fronorm{\orignn(\mymatrix{\tinput})}} $
\STATE Compute number of components needed for each layer of the compressed network $\{ \hat{R}^{(k)} \}_{k=1}^{n} \gets \hyperref[alg:find_best_rank_cnn]{\text{\FBR}}^\square\Big(\{\nnweight^{(k)}\}_{k=1}^n, \{R^{(k)}\}_{k=1}^n, \{ \lc^{(k)} \}_{k=1}^{n }, \epsilon\Big)$ 
\STATE $\hat{\orignn} \gets$ \hyperref[alg:tfc_project]{\TNNPROJECT}${}^{\triangle}\Big(\orignn, \{\hat{R}^{(k)}\}_{i=1}^n\Big)$
\STATE Return the compressed convolutional neural network $\hat{\orignn}$ 
\end{algorithmic}
\end{algorithm}

We denote  the input matrix of the $k^{\tha}$ layer in $\hat{\tnn}$ as $\hat{\mymatrix{\tinput}}^{(k)}$, and the output matrix before activation as $\hat{\mymatrix{\tOut}}^{(k)}$. Note that $\mymatrix{\tinput}^{(1)} = \hat{\mymatrix{\tinput}}^{(1)}$ as the input data is not being modified.

Algorithm~\ref{alg:tfc_compress} is desigend for general neural networks. For neural networks with $\ourNNLong$, line 3 can be done by pruning out small components from CP decomposition, and only keeping top-$\hat{R}^{(k)}$ components.
For notation simplicity, assume for each layer in $\orignn$, the width of the $k^{\tha}$ layer is a square of some integer $s^{(k)}$.
Then the input to the $k^{\tha}$ layer of $\tnn$ is a ReLu transformation of the output of the $k-1^{\tha}$ layer as in equation~\eqref{eq:tnn_in_out}. The output of the $k^{\tha}$ layer of $\tnn$ is illustrated in equation~\eqref{eq:tnn_out} as the weight tensor which permits a CP forms as in equation~\eqref{eq:cpform-tnn-fc}.
\begin{align}
\mymatrix{\tinput}^{(k)} & = \relu{\mymatrix{\tOut}^{(k-1)}} \label{eq:tnn_in_out}\\
\mymatrix{\tOut}^{(k)} &=  \sum_{i = 1}^{\hat{R}^{(k)}} \lambda_{i}^{(k)} {a_{i}^{(k)}}^\top \mymatrix{\tinput}^{(k)} b_{i}^{(k)} c_{i}^{(k)} \otimes d_{i}^{(k)} + \mytensor{\phi}^{(k)}(\mymatrix{\tinput}^{(k)}) \label{eq:tnn_out}
\end{align}

where $\mytensor{\phi}^{(k)} = \sum_{i = \hat{R}^{(k)} + 1}^{R^{(k)}} \lambda_{i}^{(k)} a_{i}^{(k)} \otimes b_{i}^{(k)} \otimes c_{i}^{(k)} \otimes d_{i}^{(k)}$, $\mytensor{\phi}^{(k)}(\mymatrix{\tinput}^{(k)})$ denotes the multilinear operation of the tensor $\mytensor{\phi}^{(k)}$ on $\mymatrix{\tinput}^{(k)}$, i.e., $\{ \mytensor{\phi}^{(k)}(\mymatrix{\tinput}^{(k)}) \}_{i,j} = \sum_{k, l} \mytensor{\phi}^{(k)}_{i,j,k,l} \mymatrix{\tinput}^{(k)}_{k,l}$ and $a_{i}^{(k)}, b_{i}^{(k)}, \hat{a}_{i}^{(k)}, \hat{b}_{i}^{(k)} \in \Rbb^{s_k}$.
Similarly, the input and output of the $k^{\tha}$ layer of the compressed neural nets $\hat{\tnn}$ satisfy 
\begin{align}
\hat{\mymatrix{\tinput}}^{(k)}& = \relu{\hat{\mymatrix{\tOut}}^{(k-1)}}\\
\hat{\mymatrix{\tOut}}^{(k)} &= \sum_{i = 1}^{\hat{R}^{(k)}} \hat{\lambda}_{i}^{(k)} (\hat{a}_{i}^{(k)})^\top \hat{\mymatrix{\tinput}}^{(k)} \hat{b}_{i}^{(k)} \hat{c}_{i}^{(k)} \otimes \hat{d}_{i}^{(k)}.
\end{align} 

\subsection{Characterizing Compressibility of FC Networks with CPL}\label{sec:properties-tnn}
Now we characterize the compressibility of the fully connected network with \ourNN $\tnn$ through properties defined in the following, namely reshaping factor, tensorization factor, layer cushion and tensor noise bound.

\begin{definition} \label{def:rf} 
(reshaping factor). The \textit{reshaping factor} $\rf^{(k)}$ of layer $k$ is defined to be the smallest value $\rf^{(k)}$ such that for any $\myvector{\originput} \in S$, 
\begin{align}
\norm{\mymatrix{\tinput^{(k)}}}&  \leq \rf^{(k)}\fronorm{\mymatrix{\tinput^{(k)}}}
\end{align}
\end{definition}
The reshaping factor upper bounds the ratio between the spectral norm and Frobenius norm of the reshaped input in the $k^{\tha}$ layer over any data example in the training dataset. Reshaping the vector examples into matrix examples improves the compressibility of the network (i.e., renders smaller $\rf^{(k)}$) as illustrated and empirically verified in~\cite{su2019tensorial}.
Note that $\hat{\mymatrix{\tinput}}^{(k)} $ is the input to the $k^{\tha}$ layer of the compressed network $\hat{\tnn}$, and $\rf^{(k)} \leq 1, \forall k$.

\begin{definition} \label{def:tf}
(tensorization factor) 
The \textit{tensorization factor} $\{ \tf_j^{(k)} \}_{j=1}^{R^{(k)}} $ of the $k^{\tha}$ layer regarding the network with \ourNN $\tnn$ and the original network $\orignn$ is defined as: 
 \begin{equation}
 \tf_j^{(k)} = \sum_{r=1}^{j} |\lambda_r^{(k)}|, \forall j.
 \end{equation}
\end{definition}
The tensorization factor measures the amplitudes of the leading components. By Lemma~\ref{lem:rcpl-fc-operator-norm}, the tensorization factor is the upper bound of operator norm of the weight tensor. 

\begin{definition} \label{def:lc}
(layer cushion). Our definition of \textit{layer cushion} for each layer $k$ is similar to~\cite{arora2018stronger}. The layer cushion $\lc^{(k)}$ of layer $k$ is defined to be the largest value $\lc^{(k)}$ such that for any $\myvector{\originput} \in S$, $\lc^{(k)} \fronorm{ \mymatrix{A^{(k)}} } \norm{ \myvector{x^{(k)}} } \leq \norm{ \myvector{x^{(k+1)}} }$.
\end{definition} 
The layer cushion defined in~\cite{arora2018stronger} is sligntly larger than ours since our RHS is $\norm{ \myvector{x^{(k+1)}} } = \relu{\mymatrix{A^{(k)}}\myvector{x^{(k)}}}$ while the RHS of the inequality in the definition of layer cushion in~\cite{arora2018stronger} is $\mymatrix{A^{(k)}}\myvector{x^{(k)}}$. The layer cushion under our settings also considers how much smaller the output $\norm{ \myvector{x^{(k+1)}} }$ is compared to is compared to the upper bound $\fronorm{ \mymatrix{A^{(k)}} } \norm{ \myvector{x^{(k)}} }$.\\ 

\begin{definition} \label{def:t_noise_bound}
 (tensor noise bound). The \textit{tensor noise bound} $\{\nb^{(k)}\}_{j=1}^{R^{(k)}}$ of the the $k^{\tha}$ layer measures the amplitudes of the remaining components after pruning out the ones with amplitudes smaller than the $j^{\tha}$ component: 
 \begin{equation}\label{eq:nb}
 \nb^{(k)}_j := \sum_{r=j+1}^{R^{(k)}} |\lambda_{r}^{(k)}|
 \end{equation}

\end{definition}
 The tensor noise bound measures the amplitudes of the CP components that are pruned out by the compression algorithm, and the smaller it is, the more low-rank the weight matrix is. We will see that a network equipped with \ourNN will be much more low-rank than standard networks.

\label{sec:notations_fc}

\subsection{Generalization Guarantee of Fully Connected Neural Networks}
We have introduced the compression mechanism in Algorithm~\ref{alg:tfc_compress}. For a fully connected network with \ourNN $\tnn$ that is characterized by the properties such as reshaping factor, tensorization factor, layer cushion and tensor noise bound, in section~\ref{sec:properties-tnn}, we derive the generalization error bound of a compression network with any chosen ranks $\{\hat{R}^{(k)}\}_{k=1}^{n}$ as follows. 

\begin{theorem}
\label{thm:main_thm}
For any fully connected network $\orignn$ of $n$ layers satisfying the Assumptions~\ref{assum:CP-tnn}, Algorithm~\ref{alg:tfc_compress} generates a compressed network $\hat{\tnn}$ such that with high probability over the training set , the expected error $L_0(\hat{\tnn})$ is bounded by
\begin{equation}
L_0(\hat{\tnn}) \le \hat{L}_{\gamma}(\orignn) + \tilde{O}\Big(\sqrt{\frac{ \sum_{k=1}^{n} \hat{R}^{(k)}(2 s^{(k)} + 2 s^{(k+1)} + 1)}{m}}\Big)
\end{equation}
for any margin $\gamma \geq 0$, 
and the rank of the $k^{\tha}$ layer, $\hat{R}^{(k)}$,  satisfies that
\begin{equation}
\label{eq:rank-choice}
\hat{R}^{(k)} = \min \Big\{ j \in [R^{(k)}] \ \Big |\  n\rf^{(k)} \nb^{(k)}_j \Pi_{i=k+1}^{n} \tf^{(i)} \nonumber \leq \frac{\gamma}{2 \max_{\myvector{x}  \in S } \fronorm{\orignn(\myvector{\originput})}} \fronorm{\mymatrix{A}^{(k)}} \Pi_{i=k}^n \lc^{(i)} \Big\}
\end{equation}
and $\rf^{(k)}, \tf^{(k)}, \lc^{(k)}$ are reshaping factor, tensorization factor, layer cushion and tensor noise bound of the $k^{\tha}$ layer in Definitions~\ref{def:rf}, ~\ref{def:tf}, and ~\ref{def:lc} respectively. $\nb^{(k)}_j$ is defined in the same way with $\nb^{(k)}$, where $\hat{R}^{(k)}$ is replaced by $j$.
\end{theorem}

The generalization error of the compressed network $L_0(\hat{\tnn})$ depends on the compressibility of the $\tnn$.  The compressibility of the $\tnn$ determines the rank that the compression mechanism should select according to Theorem~\ref{thm:main_thm}, which depends on reshaping factor $\rf^{(k)}$, tensorization factor $\tf^{(k)}$, layer cushion $\lc^{(k)}$ and tensor noise bound  $\nb^{(k)}_j$.

\paragraph{Proof sketch of Theorem~\ref{thm:main_thm}:} To prove this theorem, we introduce the following Lemma~\ref{lem:main_lemma}, which reveals that the difference between the output of the original fully connected network  $\orignn$ and that of the compressed $\tnn$ is bounded by $\epsilon \fronorm{\orignn(\myvector{\originput})}$. Then we show the covering number of the compressed network $\tnn$ by approximating each parameter with some certain accuracy is $\tilde{O}(d)$ w.r.t to a given $\epsilon$. After bounding the covering number, the rest of the proof follows from conventional learning theory.

\begin{lemma}
\label{lem:main_lemma}
For any fully connected network $\orignn$ of $n$ layers satisfying Assumption~\ref{assum:CP-tnn} , Algorithm~\ref{alg:tfc_compress} generates a compressed $\ourTFC$ $\hat{\tnn}$ where for any $\myvector{\originput} \in S$ and any error $0 \leq \epsilon \leq 1$:
\begin{equation}
\fronorm{\orignn(\myvector{\originput}) -  \hat{\tnn}(\mymatrix{\tinput})} \leq \epsilon \fronorm{\orignn(\myvector{\originput})}
\end{equation}
The compressed $\ourTFC$ $\hat{\tnn}$ consists of $\sum_{k=1}^{n} \hat{R}^{(k)}[2(s^{(k)} + s^{(k+1)}) + 1)]$ number of  parameters, where each $\hat{R}^{(k)}$ is defined as what is stated in Algorithm~\ref{alg:find_best_rank}: for each layer $k \in [n]$,   
\begin{align*}
\hat{R}^{(k)} &= \min \Big\{ j \in [R^{(k)}] \ \Big |\  \rf^{(k)} \nb^{(k)}_j \Pi_{i=k+1}^{n}  \tf^{(i)} \leq \frac{\epsilon}{n} \fronorm{\mymatrix{A}^{(k)}} \Pi_{i=k}^n \lc^{(i)} \Big\}
\end{align*}
\end{lemma}
The complete proofs are in \ref{app:proofs}

\subsection{Complete Proofs of Fully Connected Neural Networks}
\label{app:proofs}

To prove Lemma~\ref{lem:main_lemma}, Lemma~\ref{lem:helper_lemma} (introduced below) is needed.
\begin{lemma}
\label{lem:helper_lemma}
For any fully connected network $\orignn$ of $n$ layers satisfying the assumptions in section~\ref{sec:setup}, given a list of ranks $\{\hat{R}^{(k)}\}_{i=1}^n (\forall k, \hat{R}^{(k)} \leq R^{(k)}$), after tensorizing each layer in $\orignn$ and making $\orignn$ into $\tnn$, Algorithm~\ref{alg:tfc_project} generates a compressed tensorial neural network $\hat{\tnn}$ with $\sum_{k=1}^{n} r^{(k)}[2(s^{(k)} + s^{(k+1)}) + 1)]$ total parameters where for any $\myvector{\originput} \in S$:
$$\fronorm{\orignn(\myvector{\originput}) -  \hat{\tnn}(\mymatrix{\tinput})} \leq (\sum_{k=1}^n \frac{\rf^{(k)} \nb^{(k)} }{ \lc^{(k)} \fronorm{\mymatrix{A}^{(k)} }} \cdot \Pi_{i=k+1}^{n} \frac{\tf^{(i)} }{ \lc^{(i)}} ) \fronorm{\orignn(\myvector{\originput})}$$
where $\mymatrix{\tinput}$ is the matricized version of $\myvector{\originput}$, and $\rf^{(k)}, \tf^{(k)}, \lc^{(k)}, \nb^{(k)}$ are reshaping factor, tensorization factor, layer cushion, and tensor noise bound of the $k^{\tha}$ layer in Definitions~\ref{def:rf}, ~\ref{def:tf}, ~\ref{def:lc}, and~\ref{def:t_noise_bound} respectively.
\end{lemma}
\begin{proof} (of Lemma~\ref{lem:helper_lemma}) Based on Algorithm~\ref{alg:tfc_compress}, since  for each layer $k$ in the compressed network $\hat{\tnn}$, representing $\{\hat{\lambda}_{i}^{(k)}, \hat{a}_{i}^{(k)}, \hat{b}_{i}^{(k)}, \hat{c}_{i}^{(k)}, \hat{d}_{i}^{(k)}\}_{i=1}^{\hat{R}^{(k)}}$ only needs $\hat{R}^{(k)}[2(s^{(k)} + s^{(k+1)}) + 1)]$ parameters, the total number of parameters in $\hat{\tnn}$ is $\sum_{k=1}^{n} \hat{R}^{(k)}[2(s^{(k)} + s^{(k+1)}) + 1)]$.

Then as for any $\myvector{\originput} \in S$, $\orignn(\myvector{\originput}) = \tnn(\mymatrix{\tinput})$, and by construction, $\tnn(\mymatrix{\tinput}) = \mymatrix{\tinput}^{(n+1)}$ and $\hat{\tnn}(\mymatrix{\tinput} )= \hat{\mymatrix{\tinput}}^{(n+1)}$, we can prove the lemma by showing $\fronorm{ \mymatrix{\tinput}^{(n+1)} - \hat{\mymatrix{\tinput}}^{(n+1)} }$ satisfies the above inequality, and we will prove this by induction. Notice

\textbf{Induction Hypothesis:} For any layer $m \geq 0$, $\fronorm{\mymatrix{\tinput}^{(m)} - \hat{\mymatrix{\tinput}}^{(m)}} \leq (\sum_{k=1}^{m-1} \frac{ \rf^{(k)} \nb^{(k)} }{ \lc^{(k)} \fronorm{\mymatrix{A}^{(k)}}} \cdot \Pi_{i=k+1}^{m-1} \frac{ \tf^{(i)} }{ \lc^{(i)} } ) \fronorm{ \mymatrix{\tinput}^{(m)} }$

\textbf{Base case:} when $m=1$, the above inequality hold trivially as $\mymatrix{\tinput}^{(1)} = \hat{\mymatrix{\tinput}}^{(1)}$ as we cannot modify the input, and the RHS is always $\geq 0$. 

\textbf{Inductive Step:} Now we assume show that the induction hypothesis is true for all $m$, let us look what happens at layer $m+1$.  As we assume perfect recovery in each layer, $\forall k$, $\{\hat{\lambda}_{i}^{(k)}, \hat{a}_{i}^{(k)}, \hat{b}_{i}^{(k)}, \hat{c}_{i}^{(k)}, \hat{d}_{i}^{(k)}\}_{i=1}^{\hat{R}^{(k)}} = \{\lambda_{i}^{(k)}, a_{i}^{(k)}, b_{i}^{(k)}, c_{i}^{(k)}, d_{i}^{(k)}\}_{i=1}^{\hat{R}^{(k)}}$. 

Let $\mytensor{\phi}^{(k)} := \sum_{i = \hat{R}^{(k)} + 1}^{R^{(k)}} \lambda_{i}^{(k)} a_{i}^{(k)} \otimes b_{i}^{(k)} \otimes c_{i}^{(k)} \otimes d_{i}^{(k)}$, and note that $\nnweight^{(k)} = \hat{\nnweight}^{(k)} + \mytensor{\phi}$.

Then we have
\begin{equation*}
\begin{split}
\MoveEqLeft \fronorm{ \mymatrix{\tinput}^{(m+1)} - \hat{\mymatrix{\tinput}}^{(m+1)}} \\
&= \fronorm{ \relu{\tOut^{(m)}} - \relu{\hat{\tOut}^{(m)}} } \\
&\leq \fronorm{ \sum_{i = 1}^{\hat{R}^{(m)}} \lambda_{i}^{(m)} (a_{i}^{(m)})^\top \mymatrix{\tinput}^{(m)} b_{i}^{(m)} c_{i}^{(m)} \otimes d_{i}^{(m)} + \mytensor{\phi}^{(m)}(\mymatrix{\tinput}^{(m)}) 
 - \sum_{i = 1}^{\hat{R}^{(m)}} \hat{\lambda}_{i}^{(m)} (\hat{a}_{i}^{(m)})^\top \hat{\mymatrix{\tinput}}^{(m)} \hat{b}_{i}^{(m)} \hat{c}_{i}^{(m)} \otimes \hat{d}_{i}^{(m)} }\\
&= \fronorm{ \sum_{i = 1}^{\hat{R}^{(m)}} \lambda_{i}^{(m)} (a_{i}^{(m)})^\top (\mymatrix{\tinput}^{(m)} - \hat{\mymatrix{\tinput}}^{(m)}) b_{i}^{(m)} c_{i}^{(m)} \otimes d_{i}^{(m)} 
+ \mytensor{\phi}^{(m)}(\mymatrix{\tinput}^{(m)}) }
\end{split}
\end{equation*}
So
\begin{equation*}
\begin{split}
\MoveEqLeft \fronorm{\mymatrix{\tinput}^{(m+1)} - \hat{\mymatrix{\tinput}}^{(m+1)}} \\
& \leq \fronorm{\sum_{i = 1}^{\hat{R}^{(m)}} \lambda_{i}^{(m)} (a_{i}^{(m)})^\top (\mymatrix{\tinput}^{(m)} - \hat{\mymatrix{\tinput}}^{(m)}) b_{i}^{(m)} c_{i}^{(m)} \otimes d_{i}^{(m)} }
+ \fronorm{ \mytensor{\phi}^{(m)}(\mymatrix{\tinput}^{(m)})}
\end{split}
\end{equation*}

As  $\mytensor{\phi}^{(m)}(\mymatrix{\tinput}^{(m)}) = \sum_{i = \hat{R}^{(k)} + 1}^{R^{(k)}} \lambda_{i}^{(k)} (a_{i}^{(k)})^\top \mymatrix{\tinput}^{(m)} b_{i}^{(k)} c_{i}^{(k)} \otimes d_{i}^{(k)}.$ Since $\{c_i^{m}\}_i$ and are  $\{d_i^{m}\}_i$ are sets of orthogonal vectors with unit norms,  
\begin{equation*}
\begin{split}
\fronorm{ \mytensor{\phi}^{(m)}(\mymatrix{\tinput}^{(m)})} &= \sqrt{\sum_{i = \hat{R}^{(k)} + 1}^{R^{(k)}} [\lambda_{i}^{(k)} (a_{i}^{(k)})^\top \mymatrix{\tinput}^{(m)} b_{i}^{(k)}]^2} \\
&\leq \sqrt{\sum_{i = \hat{R}^{(k)} + 1}^{R^{(k)}} (\lambda_{i}^{(k)})^2 \norm{a_{i}^{(k)}}^2 \norm{\mymatrix{\tinput}^{(m)} b_{i}^{(k)}}^2 } \\
&\leq \sqrt{\sum_{i = \hat{R}^{(k)} + 1}^{R^{(k)}} (\lambda_{i}^{(k)})^2  \norm{\mymatrix{\tinput}^{(m)}}^2 \norm{b_{i}^{(k)}}^2 } \\
&= \sqrt{\sum_{i = \hat{R}^{(k)} + 1}^{R^{(k)}} (\lambda_{i}^{(k)})^2} \norm{\mymatrix{\tinput}^{(m)}} \\
&= \nb^{(m)} \norm{\mymatrix{\tinput}^{(m)}} \\
&\leq	 \nb^{(m)} \rf^{(m)} \fronorm{\mymatrix{\tinput}^{(m)}} \\
&\leq	 \frac{ \rf^{(m)} \nb^{(m)}  \fronorm{\mymatrix{\tinput}^{(m+1)}} }{\lc^{(m)} \fronorm{\mymatrix{A}^{(m)} }}
\end{split}
\end{equation*}

Similarly,  we can bound $\fronorm{\sum_{i = 1}^{\hat{R}^{(m)}} \lambda_{i}^{(m)} (a_{i}^{(m)})^\top (\mymatrix{\tinput}^{(m)} - \hat{\mymatrix{\tinput}}^{(m)}) b_{i}^{(m)} c_{i}^{(m)} \otimes d_{i}^{(m)} }$ as follows:
\begin{equation*}
\begin{split}
\MoveEqLeft \fronorm{\sum_{i = 1}^{\hat{R}^{(m)}} \lambda_{i}^{(m)} (a_{i}^{(m)})^\top (\mymatrix{\tinput}^{(m)} - \hat{\mymatrix{\tinput}}^{(m)}) b_{i}^{(m)} c_{i}^{(m)} \otimes d_{i}^{(m)} }  \\
&= \sqrt{\sum_{i = 1}^{\hat{R}^{(m)}} [\lambda_{i}^{(m)} (a_{i}^{(m)})^\top (\mymatrix{\tinput}^{(m)} - \hat{\mymatrix{\tinput}}^{(m)}) b_{i}^{(m)} ]^2} \\
&\leq \sqrt{\sum_{i = 1}^{\hat{R}^{(m)}} (\lambda_{i}^{(m)})^2}  \norm{ \mymatrix{\tinput}^{(m)} - \hat{\mymatrix{\tinput}}^{(m)} } \\
&\leq \sqrt{\sum_{i = 1}^{\hat{R}^{(m)}} (\lambda_{i}^{(m)})^2}  \fronorm{ \mymatrix{\tinput}^{(m)} - \hat{\mymatrix{\tinput}}^{(m)} } \\
&=  \sqrt{(\tf^{(m)})^2 \fronorm{\mymatrix{A}^{(m)}}^2} \fronorm{ \mymatrix{\tinput}^{(m)} - \hat{\mymatrix{\tinput}}^{(m)}} \\
&\leq  \tf^{(m)}  \fronorm{\mymatrix{A}^{(m)}} 
\cdot (\sum_{k=1}^{m-1} \frac{ \rf^{(k)} \nb^{(k)}  }{ \lc^{(k)} \fronorm{\mymatrix{A}^{(k)} }} \cdot \Pi_{i=k+1}^{m-1} \frac{ \tf^{(i)} }{ \lc^{(i)}} ) \fronorm{ \mymatrix{\tinput}^{(m)} }\\
&\leq \rf^{(m)} \tf^{(m)} \fronorm{\mymatrix{A}^{(m)}} \frac{ \fronorm{ \mymatrix{\tinput}^{(m+1)} }}{ \lc^{(m)} \fronorm{\mymatrix{A^{(m)}}} } \times 
(\sum_{k=1}^{m-1} \frac{ \rf^{(k)} \nb^{(k)} }{ \lc^{(k)} \fronorm{\mymatrix{A}^{(k)} }} \cdot \Pi_{i=k+1}^{m-1} \frac{ \tf^{(i)} }{ \lc^{(i)} } )  \\
&= (\sum_{k=1}^{m-1} \frac{\rf^{(k)} \nb^{(k)} }{ \lc^{(k)} \fronorm{\mymatrix{A}^{(k)} }} \cdot \Pi_{i=k+1}^{m} \frac{\tf^{(i)} }{ \lc^{(i)}  } ) \cdot \fronorm{ \mymatrix{\tinput}^{(m+1)} }
\end{split}
\end{equation*}

Combining the above two terms together, we have 
\begin{equation*}
\begin{split}
\MoveEqLeft \fronorm{\mymatrix{\tinput}^{(m+1)} - \hat{\mymatrix{\tinput}}^{(m+1)}} \\
&\leq (\sum_{k=1}^{m-1} \frac{\rf^{(k)}  \nb^{(k)} }{ \lc^{(k)} \fronorm{\mymatrix{A}^{(k)} }} \cdot \Pi_{i=k+1}^{m} \frac{ \tf^{(i)} }{ \lc^{(i)}  } ) \cdot \fronorm{ \mymatrix{\tinput}^{(m+1)} } 
+ \frac{\rf^{(m)} \nb^{(m)} \fronorm{\mymatrix{\tinput}^{(m+1)}} }{\lc^{(m)} \fronorm{\mymatrix{A}^{(m)} }} \\
&=  (\sum_{k=1}^{m-1} \frac{\rf^{(k)} \nb^{(k)} }{ \lc^{(k)} \fronorm{\mymatrix{A}^{(k)} }} \cdot \Pi_{i=k+1}^{m} \frac{ \tf^{(i)} }{ \lc^{(i)}  } 
+  \frac{\rf^{(m)} \nb^{(m)}}{\lc^{(m)} \fronorm{\mymatrix{A}^{(m)}}} ) \cdot \fronorm{ \mymatrix{\tinput}^{(m+1)} } \\
&=  (\sum_{k=1}^{m-1} \frac{ \rf^{(k)} \nb^{(k)} }{ \lc^{(k)} \fronorm{\mymatrix{A}^{(k)} }} \cdot \Pi_{i=k+1}^{m} \frac{\tf^{(i)} }{ \lc^{(i)}  } 
+  \frac{\rf^{(m)}\nb^{(m)}}{\lc^{(m)} \fronorm{\mymatrix{A}^{(m)}}}   \cdot \Pi_{i=m+1}^{m} \frac{\tf^{(i)} }{ \lc^{(i)}} ) \cdot \fronorm{ \mymatrix{\tinput}^{(m+1)} } \\
&= (\sum_{k=1}^{m} \frac{\rf^{(k)} \nb^{(k)} }{ \lc^{(k)} \fronorm{\mymatrix{A}^{(k)} }} \cdot \Pi_{i=k+1}^{m} \frac{\tf^{(i)} }{ \lc^{(i)} }) \cdot \fronorm{ \mymatrix{\tinput}^{(m+1)} } 
\end{split}
\end{equation*}
Where the second to the last equality is due to the fact that for any $\alpha_i, \beta \in \Rbb$, $(\Pi_{i=k+1}^{k} \alpha_i) \times \beta = \beta$.
\end{proof}
Then we can proceed to prove Lemma~\ref{lem:main_lemma}: 
\begin{proof} (of Lemma~\ref{lem:main_lemma}) 
Based on the assumptions of the components from CP decomposition for each $\nnweight^{(k)}$ in section \ref{sec:setup}, the $\{\hat{R}^{(k)} \}_{k=1}^{n}$ returned by Algorithm~\ref{alg:find_best_rank} will satisfy: 
\begin{itemize}
\item $\forall k, \ \hat{R}^{(k)} \leq R^{(k)}$
\item $\rf^{(k)} \nb^{(k)} \Pi_{i=k+1}^{n} \tf^{(i)} \leq \frac{\epsilon}{n}\fronorm{\mymatrix{A}^{(k)}} \Pi_{i=k}^n \lc^{(i)}$
\end{itemize}
Thus, 
$$\frac{ \rf^{(k)} \nb^{(k)} }{ \lc^{(k)} \fronorm{\mymatrix{A}^{(k)} }} \cdot \Pi_{i=k+1}^{n} \frac{ \tf^{(i)} }{ \lc^{(i)}} \leq \frac{\epsilon}{n}$$
 Then by lemma~\ref{lem:helper_lemma}, $$\fronorm{\orignn(\myvector{\originput}) -  \hat{\tnn}(\mymatrix{\tinput})} \leq \epsilon \fronorm{\orignn(\myvector{\originput})}$$
\end{proof}
Before proving Theorem~\ref{thm:main_thm}, Lemma~\ref{lem:helper_lemma_2} (introduced below) is needed. 
\begin{lemma}
\label{lem:helper_lemma_2}
For any fully connected network $\orignn$ of $n$ layers satisfying the assumptions in section~\ref{sec:setup} and any margin $\gamma \geq 0$, $\orignn$ can be compressed to a fully-connected tensorial neural network $\hat{\tnn}$ with $\sum_{k=1}^{n} \hat{R}^{(k)}[2(s^{(k)} + s^{(k+1)}) + 1)]$ total parameters such that for any $\myvector{\originput} \in S$, $\hat{L_0}(\hat{\tnn}) \leq \hat{L}_\gamma(\orignn)$.
Here, for each layer $k$, 
\begin{align*}
\hat{R}^{(k)} &= \min \Big\{ j \in [R^{(k)}] \ \Big |\  \rf^{(k)} \nb^{(k)}_j \Pi_{i=k+1}^{n} \tf^{(i)}\leq \frac{\epsilon}{n} \fronorm{\mymatrix{A}^{(k)}} \Pi_{i=k}^n \lc^{(i)} \Big\}
\end{align*}
\end{lemma}
\begin{proof} (of Lemma~\ref{lem:helper_lemma_2}) 
If $\gamma \geq 2 \max_{\myvector{x}  \in S } \fronorm{\orignn(\myvector{\originput})}$, for any pair $(\originput, y) \in S$, we have 
\begin{align*}
 |\orignn(\myvector{\originput})[y] - \max_{j \neq y} \orignn(\myvector{\originput})[j]|^2
 &\leq (| \orignn(\myvector{\originput})[y]| + |\max_{j \neq y} \orignn(\myvector{\originput})[j]|)^2 \\
 &\leq 4 \max_{\myvector{x}  \in S } \fronorm{\orignn(\myvector{\originput})}^2 \\
 &\leq \gamma^2
\end{align*}
Then the output margin of $\orignn$ cannot be greater than $\gamma$ for any $\myvector{\originput} \in S$. Thus $\hat{L}_\gamma(\orignn) = 1$. 

If $\gamma < 2 \max_{\myvector{x}  \in S } \fronorm{\orignn(\myvector{\originput})}$, setting $$\epsilon = \frac{\gamma}{2 \max_{\myvector{x}  \in S } \fronorm{\orignn(\myvector{\originput})}}$$ in Lemma~\ref{lem:main_lemma}, we obtain a compressed fully-connected tensorial neural network $\hat{\tnn}$ with the desired number of parameters and $$\fronorm{\orignn(\myvector{\originput}) -  \hat{\tnn}(\mymatrix{\tinput})} < \frac{\gamma}{2} \Rightarrow \forall j, | \orignn(\myvector{\originput})[j] - \hat{\tnn}(\mymatrix{\tinput})[j] | < \frac{\gamma}{2}$$
Then for any pair $(\originput, y) \in S$, if $\orignn(\myvector{\originput})[y] > \gamma + \max_{j \neq y} \orignn(\myvector{\originput})[j]$, $\hat{\tnn}$ classifies $\myvector{\originput}$ correctly as well because:
$$\hat{\tnn}(\mymatrix{\tinput})[y] > \orignn(\myvector{\originput})[y] - \frac{\gamma}{2} >  \max_{j \neq y} \orignn(\myvector{\originput})[j] + \frac{\gamma}{2} > \max_{j\neq y}\hat{\tnn}(\mymatrix{\tinput})[j]$$
Thus, $\hat{L_0}(\hat{\tnn}) \leq \hat{L}_\gamma(\orignn)$.
\end{proof}
Now we prove the main theorem~\ref{thm:main_thm} by bounding the covering number given any $\epsilon$.
\begin{proof} (of Theorem~\ref{thm:main_thm}) 
To be more specific, let us bound the covering number of the compressed network $\hat{\tnn}$ by approximating each parameter with accuracy $\acc$. 

\textbf{Covering Number Analysis for Fully Connected Neural Network} Let $\tilde{T}$ denote the network after approximating each parameter in $\hat{\tnn}$ with accuracy $\acc$ (and $\mytensor{\tilde{T}}^{(k)}$ denote its weight tensor on the $k^{th}$ layer). Based on the given accuracy, we know that 
$\forall k, \ $
$\abs{\hat{\lambda}_{i}^{(k)} - \tilde{\lambda}_{i}^{(k)}} \leq \acc$ and 
$\norm{\hat{\myvector{a}}_{i}^{(k)} - \tilde{\myvector{a}}_{i}^{(k)}} \leq \sqrt{s^{(k)}}\acc$
 (similar inequalities also hold for $\hat{\myvector{b}}_{i}^{(k)}, \hat{\myvector{c}}_{i}^{(k)}, \hat{\myvector{d}}_{i}^{(k)}$). For simplicity, in this proof, let us just use $\myvector{a}_{i}^{(k)}, \myvector{b}_{i}^{(k)}, \myvector{c}_{i}^{(k)}, \myvector{d}_{i}^{(k)}$ to denote $\hat{\myvector{a}}_{i}^{(k)}, \hat{\myvector{b}}_{i}^{(k)}, \hat{\myvector{c}}_{i}^{(k)}, \hat{\myvector{d}}_{i}^{(k)}$

Notice that
$$\mymatrix{Y}^{(k)} =  \sum_{i = 1}^{r^{(k)}} \lambda_{i}^{(k)} (\myvector{a}_{i}^{(k)})^\top \mymatrix{\tinput}^{(k)} \myvector{b}_{i}^{(k)} \myvector{c}_{i}^{(k)} \otimes \myvector{d}_{i}^{(k)}$$ 
$$\tilde{\mymatrix{Y}}^{(k)} = \sum_{i = 1}^{r^{(k)}} \tilde{\lambda}_{i}^{(k)} (\tilde{\myvector{a}}_{i}^{(k)})^\top \tilde{\mymatrix{\tinput}}^{(k)} \tilde{\myvector{b}}_{i}^{(k)} \tilde{\myvector{c}}_{i}^{(k)} \otimes \tilde{\myvector{d}}_{i}^{(k)}$$ 

Let $\epsilon^{(k)} = \fronorm{\tilde{\mymatrix{Y}}^{(k)} - \mymatrix{Y}^{(k)}}$. Then for each $k$, let us first bound 
$\abs{(\myvector{a}_{i}^{(k)})^\top \mymatrix{\tinput}^{(k)} \myvector{b}_{i}^{(k)} - (\tilde{\myvector{a}}_{i}^{(k)})^\top \tilde{\mymatrix{\tinput}}^{(k)} \tilde{\myvector{b}}_{i}^{(k)}}$ and
$\fronorm{\myvector{c}_{i}^{(k)} \otimes \myvector{d}_{i}^{(k)} - \tilde{\myvector{c}}_{i}^{(k)} \otimes \tilde{\myvector{d}}_{i}^{(k)}}$ separately. 

\textbf{Bound $\abs{(\myvector{a}_{i}^{(k)})^\top \mymatrix{\tinput}^{(k)} \myvector{b}_{i}^{(k)} - (\tilde{\myvector{a}}_{i}^{(k)})^\top \tilde{\mymatrix{\tinput}}^{(k)} \tilde{\myvector{b}}_{i}^{(k)}}$:} When $k =1$, we know that $\mymatrix{\tinput}^{(1)} = \tilde{\mymatrix{\tinput}}^{(1)}$. Let us first consider the base case where $k =1$. For simplicity, let $\myvector{a} = \myvector{a}_{i}^{(1)}$, $\tilde{\myvector{a}} = \tilde{\myvector{a}}_{i}^{(1)}$, $\myvector{b} = \myvector{b}_{i}^{(1)}$, $\tilde{\myvector{b}} = \tilde{\myvector{b}}_{i}^{(1)}$, and $\mymatrix{X} = \mymatrix{\tinput}^{(1)}$. 
Then
\begin{align*} 
\begin{split}
\MoveEqLeft \abs{(\myvector{a}_{i}^{(1)})^\top \mymatrix{\tinput}^{(1)} \myvector{b}_{i}^{(1)} - (\tilde{\myvector{a}}_{i}^{(1)})^\top \tilde{\mymatrix{\tinput}}^{(1)} \tilde{\myvector{b}}_{i}^{(1)}} \\
&= \abs{\myvector{a}^\top \mymatrix{X} \myvector{b} - \tilde{\myvector{a}}^\top \mymatrix{X} \tilde{\myvector{b}}} \\
&= \abs{\myvector{a}^\top \mymatrix{X} \myvector{b} - \myvector{a}^\top \mymatrix{X} \tilde{\myvector{b}} + \myvector{a}^\top \mymatrix{X} \tilde{\myvector{b}} - \tilde{\myvector{a}}^\top \mymatrix{X} \tilde{\myvector{b}}} \\
&= \abs{\myvector{a}^\top \mymatrix{X}(\myvector{b} - \tilde{\myvector{b}})  + (\myvector{a} - \tilde{\myvector{a}})^\top \mymatrix{X} \tilde{\myvector{b}}} \\
& \leq  \abs{\myvector{a}^\top \mymatrix{X}(\myvector{b} - \tilde{\myvector{b}})} + \abs{(\myvector{a} - \tilde{\myvector{a}})^\top \mymatrix{X} \tilde{\myvector{b}}} \\
& \leq  \norm{\mymatrix{X}^\top \myvector{a} } \norm{\myvector{b} - \tilde{\myvector{b}}} + \norm{ \myvector{a} - \tilde{\myvector{a}}} \norm{\mymatrix{X} \tilde{\myvector{b}}} \\
& \leq \acc  \sqrt{s^{(1)}} \norm{\mymatrix{X}}( \norm{\myvector{a}} + \norm{\tilde{\myvector{b}}})\\
& \leq 2 \acc  \sqrt{s^{(1)}} \norm{\mymatrix{X}} 
\end{split}
\end{align*}
The second to the last inequality is because singular values are invariant to matrix transpose.

When $k \geq 1$, similarly, let $\myvector{a} = \myvector{a}_{i}^{(k)}$, $\tilde{\myvector{a}} = \tilde{\myvector{a}}_{i}^{(k)}$ (define $b$ in a similar way),  $\mymatrix{X} = \mymatrix{\tinput}^{(k)}$, and $\tilde{\mymatrix{X}} = \tilde{\mymatrix{\tinput}}^{(k)}$. Let $\mymatrix{Y} = \mymatrix{Y}^{(k-1)}$, and $\tilde{\mymatrix{Y}} = \tilde{\mymatrix{Y}}^{(k-1)}$ (basically the output from the $(k-1)^{th}$ layer before activation). Then
\begin{align*} 
\begin{split}
\MoveEqLeft \abs{(\myvector{a}_{i}^{(k)})^\top \mymatrix{\tinput}^{(k)} \myvector{b}_{i}^{(k)} - (\tilde{\myvector{a}}_{i}^{(k)})^\top \tilde{\mymatrix{\tinput}}^{(k)} \tilde{\myvector{b}}_{i}^{(k)}} \\
&= \abs{\myvector{a}^\top \mymatrix{X} \myvector{b} - \tilde{\myvector{a}}^\top \tilde{\mymatrix{X}} \tilde{\myvector{b}}} \\
&= \abs{\myvector{a}^\top \mymatrix{X} \myvector{b} - \tilde{\myvector{a}}^\top \mymatrix{X} \tilde{\myvector{b}} + \tilde{\myvector{a}}^\top \mymatrix{X} \tilde{\myvector{b}} - \tilde{\myvector{a}}^\top \tilde{\mymatrix{X}} \tilde{\myvector{b}}} \\
&\leq \abs{\myvector{a}^\top \mymatrix{X} \myvector{b} - \tilde{\myvector{a}}^\top \mymatrix{X} \tilde{\myvector{b}}} + \abs{\tilde{\myvector{a}}^\top \mymatrix{X} \tilde{\myvector{b}} - \tilde{\myvector{a}}^\top \tilde{\mymatrix{X}} \tilde{\myvector{b}}} \\
&\leq 2 \acc  \sqrt{s^{(k)}} \norm{\mymatrix{X}} +  \norm{\mymatrix{X} - \tilde{\mymatrix{X}}}, \text{by base case $k = 1$} \\
&= 2 \acc \sqrt{s^{(k)}} \norm{\mymatrix{X}} +  \norm{\relu{\mymatrix{Y}} - \relu{\tilde{\mymatrix{Y}}}} \\
&\leq 2 \acc \sqrt{s^{(k)}} \norm{\mymatrix{X}} +  \fronorm{\relu{\mymatrix{Y}} - \relu{\tilde{\mymatrix{Y}}}} \\
&\leq 2 \acc \sqrt{s^{(k)}} \norm{\mymatrix{X}} + \fronorm{\mymatrix{Y} - \tilde{\mymatrix{Y}}} \\
&= 2 \acc \sqrt{s^{(k)}} \norm{\mymatrix{X}} + \epsilon^{(k-1)}
\end{split}
\end{align*}

Then we can also bound  $\abs{ \lambda_{i}^{(k)} (\myvector{a}_{i}^{(k)})^\top \mymatrix{\tinput}^{(k)} \myvector{b}_{i}^{(k)} - \tilde{\lambda}_{i}^{(k)} (\tilde{\myvector{a}}_{i}^{(k)})^\top \tilde{\mymatrix{\tinput}}^{(k)} \tilde{\myvector{b}}_{i}^{(k)}}$. For simplicity, let $\lambda =  \lambda_{i}^{(k)}$, 
$\tilde{\lambda} =  \tilde{\lambda}_{i}^{(k)}$, $x = (\myvector{a}_{i}^{(k)})^\top \mymatrix{\tinput}^{(k)} \myvector{b}_{i}^{(k)}$, and $\tilde{x} = (\tilde{\myvector{a}}_{i}^{(k)})^\top \tilde{\mymatrix{\tinput}}^{(k)} \tilde{\myvector{b}}_{i}^{(k)}$. Then
\begin{align*} 
\begin{split}
\MoveEqLeft \abs{ \lambda_{i}^{(k)} (\myvector{a}_{i}^{(k)})^\top \mymatrix{\tinput}^{(k)} \myvector{b}_{i}^{(k)} - \tilde{\lambda}_{i}^{(k)} (\tilde{\myvector{a}}_{i}^{(k)})^\top \tilde{\mymatrix{\tinput}}^{(k)} \tilde{\myvector{b}}_{i}^{(k)}} \\
&= \abs{\lambda x - \hat{\lambda} \hat{x}} \\
&\leq \abs{\lambda - \hat{\lambda}} \abs{x} + \abs{\hat{\lambda}} \abs{x - \hat{x}} \\
&\leq \abs{\lambda - \hat{\lambda}} \abs{x} + \abs{\lambda} \abs{x - \hat{x}} \text{, we can pick $\abs{\hat{\lambda}} \leq \abs{\lambda}$} \\
&\leq \acc \abs{x} +  \abs{\lambda}  \times (2 \acc \sqrt{s^{(k)}} \norm{\mymatrix{X}^{(k)}} + \epsilon^{(k-1)}) \\
&\leq \acc  \norm{\mymatrix{X}^{(k)}} + 2 \acc \norm{\mymatrix{X}^{(k)}} \abs{\lambda} \sqrt{s^{(k)}}  + \abs{\lambda} \epsilon^{(k-1)} \\
&= \acc  \norm{\mymatrix{X}^{(k)}} (1 + 2\abs{\lambda} \sqrt{s^{(k)}}) + \abs{\lambda} \epsilon^{(k-1)}
\end{split}
\end{align*}

\textbf{Bound $\fronorm{\myvector{c}_{i}^{(k)} \otimes \myvector{d}_{i}^{(k)} - \tilde{\myvector{c}}_{i}^{(k)} \otimes \tilde{\myvector{d}}_{i}^{(k)}}$:} Similarly let $\myvector{c} = \myvector{c}_{i}^{(k)}$ and $\tilde{\myvector{c}} = \tilde{\myvector{c}}_{i}^{(k)}$ ( define $\myvector{d}$ and $\tilde{\myvector{d}}$ in a similar way). Then
\begin{align*} 
\begin{split}
\MoveEqLeft \fronorm{\myvector{c}_{i}^{(k)} \otimes \myvector{d}_{i}^{(k)} - \tilde{\myvector{c}}_{i}^{(k)} \otimes \tilde{\myvector{d}}_{i}^{(k)}}^2  \\
&= \fronorm{\myvector{c}  \myvector{d}^\top - \tilde{\myvector{c}} \tilde{\myvector{d}}^\top}^2 \\
& = \Tr( (\myvector{c}  \myvector{d}^\top - \tilde{\myvector{c}} \tilde{\myvector{d}}^\top)^\top (\myvector{c}  \myvector{d}^\top - \tilde{\myvector{c}} \tilde{\myvector{d}}^\top)) \\
 & = \Tr( ( \myvector{d} \myvector{c}^\top - \tilde{\myvector{d}} \tilde{\myvector{c}}^\top) (\myvector{c}  \myvector{d}^\top - \tilde{\myvector{c}} \tilde{\myvector{d}}^\top))  \\ 
 &= \Tr( \myvector{d} \myvector{c}^\top \myvector{c}  \myvector{d}^\top - \tilde{\myvector{d}} \tilde{\myvector{c}}^\top \myvector{c}  \myvector{d}^\top - \myvector{d} \myvector{c}^\top \tilde{\myvector{c}} \tilde{\myvector{d}}^\top + \tilde{\myvector{d}} \tilde{\myvector{c}}^\top \tilde{\myvector{c}} \tilde{\myvector{d}}^\top) \\
 &= \Tr( \myvector{d} \myvector{c}^\top \myvector{c}  \myvector{d}^\top) - \Tr(\tilde{\myvector{d}} \tilde{\myvector{c}}^\top \myvector{c}  \myvector{d}^\top) - \Tr( \myvector{d} \myvector{c}^\top \tilde{\myvector{c}} \tilde{\myvector{d}}^\top) + \Tr(\tilde{\myvector{d}} \tilde{\myvector{c}}^\top \tilde{\myvector{c}} \tilde{\myvector{d}}^\top) \\
 &= \Tr(\myvector{c}^\top \myvector{c}  \myvector{d}^\top \myvector{d}) - \Tr(\myvector{c}^\top \tilde{\myvector{c}} \tilde{\myvector{d}}^\top \myvector{d}) +  \Tr(\tilde{\myvector{c}}^\top \tilde{\myvector{c}} \tilde{\myvector{d}}^\top \tilde{\myvector{d}}) -  \Tr(\tilde{\myvector{c}}^\top \myvector{c}  \myvector{d}^\top \tilde{\myvector{d}})\\
 &=\Tr(\myvector{c}^\top (\myvector{c} \myvector{d}^\top - \tilde{\myvector{c}}\tilde{\myvector{d}}^\top) \myvector{d} + \tilde{\myvector{c}}^\top(\tilde{\myvector{c}}\tilde{\myvector{d}}^\top - \myvector{c} \myvector{d}^\top)\tilde{\myvector{d}}) \\
 &=\myvector{c}^\top (\myvector{c} \myvector{d}^\top - \tilde{\myvector{c}}\tilde{\myvector{d}}^\top)\myvector{d} + \tilde{\myvector{c}}^\top(\tilde{\myvector{c}}\tilde{\myvector{d}}^\top - \myvector{c} \myvector{d}^\top)\tilde{\myvector{\tilde{d}}} \\
 &\leq \norm{\myvector{c}} \norm{\myvector{c} \myvector{d}^\top - \tilde{\myvector{c}}\tilde{\myvector{d}}^\top} \norm{\myvector{d}}  + \norm{\tilde{\myvector{c}}} \norm{\tilde{\myvector{c}} \tilde{\myvector{d}}^\top - \myvector{c} \myvector{d}^\top} \norm{\myvector{d}}\\
 &\leq 2\norm{\myvector{c} \myvector{d}^\top - \tilde{\myvector{c}}\tilde{\myvector{d}}^\top}, \text{as the norms of $\myvector{c}, \myvector{d}, \tilde{\myvector{c}}, \tilde{\myvector{d}}$ are $\leq 1$} \\
 & = 2\norm{\myvector{c} \myvector{d}^\top - \myvector{c}\tilde{\myvector{d}}^\top + \myvector{c}\tilde{\myvector{d}}^\top - \tilde{\myvector{c}}\tilde{\myvector{d}}^\top} \\ 
 & = 2\norm{\myvector{c}( \myvector{d}^\top - \tilde{\myvector{d}}^\top) + (\myvector{c} - \tilde{\myvector{c}})\tilde{\myvector{d}}^\top} \\
 & \leq 2(\norm{\myvector{c}( \myvector{d}^\top - \tilde{\myvector{d}}^\top)} + \norm{(\myvector{c} - \tilde{\myvector{c}})\tilde{\myvector{d}}^\top})  \\ 
&\leq 2(\norm{\myvector{c}} \norm{\myvector{d} - \tilde{\myvector{d}}} + \norm{\myvector{d}} \norm{\myvector{c} - \tilde{\myvector{c}}} ), \text{as they are rank 1 matrices} \\ 
 & \leq 4 \sqrt{s^{(k+1)}} \acc
\end{split}
\end{align*}

\textbf{Bound $\epsilon^{(k)} = \fronorm{\tilde{\mymatrix{Y}}^{(k)} - \mymatrix{Y}^{(k)}}$:} Similarly, for simplicity, let $w_i = \lambda_{i}^{(k)} (\myvector{a}_{i}^{(k)})^\top \mymatrix{\tinput}^{(k)} \myvector{b}_{i}^{(k)}$, $\tilde{w}_i = \tilde{\lambda}_{i}^{(k)} (\tilde{\myvector{a}}_{i}^{(k)})^\top \tilde{\mymatrix{\tinput}}^{(k)} \tilde{\myvector{b}}_{i}^{(k)}$, $\mymatrix{U_i} = \myvector{c}_{i}^{(k)} \otimes \myvector{d}_{i}^{(k)}$, and $\tilde{\mymatrix{U}_i} = \tilde{\myvector{c}}_{i}^{(k)} \otimes \tilde{\myvector{d}}_{i}^{(k)}$.

Since $\fronorm{\tilde{\mymatrix{Y}}^{(k)} - \mymatrix{Y}^{(k)}} = \fronorm{ \sum_{i=1}^{r^{(k)}} w_i \mymatrix{U_i}  - \sum_{i=}^{r^{(k)}} \tilde{d}_i \tilde{\mymatrix{U}_i}}$, 
\begin{equation} 
\label{eq:cov_rr}
\begin{split}
\MoveEqLeft \fronorm{ \sum_{i=1}^{r^{(k)}} w_i \mymatrix{U_i}  - \sum_{i=}^{r^{(k)}} \tilde{w}_i \tilde{\mymatrix{U}_i}} \\
&= \fronorm{ \sum_{i=1}^{r^{(k)}} (w_i \mymatrix{U_i} - \tilde{w}_i \tilde{\mymatrix{U}_i}) } \\
&\leq \sum_{i=1}^{r^{(k)}} \fronorm{w_i \mymatrix{U_i} - \tilde{w}_i \tilde{\mymatrix{U}_i}} \\
&= \sum_{i=1}^{r^{(k)}} \fronorm{w_i \mymatrix{U_i} - w_i \tilde{\mymatrix{U}_i}  + w_i\tilde{\mymatrix{U}_i}  - \tilde{w}_i \tilde{\mymatrix{U}_i}} \\
&\leq \sum_{i=1}^{r^{(k)}} \fronorm{w_i \mymatrix{U_i} - w_i \tilde{\mymatrix{U}_i}}  + \fronorm{w_i \tilde{\mymatrix{U}_i}  - \tilde{w}_i \tilde{\mymatrix{U}_i}} \\
&= \sum_{i=1}^{r^{(k)}} \fronorm{w_i (\mymatrix{U_i} - \tilde{\mymatrix{U}_i})}  + \fronorm{(w_i  - \tilde{w}_i) \tilde{\mymatrix{U}_i}} \\
&= \sum_{i=1}^{r^{(k)}} \abs{w_i}\fronorm{\mymatrix{U_i} - \tilde{\mymatrix{U}_i}}  + \abs{w_i  - \tilde{w}_i} \fronorm{\tilde{\mymatrix{U}_i}} \\
&\leq  \sum_{i=1}^{r^{(k)}}   \abs{w_i} \times \sqrt{4\sqrt{s^{(k+1)}} \acc} 
 + \big( \acc  \norm{\mymatrix{X}^{(k)}} (1 + 2\abs{\lambda_{i}} \sqrt{s^{(k)}}) + \abs{\lambda_{i}} \epsilon^{(k-1)} \big) \times \fronorm{\tilde{\mymatrix{U}_i}} \\
&\leq  \sum_{i=1}^{r^{(k)}}   \abs{ \lambda_{i}^{(k)}} \norm{\mymatrix{\tinput}^{(k)}} \times \sqrt{4\sqrt{s^{(k+1)}} \acc}  
 + \big(\acc  \norm{\mymatrix{X}^{(k)}} (1 + 2\abs{\lambda_{i}} \sqrt{s^{(k)}}) + \abs{\lambda_{i}} \epsilon^{(k-1)}\big) 
   \times \fronorm{ \tilde{\myvector{c}}_{i}^{(k)} \otimes \tilde{\myvector{d}}_{i}^{(k)} }\\
&=  \sum_{i=1}^{r^{(k)}}  2 \abs{ \lambda_{i}^{(k)}} \norm{\mymatrix{\tinput}^{(k)}} \times \sqrt{\sqrt{s^{(k+1)}} \acc} 
+ \acc  \norm{\mymatrix{X}^{(k)}} (1 + 2\abs{\lambda_{i}} \sqrt{s^{(k)}}) +  \abs{\lambda_{i}} \epsilon^{(k-1)} \\
&\leq \sum_{i=1}^{r^{(k)}} \acc \norm{\mymatrix{\tinput}^{(k)}} \big(1+ 2 \abs{\lambda_{i}^{(k)}} (\sqrt{s^{(k)}} + \sqrt{s^{(k+1)}}) \big)
+  \abs{\lambda_{i}^{(k)}} \epsilon^{(k-1)} \text{, assume $\sqrt{\sqrt{s^{(k+1)}} \acc} \leq  \sqrt{s^{(k+1)}} \acc$ } \\
&\leq r^{(k)} \times \{ \acc \norm{\mymatrix{\tinput}^{(k)}} \big( 1+ 2 \abs{\lambda_{max}^{(k)}} (\sqrt{s^{(k)}} + \sqrt{s^{(k+1)}}) \big) 
+  \abs{\lambda_{max}^{(k)}} \epsilon^{(k-1)}\}  \\
&\leq \acc r^{(k)} [1+ 2 \abs{\lambda_{max}^{(k)}} (\sqrt{s^{(k)}} + \sqrt{s^{(k+1)}}) ] \norm{\mymatrix{\tinput}^{(k)}} 
+  r^{(k)} \abs{\lambda_{max}^{(k)}} \epsilon^{(k-1)}
\end{split}
\end{equation}

Let $\alpha^{(k)} := \acc r^{(k)} [1+ 2 \abs{\lambda_{max}^{(k)}} (\sqrt{s^{(k)}} + \sqrt{s^{(k+1)}}) ] \norm{\mymatrix{\tinput}^{(k)}}$, and $\beta^{(k)} = r^{(k)} \abs{\lambda_{max}^{(k)}}$, then by the recurrence relationship in~\ref{eq:cov_rr}, the difference between the final output of the two networks are bounded by:
\begin{align*} 
\begin{split}
\MoveEqLeft \fronorm{\hat{\tnn}(\mymatrix{\tinput}) - \tilde{\tnn}(\mymatrix{\tinput})} \\
&= \fronorm{\relu{\hat{\mymatrix{Y}}^{(n)}} - \relu{\mymatrix{Y}^{(n)}}} \text{\ ($= \mymatrix{\tinput}^{(n+1)} - \mymatrix{\tinput}^{(n+1)}$)}  \\
&\leq \fronorm{\tilde{\mymatrix{Y}}^{(n)}  - \mymatrix{Y}^{(n)} } \\
&\leq \sum_{k=1}^{n} \alpha^{(k)} \Pi_{i = k+1}^{n} \beta^{(i)}
 \end{split}
\end{align*}

Since $\forall k \in [n], \norm{\mymatrix{\tinput}^{(k)}} \leq \Pi_{i = k}^n \frac{\rf^{(i)}}{\lc^{(i)} \fronorm{\mymatrix{A}^{(i)}}} \fronorm{\mymatrix{\tinput}^{(n+1)}}$, to obtain an $\epsilon$-cover of the compressed network, we can first assume $\beta^{(k)} \geq 1 \ \forall k \in [n]$. Then $\acc$ need to satisfy:
$$\acc \leq \frac{\epsilon}{ (r^{(*)} \abs{\lambda^*})^n 
\fronorm{\mymatrix{\tinput^{(n+1)}}} n r^{(*)}  (1+ 4 \abs{\lambda^{(*)}} \sqrt{s^{(*)} }) (\frac{\rf^{(*)}}{\mu^{(*)} \fronorm{\mymatrix{A^{(*)} }} })^n }$$
where $r^{(*)} = \max_{k} r^{(k)} \lambda^{(*)} = \max_{i, k} \lambda_{i}^{(k)}$, $s^{(*)} = \max_{k} s^{(k)}$, and $\frac{\rf^{(*)}}{\mu^{(*)} \fronorm{\mymatrix{A^{(*)} }} } = \max_{k} \frac{\rf^{(k)}}{\mu^{(k)} \fronorm{\mymatrix{A^{(k)} }} }$

As when $\acc$ is fixed, the number of networks in our cover will at most be $(\frac{1}{\acc})^d$ where $d$ denote the number of parameters in the original network. Hence, the covering number w.r.t to a given $\epsilon$ is $\tilde{O}(n d)$ (n is the number of layers in the given neural network). As for practical neural networks, the number of layers $n$ is usually much less than $O(\log(d))$, thus the covering number we obtained w.r.t to a given $\epsilon$ is just $\tilde{O}(d)$ for practical neural networks.
\end{proof}

\section{Neural Networks with Skip Connections}
\label{sec:sc}
\subsection{Problem Setup}
\label{subsec:sc_notations}
For neural nets with skip connections, the current theoretical analyses consider convolutional neural networks with one skip connection used on each layer, since our theoretical results can easily extend to general neural nets with skip connections. Therefore, we used the same the notations for neural nets with skip connections as what we defined for convolutional neural networks. 

\textbf{Forward pass functions}
Under the above assumptions, the only difference that we need to take into account between our analysis of CNN with skip connections and our analysis of standard CNN is the forward pass functions. In neural networks with skip connections, we have
$$\mytensor{\tinput}^{(k)}=\relu{\mytensor{\tOut}^{(k-1)}}$$
$$\mytensor{\tOut}^{(k)}= \cnnweight^{(k)} \left(\mytensor{\tinput}^{(k)}\right) +\mytensor{\tinput}^{(k)}$$
and 
$$\hat{\mytensor{\tinput}}^{(k)}=\relu{\hat{\mytensor{\tOut}}^{(k-1)}}$$
$$\hat{\mytensor{\tOut}}^{(k)}=\hat{\cnnweight}^{(k)}\left(\hat{\mytensor{\tinput}}^{(k)}\right)+\hat{\mytensor{\tinput}}^{(k)}$$
where $\cnnweight^{(k)}\left(\mytensor{\tinput}^{(k)}\right) $ and $\hat{\cnnweight}^{(k)}\left(\hat{\mytensor{\tinput}}^{(k)}\right)$ 
compute the outputs of the $k^{\tha}$ convolutional layer.

Similarly, we use \textit{tensorization factor}, \textit{tensor noise bound} and \textit{layer cushion} as in convolutional neural network defined in \ref{def:cnn_tf}, \ref{def:cnn_tnb} and \ref{def:cnn_lc}. But note that the input $\mytensor{\tinput}^{(k)}$ in the definition of \textit{layer cushion} is the input of $k^{\tha}$ layer after skip connection.

\subsection{Generalization Guarantee of Compressed Network Proposed}

\begin{theorem}
\label{thm:sc_thm}
For any convolutional neural network $\origcnn$ of $n$ layers with skip connection satisfying the assumptions in section~\ref{sec:setup} and any margin $\gamma \geq 0$, Algorithm~\ref{alg:cnn_compress} generates a compressed network $\hat{\origcnn}$ such that with high probability over the training set, the expected error $L_0(\hat{\origcnn})$ is bounded by
\begin{equation}
\hat{L}_{\gamma}(\origcnn) + \tilde{O}\Big(\sqrt{\frac{ \sum_{k=1}^{n} \hat{R}^{(k)}(s^{(k)} + o^{(k)} + k_x^{(k)} \times k_y^{(k)} + 1)}{m}}\Big)
\end{equation}
where 
\begin{equation}
\begin{aligned}
\hat{R}^{(k)} = \min \Big\{ j \in [R^{(k)}] | \nb^{(k)}_j \Pi_{i=k+1}^{n} (\tf^{(i)}_j + 1) \leq \frac{\gamma}{2 n \max_{\mytensor{\tinput}  \in S } \fronorm{\origcnn(\mytensor{\tinput})}} \Pi_{i=k}^n \lc^{(i)} \fronorm{\cnnweight^{(i)}} \Big\}
\end{aligned}
\end{equation}
\end{theorem}

To prove this theorem, Lemma~\ref{lem:sc_lemma} is needed.

\begin{lemma}
\label{lem:sc_lemma}
For any convolutional neural network $\origcnn$ of $n$ layers with skip connection satisfying the assumptions in section~\ref{sec:setup} and any error $0\leq\epsilon\leq 1$, Algorithm~\ref{alg:cnn_compress} generates a compressed tensorial neural network $\hat{\mytensor{\origcnn}}$ such that for any $\tinput\in S$:
\begin{equation}
\fronorm{\origcnn(\mytensor{\tinput}) -  \hat{\origcnn}(\mytensor{\tinput})} \leq \epsilon \fronorm{\origcnn(\mytensor{\tinput})}
\end{equation}
The compressed convolutional neural network $\hat{\origcnn}$ has with $\sum_{k=1}^n\hat{R}^{(k)}(s^{(k)}+o^{(k)}+k_x^{(k)} \times k_y^{(k)}+1)$ total parameters, where each $\hat{R}^{(k)}$ satisfies:
\begin{equation}
\begin{aligned}
\hat{R}^{(k)} = \min \Big\{ j \in [R^{(k)}] | \nb^{(k)}_j \Pi_{i=k+1}^{n} (\tf^{(i)}_j + 1) \leq \frac{\epsilon}{n} \Pi_{i=k}^n \lc^{(i)} \fronorm{\cnnweight^{(i)}} \Big\}
\end{aligned}
\end{equation}

\end{lemma}

\subsection{Complete Proofs of Neural Networks with Skip Connection}
\label{app:sc_proofs}
To prove Lemma~\ref{lem:sc_lemma}, the following Lemma~\ref{lem:sc_helper_lemma} is needed.
\begin{lemma}
\label{lem:sc_helper_lemma}
For any convolutional neural network $\origcnn$ of $n$ layers with skip connection satisfying the assumptions in section~\ref{sec:setup}, Algorithm~\ref{alg:cnn_project} generates a compressed tensorial neural network $\hat{\origcnn}$ where for any $\mytensor{\tinput} \in S$:
$$
\fronorm{\origcnn(\mytensor{\tinput})-\hat{\origcnn}(\mytensor{\tinput})}
\leq \left( \sum_{k=1}^{n} \frac{\nb^{(k)}}{\lc^{(k)} \fronorm{\mytensor{M}^{(k)}}}
\prod_{l=k+1}^{n}\frac{\tf^{(l)}+1}{\lc^{(l)} \fronorm{\mytensor{M}^{(l)}}} \right) \fronorm{\origcnn(\mytensor{\tinput})}
$$
where $\nb$, $\lc$, and $\tf$ are \textit{tensor noise bound}, \textit{layer cushion}, \textit{tensorization factor} defined in \ref{def:cnn_tnb}, \ref{def:cnn_lc} and \ref{def:cnn_tf} respectively. 
\end{lemma}

\begin{proof} (of Lemma~\ref{lem:sc_helper_lemma}) 

We know by construction, $\origcnn(\mytensor{\tinput}) = \mytensor{\tinput}^{(n+1)}$ and $\hat{\origcnn}(\mytensor{\tinput} )= \hat{\mytensor{\tinput}}^{(n+1)}$, we can just show $\fronorm{ \mytensor{\tinput}^{(n+1)} - \hat{\mytensor{\tinput}}^{(n+1)} }$ satisfies the above inequality, and we will prove this by induction. Notice

\textbf{Induction Hypothesis:} For any layer $m > 0$, 
$$\fronorm{\mytensor{\tinput}^{(m)} - \hat{\mytensor{\tinput}}^{(m)}} \leq \left( \sum_{k=1}^{m} \frac{\nb^{(k)}}{\lc^{(k)} \fronorm{\mytensor{M}^{(k)}}} \prod_{l=k+1}^{m}\frac{\tf^{(l)}+1}{\lc^{(l)} \fronorm{\mytensor{M}^{(l)}}} \right) \fronorm{ \mytensor{\tinput}^{(m)} }$$

\textbf{Base case:} when $m=1$, the above inequality hold trivially as $\mytensor{\tinput}^{(1)} = \hat{\mytensor{\tinput}}^{(1)}$ as we cannot modify the input, and the RHS is always $\geq 0$. 

\textbf{Inductive Step:} Now we assume show that the induction hypothesis is true for all $m$, then at layer $m+1$ we have  

\begin{equation*}
\begin{split}
\MoveEqLeft \fronorm{\mytensor{\tinput}^{(m+1)} - \hat{\mytensor{\tinput}}^{(m+1)}} \\
&= \fronorm{\relu{\mytensor{\tOut}^{(m)}} - \relu{\hat{\mytensor{\tOut}}^{(m)}}} \\
&\leq \fronorm{\mytensor{\tOut}^{(m)} - \hat{\mytensor{\tOut}}^{(m)}} \\
&\leq \fronorm{ \cnnweight^{(m)} \left( \mytensor{\tinput}^{(m)} \right) + \mytensor{\tinput}^{(m)} - \left( \hat{\cnnweight}^{(m)} \left(\hat{\mytensor{\tinput}}^{(m)}\right) +\hat{\mytensor{\tinput}}^{(m)} \right)} \\
&\leq \fronorm{ \cnnweight^{(m)} \left(\mytensor{\tinput}^{(m)}\right) - \hat{\cnnweight}^{(m)} \left(\hat{\mytensor{\tinput}}^{(m)}\right)}+ \fronorm{\mytensor{\tinput}^{(m)} - \hat{\mytensor{\tinput}}^{(m)}} \\
&\leq \fronorm{\hat{\cnnweight}^{(m)} \left( \mytensor{\tinput}^{(m)} - \hat{\mytensor{\tinput}}^{(m)} \right) + \left( \cnnweight^{(m)} - \hat{\cnnweight}^{(m)} \right)\left(\mytensor{\tinput}^{(m)}\right)} + \fronorm{\mytensor{\tinput}^{(m)} - \hat{\mytensor{\tinput}}^{(m)}} \\
&\leq \sqrt{HW} \left(\tf^{(m)} + 1\right) \fronorm{\mytensor{\tinput}^{(m)} - \hat{\mytensor{\tinput}}^{(m)}} + \sqrt{HW} \nb^{(m)} \fronorm{\mytensor{\tinput}^{(m)}} \\
&\leq \sqrt{HW} \left(\tf^{(m)} + 1\right) \left( \sum_{k=1}^{m-1} \frac{\nb^{(k)}}{\lc^{(k)} \fronorm{\mytensor{M}^{(k)}}}
\prod_{l=k+1}^{m-1}\frac{\tf^{(l)}+1}{\lc^{(l)} \fronorm{\mytensor{M}^{(l)}}} \right) 
\fronorm{\mytensor{\tinput}^{(m)}} + \sqrt{HW} \nb^{(m)} \fronorm{\mytensor{\tinput}^{(m)}} \\
&\leq \left( \sum_{k=1}^{m-1} \frac{\nb^{(k)}}{\lc^{(k)} \fronorm{\mytensor{M}^{(k)}}}
\prod_{l=k+1}^{m-1}\frac{\tf^{(l)}+1}{\lc^{(l)} \fronorm{\mytensor{M}^{(l)}}} \right) \frac{(\tf^{(m)} + 1)}{\lc^{(m)} \fronorm{\mytensor{M}^{(m)}}} \fronorm{\mytensor{\tinput}^{(m+1)}} + \frac{\nb^{(m)}}{\lc^{(m)} \fronorm{\mytensor{M}^{(m)}}} \fronorm{\mytensor{\tinput}^{(m+1)}} \\
&\leq \left( \sum_{k=1}^{m} \frac{\nb^{(k)}}{\lc^{(k)} \fronorm{\mytensor{M}^{(k)}}}
\prod_{l=k+1}^{m}\frac{\tf^{(l)}+1}{\lc^{(l)} \fronorm{\mytensor{M}^{(l)}}} \right) \fronorm{\mytensor{\tinput}^{(m+1)}}
\end{split}
\end{equation*}

The proof of Lemma~\ref{lem:sc_helper_lemma} is then completed by induction.

\end{proof}

Now we can proove Lemma~\ref{lem:sc_lemma}

\begin{proof} (of Lemma~\ref{lem:sc_lemma}) 

The proof is similar with the proof of Lemma~\ref{lem:cnn_lemma}. The only difference is we replace $\tf^{(l)}$ by $\tf^{(l)}+1$.

\end{proof}

To prove Theorem~\ref{thm:sc_thm}, the following lemma is needed.

\begin{lemma}
\label{lem:sc_helper_lemma_2}
For any convolutional neural network $\origcnn$ of $n$ layers with skip connection satisfying the assumptions in section~\ref{sec:setup} and any margin $\gamma \geq 0$, $\origcnn$ can be compressed to a tensorial convolutional neural network $\hat{\origcnn}$ with $\sum_{k=1}^{n} \hat{R}^{(k)}(s^{(k)}+t^{(k)}+k_x^{(k)} \times k_y^{(k)}+1)$ total parameters such that for any $\mytensor{\tinput} \in S$, $\hat{L_0}(\hat{\origcnn}) \leq \hat{L}_\gamma(\origcnn)$.
Here, for each layer $k$, 
$$
\begin{aligned}
\hat{R}^{(k)} = \min \Big\{ j \in [R^{(k)}] | \nb^{(k)}_j \Pi_{i=k+1}^{n} (\tf^{(i)}_j + 1) \leq \frac{\epsilon}{n} \Pi_{i=k}^n \lc^{(i)} \fronorm{\cnnweight^{(i)}} \Big\}
\end{aligned}
$$
\end{lemma}

The proof of Lemma~\ref{lem:sc_helper_lemma_2} is the same with Lemma~\ref{lem:helper_lemma_2}. And by setting $\epsilon=\frac{\gamma}{2 \max_{\mytensor{\tinput}\in S}}$, we get the desired expression of $\hat{R}^{(k)}$ in the main theorem.

\begin{proof} (of Theorem~\ref{thm:sc_thm}) 
Similarly, let us bound the covering number of the compressed network $\hat{\origcnn}$ by approximating each parameter with accuracy $\acc$. 

\textbf{Covering Number Analysis for Convolutional Neural Network} Let $\tilde{\cnnweight}$ denote the network after approximating each parameter in $\hat{\origcnn}$ with accuracy $\acc$. We use the same assumptions and notations with the proof of Theorem~\ref{thm:cnn_thm}. And we still use $\mytensor{\tinput}^{(k)}, \mytensor{\tOut}^{(k)}, \origcnn^{(k)}$ to denote $\hat{\mytensor{\tinput}}^{(k)}, \hat{\mytensor{\tOut}}^{(k)}, \hat{\origcnn}^{(k)}$

\textbf{Bound $\tau^{(k)}=\fronorm{\tilde{\mytensor{\tOut}}^{(k)}-\hat{\mytensor{\tOut}}^{(k)}}$:}

\begin{equation*}
\begin{split}
\MoveEqLeft \fronorm{\tilde{\mytensor{\tOut}}^{(k)}-\mytensor{\tOut}^{(k)}} \\
&= \fronorm{\tilde{\origcnn}^{(k)}(\tilde{\mytensor{\tinput}}^{(k)}) + \tilde{\mytensor{\tinput}}^{(k)} - \big(\origcnn^{(k)}(\mytensor{\tinput}^{(k)}) + \mytensor{\tinput}^{(k)} \big) } \\
&\leq \fronorm{\tilde{\origcnn}^{(k)}(\tilde{\mytensor{\tinput}}^{(k)})-\origcnn^{(k)}(\mytensor{\tinput}^{(k)})} + \fronorm{\tilde{\mytensor{\tinput}}^{(k)} - \mytensor{\tinput}^{(k)}} \\
&= \fronorm{\tilde{\origcnn}^{(k)}(\tilde{\mytensor{\tinput}}^{(k)})-\origcnn^{(k)}(\mytensor{\tinput}^{(k)})} + \fronorm{\relu{\tilde{\mytensor{\tOut}}^{(k)}} - \relu{\mytensor{\tOut}^{(k)}}} \\
&\leq \fronorm{\tilde{\origcnn}^{(k)}(\tilde{\mytensor{\tinput}}^{(k)})-\origcnn^{(k)}(\mytensor{\tinput}^{(k)})} + \fronorm{\tilde{\mytensor{\tOut}}^{(k-1)} - \mytensor{\tOut}^{(k-1)}} \\
&= \fronorm{\tilde{\origcnn}^{(k)}(\tilde{\mytensor{\tinput}}^{(k)})-\origcnn^{(k)}(\mytensor{\tinput}^{(k)})} + \tau^{(k-1)} \\
\end{split}
\end{equation*}

Based on the proof of Theorem~\ref{thm:cnn_thm} (in Appendix~\ref{app:cnn_proofs}), we can easily get 

\begin{equation*}
\begin{split}
\MoveEqLeft \fronorm{\tilde{\mytensor{\tOut}}^{(k)}-\mytensor{\tOut}^{(k)}} \\
&= \sum_{k=1}^n\sum_{i=1}^k\alpha^{(i)}\prod_{t=i+1}^{k}\beta^{(t)}
\end{split}
\end{equation*}

where $\alpha^{(k)}=4 HW \fronorm{\mytensor{\tinput}^{(k)}}^2 \big( \sum_r^{\hat{R}^{(k)}}(\lambda_r^{(k)})^2 \hat{R}^{(k)} k_x^{(k)} k_y^{(k)} + 4\sum_r^{\hat{R}^{(k)}}(\lambda_r^{(k)})^2(o^{(k)}+s^{(k)})\hat{R}^{(k)}+2(\hat{R}^{(k)})^2 \big) \acc^2$, \\
and $\beta^{(k)}=2\fronorm{\tilde{\cnnweight}^{(k)}}^2$.

Since $\forall k \in [n], \norm{\mytensor{\tinput}^{(k)}} \leq \Pi_{i = k}^n \frac{\fronorm{\mytensor{\tinput}^{(n+1)}}}{\lc^{(i)} \fronorm{\cnnweight^{(i)}}} $, to obtain an $\epsilon$-cover of the compressed network, we can first assume $\beta^{(k)} \geq 1 \ \forall k \in [n]$. Then $\acc$ need to satisfy:

$$\acc \leq \frac{\epsilon}{ 2 \sqrt{HW}n^2\fronorm{\mytensor{\tinput}^{(n+1)}}  \hat{R}^{(*)} \sqrt{(\lambda^{(*)})^2 k_x^{(*)} k_y^{(*)} + 4(\lambda^{(*)})^2(o^{(*)}+s^{(*)})+2}(\frac{\sqrt{2}\fronorm{\tilde{\cnnweight}^{(*)}}}{\lc^{(*)}\fronorm{\cnnweight^{(*)}}})^n}$$

where $\hat{R}^{(*)} = \max_{k} r^{(k)}$ $\lambda^{(*)} = \max_{r, k} \lambda_{r}^{(k)}$, $s^{(*)} = \max_{k} s^{(k)}$, $o^{(*)} = \max_{k} o^{(k)}$, $k_x^{(*)} = \max_{k} k_x^{(k)}$, $k_y^{(*)} = \max_{k} k_y^{(k)}$ and 
$\frac{\fronorm{\tilde{\cnnweight}^{(*)}}} { \mu^{(*)} \fronorm{\cnnweight^{(*)}}} = \max_{k} \frac{\fronorm{\tilde{\cnnweight}^{(k)}}} {\mu^{(k)} \fronorm{\cnnweight^{(*)}}}$

So the skip connections don't change the limiting behavior of the covering number, which w.r.t to a given $\epsilon$ is $\tilde{O}(n d)$ ($n$ is the number of layers in the given neural network, $d$ is the number of parameters), and $\tilde{O}(d)$ for practical neural networks. Because skip connections don't need extra parameters, the neural network still has $\sum_{k=1}^{n} \hat{R}^{(k)}(s^{(k)}+t^{(k)}+k_x^{(k)}\times k_y^{(k)}+1)$ total parameters.

\end{proof}

\newpage
\section{Additional Algorithms and Algorithmic Details}
\label{app:algos}
\textbf{Details of Step 3 in Algorithm~\ref{alg:cnn_compress}.} We use the alternating least squares (ALS) in the implementation of step 3, which is the `parafac’ method of the tensorly library~\citep{kossaifi2019tensorly}, to obtain the CP decomposition. Though CP decomposition is in general NP-hard, the ALS method usually converges for most tensors, with polynomial convergence rate w.r.t. the given precision of the allowed reconstruction error~\citep{anandkumar2014provable, anandkumar2014guaranteed, anandkumar2014tensor}. In addition, step 3 obtains a CP parametrization of the weight tensor rather than recovers the true components of the weight tensor’s CP decomposition. The rank in the CP decomposition is selected in step 2 and is an upper bound of the true rank of the tensor (Proposition~\ref{assum:CP-cnn}). Thus, with the chosen rank, we can obtain a CP decomposition with a very low reconstruction error. In practice, for our cases, the CP decomposition method (ALS) used in step 3 always converges within a few iterations, with reasonable run time.
\begin{algorithm}[htbp!]
\caption{Find Best Rank for CNN (\FBRC)}\label{alg:find_best_rank_cnn}
\begin{algorithmic}[1]
\REQUIRE A list of weight tensors $\{\cnnweight^{(k)}\}_{k=1}^n$ in the original network $\origcnn$ where each $\cnnweight^{(k)} \in \Rbb^{s^{(k)} \times o^{(k)} \times k_x^{(k)} \times k_y^{(k)}}$, a list of number of components $\{R^{(k)}\}_{k=1}^n$, a list of layer cushions $\{ \lc^{(k)} \}_{k=1}^{n }$ of the original network, and a perturbation parameter $\epsilon$ which denotes the maximum error we could tolerate regarding the difference between the output of original network and that of compressed network.
\ENSURE Returns a list of number of components $\{ \hat{R}^{(k)} \}_{k=1}^{n}$ for the compressed network such that  $\fronorm{\origcnn(\mytensor{\tinput})-  \hat{\origcnn}(\mytensor{\tinput})} \leq \epsilon$. Notice that for each $k$, if the original network does not have skip connections, $\hat{R}^{(k)}$ satisfies that 
\begin{equation}
\label{eq:alg_fbrc}
\nb^{(k)}_{\hat{R}^{(k)}} \Pi_{i=k+1}^{n} \tf^{(i)}_{\hat{R}^{(k)}} \leq \frac{\epsilon}{n} \Pi_{i=k}^n \lc^{(i)} \fronorm{\cnnweight^{(i)}}
\end{equation}
or if skip connection is used, $\hat{R}^{(k)}$ satisfies that 
\begin{equation}
\label{eq:alg_fbrc_skip}
\nb^{(k)}_{\hat{R}^{(k)}} \Pi_{i=k+1}^{n} (\tf^{(i)}_{\hat{R}^{(k)}}+1) \leq \frac{\epsilon}{n} \Pi_{i=k}^n \lc^{(i)} \fronorm{\cnnweight^{(i)}}
\end{equation}
\STATE{For each layer $k$, calculate the following properties: layer cushion $\lc^{(k)}$, weight norm $\fronorm{\cnnweight^{(k)}}$, then calculate the RHS $\frac{\epsilon}{n} \Pi_{i=k}^n \lc^{(i)} \fronorm{\cnnweight^{(i)}}$ for each $k$}
\STATE{Find the smallest $\hat{R}^{(n)}$} such that the \nblong for the last layer $\nb^{(n)}$ satisfies $\nb^{(n)} \leq \frac{\epsilon}{n} \lc^{(n)} \fronorm{\cnnweight^{(n)}}$
\FOR{$k=n-1$ to $1$} 
	\IF{$\orignn$ does not have skip connections}
		\STATE{Calculate the multiplication of \tflong for layers upper than $k$, i.e., $\Pi_{i=k+1}^{n} \tf^{(i)}_{\hat{R}^{(i)}}$, based on the choices of $\hat{R}^{(i)}$ for $k \leq i \leq n$}
		\STATE{Find the smallest $\hat{R}^{(k)}$ by calculating the largest possible $\nb^{(k)}$ such that Equation~\ref{eq:alg_fbrc} holds.}
	\ELSE
		\STATE{Calculate the multiplication of \tflong for layers upper than $k$, i.e., $\Pi_{i=k+1}^{n} (\tf^{(i)}_{\hat{R}^{(k)}}+1)$, based on the choices of $\hat{R}^{(i)}$ for $k \leq i \leq n$}
		\STATE{Find the smallest $\hat{R}^{(k)}$ by calculating the largest possible $\nb^{(k)}$ such that Equation~\ref{eq:alg_fbrc_skip} holds.}
	\ENDIF
\ENDFOR
\STATE{Return $\{ \hat{R}^{(k)} \}_{k=1}^{n}$}
\end{algorithmic}
\end{algorithm}

\textit{Remark}. The \FBRC algorithm finds a set of ranks that satisfies inequality~\ref{eq:alg_fbrc} (CNNs) or ~\ref{eq:alg_fbrc_skip} (NNs with skip connections) within polynomial time because of the following guarantees. The total number of possible sets of ranks (say $T$), which the \FBRC algorithm will at most search through, is equal to the product of the ranks of all layers. The rank of each layer is upper bounded by Proposition~\ref{assum:CP-cnn} and thus $T$ is polynomial w.r.t. the shape of the original weight tensors and the number of layers. Moreover, the search will definitely succeed as the inequalities~\ref{eq:alg_fbrc} and~\ref{eq:alg_fbrc_skip} automatically hold when $\hat{R}^{(k)} = R^{(k)}$.

\begin{algorithm}[tbh]
\caption{Find Best Rank (\FBR)}\label{alg:find_best_rank}
\begin{algorithmic}[1]
\REQUIRE A list of tensors $\{\nnweight^{(k)}\}_{k=1}^n$ where each $\nnweight^{(k)} \in \Rbb^{s^{(k)}_1 \times s^{(k)}_2 \times s^{(k+1)}_1 \times s^{(k+1)}_2}$ is reshaped from a matrix $\mymatrix{A^{(k)} } $,  a list of number of components $\{R^{(k)}\}_{k=1}^n$, a list of layer cushions $\{ \lc^{(k)} \}_{k=1}^{n }$ of the original network, and a perturbation parameter $\epsilon$ which denotes the maximum error we could tolerate regarding the difference between the output of original network and that of compressed network.
\ENSURE Returns a list of number of components $\{ \hat{R}^{(k)} \}_{k=1}^{n}$ for the compressed network such that  $\fronorm{\orignn(\mymatrix{\tinput})-  \hat{\orignn}(\mymatrix{\tinput})} \leq \epsilon$.

\begin{align}
\label{eq:alg_fbr}
\rf^{(k)} \nb^{(k)}_{\hat{R}^{(k)}} \Pi_{i=k+1}^{n} \tf^{(i)}_{\hat{R}^{(k)}} \leq \frac{\epsilon}{n}\fronorm{\mymatrix{A}^{(k)}} \Pi_{i=k}^n \lc^{(i)} \}
\end{align}
\STATE{For each layer $k$, calculate the following properties: reshaping factor $\rf^{(k)}$, layer cushion $\lc^{(k)}$, weight norm $\fronorm{\mymatrix{A}^{(k)}}$, then calculate the RHS $\frac{\epsilon}{n}\fronorm{\mymatrix{A}^{(k)}} \Pi_{i=k}^n \lc^{(i)}$ for each $k$}
\STATE{Find the smallest $\hat{R}^{(n)}$ such that the \nblong for the last layer $\nb^{(n)}$ satisfies $\rf^{(n)} \nb^{(n)} \leq \frac{\epsilon}{n} \lc^{(n)} \fronorm{\mymatrix{A}^{(k)}}$}
\FOR{$k=n-1$ to $1$} 
	\STATE{Calculate the multiplication of \tflong for layers upper than $k$, i.e., $\Pi_{i=k+1}^{n} \tf^{(i)}_{\hat{R}^{(i)}}$, based on the choices of $\hat{R}^{(i)}$ for $k \leq i \leq n$}
	\STATE{Find the smallest $\hat{R}^{(k)}$ by calculating the largest possible $\nb^{(k)}$ such that Equation~\ref{eq:alg_fbr} holds.}
\ENDFOR
\STATE{Return $\{ \hat{R}^{(k)} \}_{k=1}^{n}$}
\end{algorithmic}
\end{algorithm}

\begin{algorithm}[tbh]
\caption{\CNNPROJECT}\label{alg:cnn_project}
\label{algo:cnn_project}
\begin{algorithmic}[1]
\REQUIRE A convolutional neural network $\origcnn$ of $n$ layers where its weight tensor $\cnnweight^{(k)}$ of the $k^{\tha}$ layer is parametrized by $\{{\lambda}_{r}^{(k)}, {a}_{r}^{(k)}, {b}_{r}^{(k)}, {c}_{r}^{(k)}\}_{r=1}^{R^{(k)}}$, and a list of ranks $\{\hat{R}^{(k)}\}_{i=1}^n$.  
\ENSURE Returns a compressed network $\hat{\origcnn}$ of $\origcnn$ where for each layer $k$,  $\norm{\hat{\cnnweight}^{(k)}}$ is constructed by the top $\hat{R}^{(k)}$ components from CP components of $\cnnweight^{(k)}$. 
\FOR{$k=1$ to $n$} 
\STATE $\hat{\cnnweight}^{(k)} \gets \sum_{r = 1}^{\hat{R}^{(k)}} {\lambda}_{r}^{(k)} {a}_{r}^{(k)} \otimes {b}_{r}^{(k)}\otimes {c}_{r}^{(k)}$
\STATE Let $\hat{\cnnweight}^{(k)}$ be the weight tensor of the $k^{th}$ layer in $\hat{\origcnn}$
\ENDFOR
\STATE Return $\hat{\origcnn}$
\end{algorithmic}
\end{algorithm}
\begin{algorithm}[tbh]
\caption{\TNNPROJECT}\label{alg:tfc_project}
\begin{algorithmic}[1]
\REQUIRE A fully connected neural network $\tnn$ of $n$ layers where its weight tensor $\nnweight^{(k)}$ of the $k^{\tha}$ layer is parametrized by $\{{\lambda}_{r}^{(k)}, {a}_{r}^{(k)}, {b}_{r}^{(k)}, {c}_{r}^{(k)}, {d}_{r}^{(k)}\}_{r=1}^{R^{(k)}}$, and a list of ranks $\{\hat{R}^{(k)}\}_{i=1}^n$.  
\ENSURE Returns a compressed network $\hat{\tnn}$ of $\orignn$ where for each layer $k$,  $\norm{\mytensor{\hat{T}}^{(k)}}$ is constructed by the top $\hat{R}^{(k)}$ components from CP components of $\nnweight^{(k)}$. 
\FOR{$k=1$ to $n$} 
\STATE $\hat{\nnweight}^{(k)} \gets \sum_{r = 1}^{\hat{R}^{(k)}} {\lambda}_{r}^{(k)} {a}_{r}^{(k)} \otimes {b}_{r}^{(k)}\otimes {c}_{r}^{(k)} \otimes {d}_{r}^{(k)}$
\STATE Let $\mytensor{\hat{T}}^{(k)}$ be the weight tensor of the $k^{th}$ layer in $\hat{\tnn}$
\ENDFOR
\STATE Return $\hat{\tnn}$
\end{algorithmic}
\end{algorithm}

\end{document}